\documentclass[letterpaper]{article}
\usepackage{aaai2026} 
\usepackage{times} 
\usepackage{helvet} 
\usepackage{courier} 
\usepackage[hyphens]{url} 
\usepackage{graphicx} 
\usepackage{graphicx} 
\usepackage{natbib} 
\usepackage{caption} 
\usepackage{algorithm}
\usepackage{algorithmic}

\usepackage{microtype}
\usepackage{graphicx}
\usepackage{subfigure}
\usepackage{booktabs} 

\usepackage{amsmath}
\usepackage{amssymb}
\usepackage{mathtools}
\usepackage{amsthm}

\usepackage[capitalize,noabbrev]{cleveref}

\theoremstyle{plain}
\newtheorem{theorem}{Theorem}[section]
\newtheorem{prop}[theorem]{Proposition}
\newtheorem{thm}[theorem]{Theorem}
\newtheorem{lem}[theorem]{Lemma}
\newtheorem{cor}[theorem]{Corollary}
\theoremstyle{definition}
\newtheorem{defn}[theorem]{Definition}
\newtheorem{assum}[theorem]{Assumption}
\theoremstyle{remark}
\newtheorem{remark}[theorem]{Remark}

\usepackage[utf8]{inputenc} 
\usepackage[T1]{fontenc}    
\usepackage{amsfonts}       
\usepackage{mathtools}
\usepackage{mathrsfs}
\usepackage{booktabs}       

\usepackage{nicefrac}       
\usepackage{microtype}      
\usepackage{xcolor}         
\usepackage{comment}

\usepackage{algorithm}
\usepackage{algorithmic}
\usepackage{enumerate}
\usepackage{times}
\usepackage{tikz}
\usepackage{dsfont}
\usepackage{color}
\usepackage{units}

\newcommand{\cA}{\mathcal{A}}
\newcommand{\cB}{\mathcal{B}}
\newcommand{\cC}{\mathcal{C}}

\newcommand{\cF}{\mathcal{F}}
\newcommand{\cG}{\mathcal{G}}

\newcommand{\cM}{\mathcal{M}}
\newcommand{\cN}{\mathcal{N}}
\newcommand{\cO}{\mathcal{O}}
\newcommand{\ctO}{\tilde{\mathcal{O}}}
\newcommand{\cP}{\mathcal{P}}

\newcommand{\cR}{\mathcal{R}}
\newcommand{\cS}{\mathcal{S}}

\newcommand{\bE}{\mathbb{E}}
\newcommand{\bN}{\mathbb{N}}
\newcommand{\bP}{\mathbb{P}}
\newcommand{\bR}{\mathbb{R}}
\newcommand{\bZ}{\mathbb{Z}}

\newcommand{\ust}{^{\star}}

\newcommand{\up}{^{\prime}}
\newcommand{\upp}{^{\prime\prime}}
\newcommand{\Te}{\Theta}
\newcommand{\te}{\theta}

\newcommand{\gm}{\gamma}
\newcommand{\eps}{\epsilon}

\newcommand{\lm}{\lambda}

\newcommand{\ts}{\tilde{s}}

\newcommand{\uc}[1]{^{(#1)}}

\newcommand{\pzrlmb}{\textit{PZRL-MB}}
\newcommand{\pzrlmf}{\textit{PZRL-MF}}

\newcommand{\deff}{d_{\text{eff.}}}
\newcommand{\idxf}{\text{Index}\uc{f}}
\newcommand{\idxb}{\text{Index}\uc{b}}

\newcommand{\norm}[1]{\left\lVert#1\right\rVert}
\newcommand{\abs}[1]{\left|#1\right|}

\newcommand{\br}[1]{\left(#1\right)}
\newcommand{\flbr}[1]{\left\{#1\right\}}
\newcommand{\sqbr}[1]{\left[#1\right]}
\newcommand{\ceil}[1]{\left\lceil#1\right\rceil}
\newcommand{\floor}[1]{\left\lfloor#1\right\rfloor}

\newcommand{\ovl}[1]{\mkern 1.5mu\overline{\mkern-1.5mu#1\mkern-1.5mu}\mkern 1.5mu}

\newcommand{\diamb}[2]{\mbox{diam}\uc{b}_{#1}\br{#2}}
\newcommand{\diamf}[2]{\mbox{diam}\uc{f}_{#1}\br{#2}}

\newcommand{\diamc}[1]{\mbox{diam}\br{#1}}
\newcommand{\spn}[1]{~\textit{sp}\br{#1}}
\newcommand{\ind}[1]{\mathbb{I}_{\flbr{#1}}}

\newcommand{\inv}{^{-1}}
\newcommand{\mycomment}[1]{}

\newcommand{\al}[1]{\begin{align}#1\end{align}}
\newcommand{\nal}[1]{\begin{align*}#1\end{align*}}

\usepackage{newfloat}
\usepackage{listings}
\DeclareCaptionStyle{ruled}{labelfont=normalfont,labelsep=colon,strut=off} 
\floatstyle{ruled}
\newfloat{listing}{tb}{lst}{}
\floatname{listing}{Listing}
\title{Policy Zooming: Adaptive Discretization-based Infinite-Horizon Average-Reward Reinforcement Learning}
\author{
	Avik Kar\textsuperscript{\rm 1}, Rahul Singh
}
\affiliations{
	\textsuperscript{\rm 1} Department of Computer Science and Automation, \\
    Indian Institute of Science\\
	avikkar@iisc.ac.in, rahulsingh0188@gmail.com
}

\begin{document}
\maketitle
\begin{abstract}
    We study the infinite-horizon average-reward reinforcement learning (RL) for continuous space Lipschitz MDPs in which an agent can play policies from a given set $\Phi$. The proposed algorithms efficiently explore the policy space by ``zooming'' into the ``promising regions'' of $\Phi$, thereby achieving adaptivity gains in the performance. We upper bound their regret as $\tilde{\mathcal{O}}\big(T^{1 - d_{\text{eff.}}^{-1}}\big)$, where $d_{\text{eff.}} = d^\Phi_z+2$ for model-free algorithm $\textit{PZRL-MF}$ and $d_{\text{eff.}} = 2d_\mathcal{S} + d^\Phi_z + 3$ for model-based algorithm $\textit{PZRL-MB}$. Here, $d_\mathcal{S}$ is the dimension of the state space, and $d^\Phi_z$ is the zooming dimension given a set of policies $\Phi$. $d^\Phi_z$ is an alternative measure of the complexity of the problem, and it depends on the underlying MDP as well as on $\Phi$. Hence, the proposed algorithms exhibit low regret in case the problem instance is benign and/or the agent competes against a low-complexity $\Phi$ (that has a small $d^\Phi_z$). When specialized to the case of finite-dimensional policy space, we obtain that $d_{\text{eff.}}$ scales as the dimension of this space under mild technical conditions; and also obtain $d_{\text{eff.}} = 2$, or equivalently $\tilde{\mathcal{O}}(\sqrt{T})$ regret for $\textit{PZRL-MF}$, under a curvature condition on the average reward function that is commonly used in the multi-armed bandit (MAB) literature.
\end{abstract}

\section{Introduction}\label{sec:intro}
Reinforcement Learning (RL)~\citep{sutton2018reinforcement} is a popular framework in which an agent repeatedly interacts with an unknown environment modeled by a Markov decision process (MDP)~\citep{puterman2014markov} and the goal is to choose actions sequentially in order to maximize the cumulative rewards earned by the agent.~We study infinite-horizon average reward MDPs in continuous state and action spaces endowed with a metric, in which the transition kernel and reward functions are Lipschitz (Assumption~\ref{assum:lip}).~The class of Lipschitz MDPs covers a broad class of problems, such as the class of linear MDPs~\citep{jin2020provably}, RKHS MDPs~\cite{chowdhury2019online}, linear mixture models, RKHS approximation, and the nonlinear function approximation framework considered in~\cite{osband2014model} and~\cite{kakade2020information}.~See~\cite{maran2024no,maran2024projection} for more details, or refer to Figure~\ref{fig:mdp_rel}.~We note that even though discrete and linear MDPs have been extensively studied in the literature, they might not be suitable for many real-world applications since it is becoming increasingly common to deploy RL and control algorithms in systems that are non-linear and continuous~\cite{nair2022r3m,kumar2021rma}.
\begin{figure}[t]
    \centering
    \includegraphics[width=0.7\linewidth]{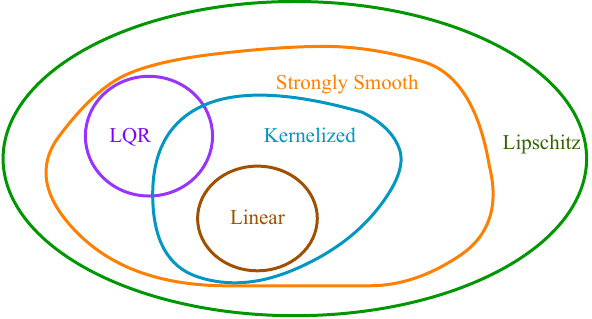}
    \caption{Relations among families of continuous space RL problems. LQR stands for Linear Quadratic Regulator~\citep{abbasi2011regret}.~Our assumptions correspond to the green set. Diagram is taken from~\citep{maran2024no}, see~\citep{maran2024no} for more details.\vspace{-20pt}}
    \label{fig:mdp_rel}
\end{figure}

Let $d_{\cS}$ and $d_{\cA}$ denote the dimensions of the state and action spaces, respectively, and define $d := d_{\cS} + d_{\cA}$. For episodic Lipschitz MDPs, the regret scales as $\ctO\big(K^{1 - \deff\inv}\big)$,\footnote{$\ctO$ suppresses poly-logarithmic dependence in $K$ or $T$.} where $K$ is the number of episodes and $\deff$ is the \emph{effective dimension}, which depends on both the underlying MDP and the algorithm. For instance, using a \textit{fixed discretization} yields $\deff = d + 2$~\citep{song2019efficient}. In contrast, adaptive algorithms can exploit MDP structure to reduce $\deff$.~Prior works~\citep{cao2020provably,sinclair2023adaptive} employ adaptive discretization and a technique called ``zooming,'' which reduces $\deff$ to $d_z + 2$, where $d_z$ is the \emph{zooming dimension}. However, this notion of $d_z$, designed for episodic settings, fails to capture adaptivity in average reward problems. Specifically, as the horizon grows, $d_z \to d$, making adaptive methods no better than fixed discretization.~To address this,~\citet{kar2024provably} introduced a new definition of zooming dimension tailored for average reward RL, and achieved $\deff = 2d_\cS + d_z + 3$, with $d_z \le d$. However, their methods assume compactness of the state-action space and focus solely on the complexity of the MDP.

In this work, we propose adaptive discretization-based algorithms that (i) handle non-compact spaces and (ii) apply to the setups efficiently where the performance is to be compared against a known class of policies. Specifically, we propose the \emph{zooming dimension} given a policy class $\Phi$, denoted by $d^\Phi_z$, that captures the joint complexity of the MDP and the comparator policy class. Thus, we refine the idea of $d_z$ to depend on a policy class as well.~If the optimal policy belongs to a ``simple'' class, this refinement enables significantly smaller $d^\Phi_z$, yielding $d^\Phi_z \ll d$. We analyze regret with respect to a given policy class $\Phi$~\eqref{def:regret}, a widely accepted approach in complex systems~\citep{hazan2022introduction,rakhlin}. Our model-free and model-based algorithms achieve regret bounds with effective dimensions $\deff = d^\Phi_z + 2$ and $\deff = 2 d_\cS + d^\Phi_z + 3$, respectively. It turns out that our algorithms activate policies from the given policy class in an efficient way as compared to an algorithm that uses a uniform grid for policy search. Figure~\ref{fig:actv_pol} depicts that \pzrlmf~activates fewer policies from suboptimal regions and more from near-optimal regions as compared to a naive uniform discretization.

\begin{figure}[!t]
    \centering
    \subfigure[\pzrlmf]{
    \label{fig:adappol}
    \includegraphics[height=1in]{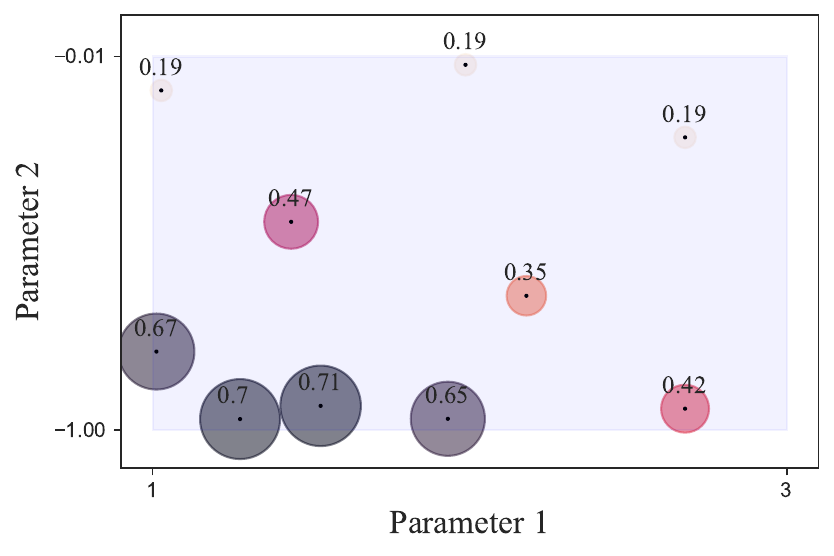}}
    \subfigure[Policy UCB~(Algo~\ref{algo:pucb})]{
    \label{fig:unifpol}
    \includegraphics[height=1in]{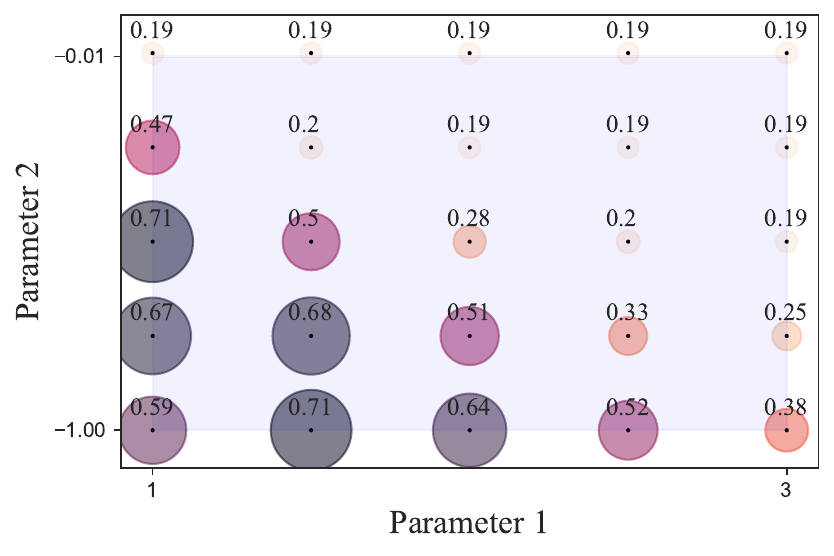}}
    \caption{We show the policies activated by different algorithms for one single trajectory of the transmission scheduling example (See Section~\ref{sec:sim}). The radius of the balls around an active policy is proportional to its average reward. Uniform discretization-based algorithms waste resources to learn a larger number of policies, whereas adaptive algorithms activate more policies from the near-optimal regions.}
    \label{fig:actv_pol}
\end{figure}

To illustrate the intuition behind zooming, we revisit its origin in the simpler setting of Lipschitz MABs~\citep{kleinberg2008multi}.

\textbf{Lipschitz MABs: The Zooming Algorithm.} The agent maintains a set of ``active arms,'' and their ``confidence balls'' whose radii are equal to the confidence radii associated with the corresponding arm's estimated reward. Thanks to Lipschitz continuity, rewards of nearby arms can be inferred from the active ones.\footnote{Let $a$ be an active arm with confidence radius $\eta$ and empirical mean $\hat{\mu}_a$. Then, for any arm $a\up$ in its confidence ball, the mean reward lies in $[\hat{\mu}_a - (1+L)\eta, \hat{\mu}_a + (1+L)\eta]$ with high probability.} New arms are activated only if not covered by existing confidence balls. The agent selects the arm with the highest upper confidence bound (UCB) index. Because the confidence radius shrinks with the number of plays, the algorithm ``zooms in'' on promising regions, those with high UCB index. This adaptive behavior yields an effective dimension of $d_z + 2$.~A similar idea has been explored in \citep{bubeck2011x}.

\subsection{Challenges}
For MABs, the zooming algorithm plays only from amongst the active arms since the UCB index of an active arm turns out to be an optimistic estimate of the mean reward of each arm lying inside its confidence ball.~For policies, however, rewards are not unbiased estimates of the long-term average unless the controlled Markov process (CMP) is at stationarity, making confidence radius design and optimism proofs more involved.~Moreover, to our knowledge, model-free UCB indices have not been explored for average reward RL, not even in tabular MDPs.~Another challenge lies in selecting an appropriate norm for measuring distances between policies~(note that since the policy space is not finite-dimension, all norms are not equivalent).

\subsection{Contributions}\label{subsec:contribution}
\quad 1.~To the best of our knowledge, this is the first work to provide finite-time regret bounds for average reward RL in general state-action space MDPs with $d > 1$. Prior works~\citep{ortner2012online, qian2019exploration, wei2021learning, he2023sample} are either limited to finite action spaces or assume $d_{\cS} = 1$.

2.~We develop two algorithms, \pzrlmf~(model-free) and \pzrlmb~(model-based), that use policy-based zooming and UCB methods~\citep{lattimore2020bandit}.~Our main novelty is a new complexity measure for average reward RL: the zooming dimension $d_z^\Phi$, defined via policy covers of $\Phi$.~We show regret bounds of $\ctO\big(T^{1-\deff\inv}\big)$ with $\deff = d^\Phi_z + 2$ for \pzrlmf, and $\deff = 2d_{\cS} + d^\Phi_z + 3$ for \pzrlmb.~Importantly, a small $d_z^\Phi$ does not imply the MDP belongs to nice class of MDPs, such as linear MDPs or tabular MDPs. When $\Phi$ is parameterized over $W \subset \bR^{d_w}$, we show $d_z^\Phi \le d_w$ under mild assumptions.~For MDPs with bi-Lipschitz average reward functions, we get $\deff = 2$ and hence an $\cO(\sqrt{T})$ regret for \pzrlmf.~Note that the bi-Lipschitz assumption is common in structured bandits~\citep{cope2009regret, yu2011unimodal, combes2020unimodal}.

3.~Along the way, we prove a novel sensitivity result (Theorem~\ref{thm:lip_avg_reward}) for Markov processes on general state spaces~\citep{meyn2012markov}.~We bound the distance between the stationary distributions of Markov chains in terms of a weighted distance measure~\eqref{def:policy_metric} between the transition kernels, improving over the existing results~\citep{mitrophanov2005sensitivity,mouhoubi2021perturbation} that bounds the same quantity in terms of the sup distance between the transition kernels.

4. Existing algorithms for general state spaces are often computationally intractable without linearity~\citep{ayoub2020model} or deterministic dynamics~\citep{wu2024computationally}. In contrast, our methods are efficient, and in fact, \pzrlmf~has the same computational complexity as zooming for MABs.

5. In Section~\ref{sec:sim}, we demonstrate the applicability of our framework via a transmission scheduling problem. This MDP is neither tabular nor linear; however, the optimal policy is known to belong to a known class of policies that can be described by finitely many parameters. Simulation results show the practical relevance of our algorithms.

\subsection{Past Works}
\textit{\underline{Lipschitz episodic MDPs}}: 
\citet{domingues2021kernel} uses smoothing kernels to estimate the transition kernel, and obtains a regret upper bound with $\deff = 2d+1$.~\citet{cao2020provably} performs adaptive discretization and zooming and achieves regret upper bound with $\deff = d_z+2$, where $d_z$ is the zooming dimension defined specifically for the episodic case.~\citet{sinclair2023adaptive} also obtains adaptivity gains with $\deff = d_z+d_\cS$ for a model-based algorithm.~The same work also shows a regret lower bound of $\Omega(K^{1 - (d_z+2)\inv})$.~This lower bound is worse than the $\cO(\sqrt{K})$ dependence that is achievable for the tabular case or the function approximation techniques. This is not surprising since Lipschitz MDPs are a broader class of MDPs~\citep{maran2024no, maran2024projection}; See Figure~\ref{fig:mdp_rel}. Even though function approximation techniques yield an $\cO(\sqrt{K})$ regret, this comes at the expense of a larger prefactor in the regret bound as compared to Lipschitz MDPs. Moreover, function approximation techniques are computationally inefficient unless the underlying MDP is linear, and the feature maps are \emph{known}. The knowledge of feature maps seems to be a restrictive assumption since learning features efficiently is an active topic in itself~\citep{modi2024model}.

\textit{\underline{Average reward RL}}:~Tabular MDPs are well-studied by now, and popular algorithms with a tight $\ctO(\sqrt{DSAT})$ regret bound exist~\citep{jaksch2010near,tossou2019near}; where $D$ is the MDP diameter.~In contrast, continuous MDPs have only recently gained attention.~\cite{wei2021learning} uses function approximation, in which the relative value function is a linear function of the features, and obtains a $\ctO(\sqrt{T})$ regret.~\citet{he2023sample} uses function approximation techniques and obtains a regret bound of $\ctO(\textit{poly}(d_E, B) \sqrt{d_F T})$, where $B$ is the span of the relative value function.~When the transition kernel of the underlying MDP is $\alpha$-H\"older continuous and infinitely often smoothly differentiable, then~\citet{ortner2012online} obtains a regret upper bound with $\deff = (2d + 2\alpha)/\alpha$.~\citet{kar2024provably} performs adaptive discretization and zooming and achieves regret upperbound with $\deff = 2d_\cS + d_z + 3$ for Lipschitz MDPs with compact state-action spaces.
\section{Problem Setup}\label{sec:prelim}
\textbf{Notation.} $\bN$ denotes the set of natural numbers.~Let $(\Omega, \cF)$ be a measurable space, and let $\mu: \cF \mapsto \bR$ be a signed measure, then we denote total variation norm~\citep{folland2013real} of $\mu$ by $\norm{\mu}_{TV}$, i.e., $\norm{\mu}_{TV} := \sup\flbr{\sum_i{|\mu(B_i)|}: \{B_i\}_i \subset \cF \mbox{ partitions } \Omega}$.~For $B \subseteq \cS$, $\diamc{B} := \sup_{s,s\up \in B}{\rho(s,s\up)}$.~We denote $a\wedge b$ the minimum, and $a \vee b$ the maximum of $a, b \in \bR$.~$\ceil{a}$ denotes the smallest integer that is larger than $a$ for $a \in \bR$.~In general, we use a superscript $(f)$ ($(b)$) to indicate that an object is associated with algorithm \pzrlmf~(\pzrlmb).

Let $\cM = (\cS, \cA, p, r)$ be an MDP, where the dimensions of the state-space $\cS$ and action-space $\cA$ are $d_\cS$ and $d_\cA$, respectively.~The spaces $\cS, \cA$ are endowed with metrics $\rho_\cS$ and $\rho_\cA$, respectively.~The space $\cS \times \cA$ is endowed with a metric $\rho$ that is sub-additive, i.e., we have, $\rho\br{(s,a),(s\up,a\up)} \leq \rho_\cS (s, s\up) + \rho_\cA (a, a\up)$, for all $(s, a), (s\up, a\up) \in \cS \times \cA$.~We let $\cS$ be endowed with Borel $\sigma$-algebra $\cB_\cS$.~The state and the action taken at time $t$ are denoted by $s_t,a_t$, respectively. The transition kernel is $p: \cS \times \cA \times \cB_\cS \to [0,1]$, i.e., $\bP\left(s_{t+1}\in B| s_t=s, a_t=a\right) = p(s,a,B)$, a.s., for all $(s,a,B) \in \cS \times \cA \times \cB_\cS, ~t \in \{0\} \cup \bN$, and is not known to the agent.~The reward function $r: \cS \times \cA \to [0,1]$ is a measurable map, and the reward earned by the agent at time $t$ is equal to $r(s_t,a_t)$.~A stationary deterministic policy is a measurable map $\phi: \cS \to \cA$ that implements the action $\phi(s)$ when the system state is $s$. Let $\Phi_{SD}$ be the set of all such policies.~The infinite horizon average reward for the MDP $\cM$ under a policy $\phi$ is denoted by $J_\cM(\phi)$, and is defined as,
\begin{align*}
    J_\cM(\phi) := \underset{t\to\infty}{\lim\inf}{\frac{1}{T} \bE\sqbr{\sum_{t=0}^{T-1}{r(s_t,\phi(s_t))}}}.
\end{align*}
The maximum average reward attainable with a set of policies $\Phi \subseteq \Phi_{SD}$ is denoted by $J\ust_{\cM,\Phi}$.~The regret of a learning algorithm $\psi$ w.r.t. a class of \textit{comparator} policies $\Phi$ until $T$ is defined as~\cite{rakhlin},
\begin{align}
    \cR_\Phi(T;\psi) &:= T J\ust_{\cM,\Phi} -  \sum_{t=0}^{T-1} r(s_t,a_t).\label{def:regret}
\end{align}
The current work derives an upper bound on $\cR_\Phi(T;\psi)$ in terms of the \emph{zooming dimension}, a joint complexity measure of class $\Phi$ and the MDP $\cM$ when the algorithms are only allowed to play policies from $\Phi$.~Note that if $\Phi$ contains an optimal policy, then $\cR_\Phi(T;\psi)$ is the usual regret.~An MDP is Lipschitz if it satisfies the following properties.
\begin{assum}[Lipschitz continuity]\label{assum:lip}
    \begin{enumerate}[(i)]
        \item \label{assum:lip_r} The reward function $r$ is $L_r$-Lipschitz, i.e., $\forall ~s, s\up \in \cS, a, a\up \in \cA$,
        \begin{align*}
            | r(s,a) - r(s\up,a\up) | \le L_r \rho\left((s,a),(s\up,a\up)\right).
        \end{align*}
        \item \label{assum:lip_p} 
        The transition kernel $p$ is $L_p$-Lipschitz, i.e., $\forall~ s, s\up \in \cS, a, a\up \in \cA$,
        \begin{align*}
            \norm{p(s,a,\cdot) - p(s\up, a\up, \cdot)}_{TV} \le L_p \rho\br{(s,a),(s\up,a\up)}.
        \end{align*}
    \end{enumerate}
\end{assum}
While studying infinite-horizon average reward MDPs, some sort of ergodicity assumption is required. In fact, uniform ergodicity is the weakest known sufficient condition that ensures efficient computation of an optimal policy~\citep{arapostathis1993discrete} even when the MDP is known.

\begin{assum}[Ergodicity]\label{assum:unif_ergodic}
    Let $\Phi \subseteq \Phi_{SD}$ be the comparator class of policies. The CMP $\{s_t\}_t$ that is induced by transition kernel $p$ under application of any $\phi \in \Phi$ is uniformly ergodic~\citep{douc2018markov}, that is, there exist two constants, $C \in (0,\infty)$ and $\alpha \in (0,1)$ and for every $\phi \in \Phi$, there exists a unique distribution $\mu\uc{\infty}_{\phi,p}$ such that
    \begin{align}
        \norm{\mu\uc{t}_{\phi,p,s} - \mu\uc{\infty}_{\phi,p}}_{TV} \le C \alpha^t, ~\forall s \in \cS, t \in \{0\} \cup \bN,
    \end{align}
    where $\mu\uc{t}_{\phi,p,s}$ denotes the distribution of $s_t$ given $s_0=s$.
\end{assum}

We call $\mu\uc{\infty}_{\phi,p}$ as the stationary distribution of the CMP induced by $p$ under the application of policy $\phi$.~We need $\Phi$ to be endowed with an appropriate metric $\rho_\Phi$ such that $\Phi$ is bounded and $J_\cM$ is a Lipschitz function on $\Phi$.~In fact, under Assumption~\ref{assum:lip} and Assumption~\ref{assum:unif_ergodic}, $J_\cM$ is Lipschitz w.r.t. the metric $\rho_\Phi$ set equal to the metric induced by $\infty$ norm. Further, $J_\cM$ is Lipschitz w.r.t. a weighted distance measure too under a mild technical condition.~We show this two results in Theorem~\ref{thm:lip_avg_reward}.~For a discussion on the metric space of policies, see Appendix~\ref{app:pol_space}.~We next define the zooming dimension $d^\Phi_z$, a joint complexity measure for the problem instance and the comparator policy class $\Phi$.

\noindent\textbf{Zooming dimension.}~Suboptimality of a policy $\phi$ w.r.t. $\Phi$ is defined as $\Delta_\Phi(\phi) := J\ust_{\cM,\Phi} - J_\cM(\phi)$.~Define the sets of policies $\Phi_\gm := \{\phi \in \Phi~|~ \Delta_\Phi(\phi) \in (\gm, 2\gm]\}$, and $\Phi_{\leq\gm} := \{\phi \in \Phi~|~ \Delta_\Phi(\phi) \leq \gm\}$.~Then, the zooming dimension of the problem given the policy space $\Phi$ is defined as
\begin{align}\label{def:zoomingdim}
    d^\Phi_z := \inf \Big\{&d\up > 0 ~|~ \cN_{\frac{\gm}{c_z}}\br{\Phi_\gm} \leq c_{z_1} \gm^{-d\up}, \mbox{ and }\notag\\
    &\cN_{\frac{\gm}{c_z}}\br{\Phi_{\leq\gm}} \leq c_{z_2} \gm^{-d\up},~\forall  \gm \geq 0 \Big\},
\end{align}
where $\cN_{\gm}\br{\Phi\up}$ denotes the $\gm$-covering number~(Definition~\ref{def:sizeofset}) of $\Phi\up \subseteq \Phi$ w.r.t. metric $\rho_\Phi$, $c_{z_1}$ and $c_{z_2}$ are problem-dependent constants, $c_z := 2(\max{\{2, C_{ub}\}} + L_J)$.~$C_{ub}$ is a problem-dependent constants that are defined in \eqref{def:C_ub}.
\begin{remark}
    We note that even if the policy class is high-dimensional, the zooming dimension could be small, as it is a measure of the size of the set of near-optimal policies.
\end{remark}

\section{Algorithm}\label{sec:algo}
We propose a model-free algorithm~\pzrlmf~and a model-based algorithm~\pzrlmb.~Both the algorithms combine policy-based zooming with the principle of optimism in the face of uncertainty~\cite{lattimore2020bandit}.~They maintain a set of active policies, compute their UCB indices, and then play an active policy with the highest UCB index in the current episode.~Its zooming component zooms in and activates only those policies from $\Phi$, for which it is not possible to generate a good estimate of its performance using the performance estimates of nearby active policies.~However, they differ in the way in which they compute UCB indices and activate new policies and hence are discussed separately.~The algorithms are summarized in Algorithm~\ref{algo:pzrl}.
\subsection{\pzrlmf}
\noindent\textbf{Policy Diameter.} Diameter at time $t$ is defined as,
\begin{align}\label{def:diamf}
    &\diamf{t}{\phi} := \notag \\
    & \frac{C}{1 - \alpha} \br{\sqrt{\frac{c\uc{f}_d \log{\br{\frac{T}{\delta}}}}{1 \vee N_t(\phi)}} + \frac{1 + K_t(\phi)}{1 \vee N_t(\phi)}}, \phi \in \Phi,
\end{align}
where $N_t(\phi)$ is the number of plays of the policy $\phi$ until time $t$, $K_t(\phi)$ is the number of episodes that began before time $t$, and in which $\phi$ was played, while $c\uc{f}_d$ is a constant that satisfies \eqref{def:cdf}.

\noindent\textbf{Active Policies.}~$\Phi^{act.}_t$ denotes the set of policies that are active at time $t$.~Define the following ball in the policy space, 
\begin{align}
    B_{\phi,t} := \flbr{\phi\up \in \Phi: \rho_\Phi(\phi\up, \phi) \leq \diamf{t}{\phi}}. \label{def:pol_ball}
\end{align}
Since the confidence ball associated with a policy shrinks when it is played, in the possible event that $\cup_{\phi \in \Phi^{act.}_t}{B_t(\phi)}$ does not cover the set $\Phi$ anymore, the proposed algorithm activates a new policy to ensure that the union of confidence balls of the active policies covers $\Phi$.~Thus, \pzrlmf~possesses the \emph{covering invariance} property, i.e., $\cup_{\phi \in \Phi^{act.}_t}{B_{\phi,t}}$ covers $\Phi$ at all times.

\noindent\textbf{Model-free UCB Index.}~Let $\phi_t$ denote the policy played at time $t$. The UCB index at time $t$ is defined as
\begin{align}
    \idxf_t(\phi) &:= \frac{1}{N_t(\phi)}\sum_{i=0}^{t-1}{\ind{\phi_i = \phi} r(s_i,\phi(s_i))} \notag\\
    &\quad+ (1 + L_J) ~\diamf{t}{\phi},~ \phi \in \phi^{act.}_t, \label{def:index_mf}
\end{align}
where $L_J$ is the Lipschitz constant associated with $J_{\cM}$.
\begin{algorithm}[t]
    \caption{Policy Zooming for RL~(\pzrlmf/\pzrlmb)}
    \label{algo:pzrl}
    \begin{algorithmic}
        \STATE {\bfseries Input} Horizon $T$, confidence parameter $\delta$, ergodicity coefficient $\alpha$ and policy class $\Phi$
        \STATE {\bfseries Initialize} $h=0$, $k=0$, $\Phi^{act.}_0 = \flbr{}$.
        \FOR{$t= 0$ to $T-1$}
            \STATE Update the set of active policies $\Phi^{act.}_t \subset \Phi$
            \IF{$h \geq H_k$}
                \STATE $k \leftarrow k+1$, $h \leftarrow 0$
                \STATE For every $\phi \in \Phi^{act.}_t$ compute $\text{Index}_t(\phi)$, where\\$\text{Index}_t(\phi) = \idxf_t(\phi)$~\eqref{def:index_mf} if \pzrlmf, \\$\text{Index}_t(\phi) = \idxb_t(\phi)$~\eqref{def:index_mb} if \pzrlmb.
                \STATE Choose $\phi\uc{k} \in \arg\max_{\phi \in \Phi^{act.}_t}{\text{Index}_t(\phi)}$.
                \STATE $H_k = 1 \vee N_t(\phi_k)$
            \ENDIF
            \STATE $h \leftarrow h+1$
            \STATE Play $a_t = \phi\uc{k}(s_t)$, observe $s_{t+1}$ and receive $r(s_t, a_t)$.
        \ENDFOR
	\end{algorithmic}
\end{algorithm}

\subsection{\pzrlmb}
We assume $\cS$ to be bounded for \pzrlmb. \pzrlmb~maintains an adaptive partition of the state space $\cS$ for each active policy. We use $\cP_{\phi,t}$ to denote the state partition corresponding to policy $\phi$ at time $t$; see Appendix~\ref{app:adap_disc} for more details on the procedure to create these partitions.~Loosely speaking, as time progresses, $\cP_{\phi,t}$ is finer in those regions of $\cS$ that have been visited relatively more number of times while playing $\phi$.~$\cP_{\phi,t}$ consists of a certain type of subsets of $\cS$ called cells~(Definition~\ref{def:cell}).~The cells comprising $\cP_{\phi,t}$ are called active cells at time $t$ corresponding to policy $\phi$.~Let $q_{\phi,t}\inv(s)$ be the active cell corresponding to $\phi$ at time $t$ that contains the state $s$.

\noindent\textbf{Policy Diameter.}~The model-based policy diameter for $\phi \in \Phi$ is defined as follows:
\begin{align}\label{def:diamb}
    \diamb{t}{\phi} := \int_{\cS}{\diamc{q_{\phi,t}\inv(s)} \mu\uc{\infty}_{\phi,p}(ds)}.
\end{align}
Policy balls for \pzrlmb~are defined similar to~\eqref{def:pol_ball}, replacing $\diamf{t}{\phi}$ with $\diamb{t}{\phi}$.~Similar to \pzrlmf, \pzrlmb~maintains a set of active policies, $\Phi^{act.}_t$ that satisfies the covering invariance property.

\textit{Approximate Diameter}: The agent cannot compute $\diamb{t}{\phi}$ since it does not know $\mu\uc{\infty}_{\phi,p}$.~However, the agent can compute a ``tight'' lower bound of $\diamb{t}{\phi}$, which can then be used in \eqref{def:pol_ball} in lieu of $\diamb{t}{\phi}$. This increases the regret upper bound only by a constant factor. This is discussed in Appendix~\ref{app:eval_diam}.

\noindent\textbf{Model-based UCB Index.}~\pzrlmb~evaluates the UCB indices of the active policies by using estimate of the transition kernel.~The algorithm constructs a set of plausible discretized transition kernels using its estimate~\eqref{def:confball}.~A curated bias term is added to the discretized reward function in order to overcome the discretization error. Then, the UCB indices are computed via iterations~\eqref{eq:epvib}, which are similar to the policy evaluation algorithm.~Computation of the UCB index involves the following three steps:

(i) Estimating the Transition Kernel: Denote $S_{\phi,t}$ to be the set of representative points of the cells in $\cP_{\phi,t}$, and denote $\bar{S}_{\phi,t}$ to be the set of representative points of all the cells of size of the smallest cell in $\cP_{\phi,t}$. At time $t$, for every active policy $\phi$, \pzrlmb~constructs the empirical transition distribution, $\hat{p}\uc{d}_{\phi,t}(s,\cdot)$~\eqref{def:p_hat_disc} with the width of the bins set equal to the diameter of $q\inv_{\phi,t}(s)$ for every $s \in S_{\phi,t}$. Then a continuous extension of $\hat{p}\uc{d}_{\phi,t}$, $\hat{p}_{\phi,t}$ is computed~\eqref{def:p_hat}. $\hat{p}_{\phi,t}$ is again discretized with the width of the bins set equal to the diameter of the smallest active cell. This discrete estimate is denoted by $\wp_{S_{\phi,t} \to \bar{S}_{\phi,t},\hat{p}_{\phi,t}}$.~For a detailed discussion on the estimation of the transition kernel, see Appendix~\ref{subapp:discker}.


(ii) Confidence Ball: For a policy $\phi$ and a representative state $s \in S_{\phi,t}$, the confidence radius associated with the estimate $\wp_{S_{\phi,t} \to \bar{S}_{\phi,t},\hat{p}_{\phi,t}}$ is defined as follows,
\begin{align}\label{def:eta_k}
    \eta_{\phi,t}(s) :=& 3 \br{\frac{c\uc{b}_d \log{\br{T\delta\inv}}}{\sum_{i=0}^{t-1}{\ind{s_i \in {q_{\phi,t}\inv(s)}_i}}}}^\frac{1}{d_\cS + 2} \notag\\
    & + \br{3 (1 + L_\phi) L_p + C_p}~\diamc{q_{\phi,t}\inv(s)},
\end{align}
where $c\uc{b}_d>0$ is a constant that is discussed in Lemma~\ref{lem:conc_ineq}, $C_p$ is as described in Assumption~\ref{assum:bdd_der}, and $L_\phi$ is the Lipschitz constant associated with $\phi$, i.e., for all $s, s\up \in \cS$, $\rho_\cA(\phi(s),\phi(s\up)) \leq L_\phi \rho_\cS(s,s\up)$.~It follows from the rule used for activating a new cell (Definition~\ref{def:activationrule}) that we have,
\begin{align}
    \eta_{\phi,t}(s) \leq C_{\eta,\phi}~ \diamc{q_{\phi,t}\inv(s)}, \label{relate:confrad_diam}
\end{align}
for every $s \in S_{\phi,t}$, where $C_{\eta,\phi}:= 3\br{1 + (1 + L_\phi) L_p}$.~Let $\Te_{\phi,t}$ denote the set of all possible discretized transition kernels that describe outgoing transition probabilities from points in $S_{\phi,t}$, with a support on the discrete state space $\bar{S}_{\phi,t}$.~We define a set of transition probability kernels associated with $\wp_{S_{\phi,t} \to \bar{S}_{\phi,t},\hat{p}_{\phi,t}}$ as follows,
\begin{align}
    \cC_{\phi,t} := &\big\{\te \in \Te_{\phi,t} \mid \norm{\te(s,\cdot) - \wp_{S_{\phi,t} \to \bar{S}_{\phi,t},\hat{p}_{\phi,t}}(s,\cdot)}_1 \notag\\
    &\leq \eta_{\phi,t}(s) \mbox{ for every } s \in S_{\phi,t}\big\},\label{def:confball}
\end{align}

(iii) Computing the UCB Indices of Active Policies: Let us fix a time $t$.~To obtain the UCB index of a policy $\phi \in \Phi$, we perform the following iterations,
\begin{align}
    \ovl{V}^{\phi,t}_0(s) &= 0,\notag\\
    \ovl{V}^{\phi,t}_{i+1}(s) &= r(s,\phi(s)) + (1+L_\phi) L_r~ \diamc{q_{\phi,t}\inv(s)} \notag\\
    &\quad + \max_{\te \in \cC_{\phi,t}}{\sum_{s\up \in \bar{S}_{\phi,t}}{\te(q(q\inv_{\phi,t}(s)),s\up) \ovl{V}^{\phi,t}_{i}(s\up)}},\label{eq:epvib}
\end{align}
$s \in \bar{S}_{\phi,t}$, $i \in \bZ_+$. The difference of two consecutive iterates of \eqref{eq:epvib} is shown to converge in Lemma~\ref{lem:conv_eqpvi}. We define the UCB indices as follows,
\begin{align}
    \idxb_t(\phi) &:= \lim_{i \to \infty}{\br{\ovl{V}^{\phi,t}_{i+1}(s) - \ovl{V}^{\phi,t}_{i}(s)}} \notag \\
    &\quad + L_J~\diamb{t}{\phi}, \label{def:index_mb}
\end{align}
for any $s \in \bar{S}_{\phi,t}$.

\begin{remark}
    Similar to the zooming algorithm for bandits~\citep{kleinberg2019bandits}, we assume access to an oracle that takes as input a finite collection of open balls, and then either declares that they cover $\Phi$, or outputs a point that is uncovered.~In general, such an oracle may not be computationally efficient. However, when $\Phi$ has a finite dimensional parameterization, we can perform a grid search.
\end{remark}
\section{Regret Analysis}\label{sec:regret}
In this section, we present our main results Theorem~\ref{thm:reg_ub_mf} and Theorem~\ref{thm:reg_ub_mb}, that yield upper bounds on the regret of \pzrlmf~and \pzrlmb, respectively.~Before presenting these, we first show that the average reward function $J_\cM(\cdot)$ is a Lipschitz function of the policies w.r.t.~the sup-norm distance. Furthermore, under a mild assumption, it is also a Lipschitz function of the policies w.r.t. a weighted distance measure.~This result provides an important insight on choosing an appropriate norm in the definition of zooming dimension. Define
\begin{align*}
    \rho_{\Phi,\infty}(\phi_1, \phi_2) := \sup_{s \in \cS}{\rho_\cA(\phi_1(s),\phi_2(s))},~\forall \phi_1, \phi_2 \in \Phi.
\end{align*}
Consider a probability measure $\nu$ on $(\cS,\cB_\cS)$. Define the metric,
\begin{align}\label{def:policy_metric}
    \rho_{\Phi,\nu}(\phi_1,\phi_2) := \int_{\cS}{\rho_\cA(\phi_1(s), \phi_2(s))~d \nu(s)},
\end{align}
\begin{thm}\label{thm:lip_avg_reward}
    Let the MDP $\cM$ satisfy Assumptions~\ref{assum:lip} and ~\ref{assum:unif_ergodic}.~(i) Then, the infinite horizon average reward is $L_{J,\infty}$-Lipschitz w.r.t. the metric $\rho_{\Phi,\infty}$, i.e., for $\phi_1, \phi_2 \in \Phi$ we have,
    \begin{align*}
        |J_\cM(\phi_1) - J_\cM(\phi_2)| \leq L_{J,\infty} \rho_{\Phi,\infty}(\phi_1, \phi_2),
    \end{align*}
    where,
    \begin{align}
        L_{J,\infty} := L_r + \frac{L_p}{2(1 - \alpha)} \br{\ceil{\log_{\frac{1}{\alpha}}{(C)}} + 1}. \label{def:LJup}
    \end{align}
    (ii) Furthermore, if $\mu^{(\infty)}_{\phi,p}(\xi) \leq \kappa \nu(\xi),~\forall \xi \in \cB_\cS, \phi \in \Phi$, for some probability measure $\nu$ and a constant $\kappa > 0$, then $J_{\cM}(\cdot)$ is $L_{J,\nu}$-Lipschitz w.r.t. the metric $\rho_{\Phi,\nu}$, i.e., for $\phi_1, \phi_2 \in \Phi$ we have,
    \begin{align*}
        |J_\cM(\phi_1) - J_\cM(\phi_2)| \leq L_{J,\nu} \rho_{\Phi,\nu}(\phi_1, \phi_2),
    \end{align*}
    where,
    \begin{align}
        L_{J,\nu}:= \kappa L_{J,\infty}. \label{def:LJ}
    \end{align}
\end{thm}
The above theorem is proved in Appendix~\ref{app:mdp_prop}.~The next two theorems are the main results of this work, and bound the regrets of \pzrlmf~and \pzrlmb.
\begin{thm}\label{thm:reg_ub_mf}
    If the MDP $\cM$ satisfies Assumptions \ref{assum:lip} and~\ref{assum:unif_ergodic}, then with a probability at least $1 - \delta$, the regret of~\pzrlmf, i.e. $\cR_{\Phi}(T;\pzrlmf)$, is upper-bounded as $\ctO(T^{1 - \deff\inv})$ where $\deff = d^\Phi_z + 2$.
\end{thm}
The following assumptions are required for the analysis of~\pzrlmb.

\begin{assum}[Bounded Radon-Nikodym derivative]\label{assum:bdd_der}
    The probability measures $\{p_\phi(s,\cdot)\}$ are absolutely-continuous w.r.t. the Lebesgue measure on $(\cS,\cB_\cS)$, with density functions given by $\{f_{\phi,s}\}$ for every $\phi \in \Phi$.~We assume that these densities satisfy 
    \begin{align*}
        \norm{\frac{\partial f_{\phi,s}(s_+)}{\partial s_+(i)}}_\infty \leq C_p, \forall s \in \cS, i = 1, 2, \ldots,  d_\cS,
    \end{align*}
    where $s_+ = (s_+(1),s_+(2),\cdots,s_+(d_\cS))$.
\end{assum}

\begin{assum}\label{assum:stn_dist}
    There exists $\kappa\up > 0$ such that for every $\zeta \subseteq \cB_\cS$, $\mu^{(\infty)}_{\phi,p}(\zeta) \ge \kappa\up \lambda(\zeta)$, where $\lambda$ is the Lebesgue measure~\citep{billingsley2017probability} on $(\cS, \cB_\cS)$. 
\end{assum}

We use $\Phi_{\text{Lip.}}$ to denote the class of Lipschitz policies. Note that it is a broad class; for example, it contains a set of continuous policies.
\begin{thm}\label{thm:reg_ub_mb}
    Let $\Phi\subseteq \Phi_{\text{Lip.}}$.~If the MDP $\cM$ satisfies Assumptions \ref{assum:lip},~\ref{assum:unif_ergodic}, \ref{assum:bdd_der} and~\ref{assum:stn_dist}, then with a probability at least $1 - \delta$, the regret of~\pzrlmb, i.e. $\cR_{\Phi}(T;\pzrlmb)$, is upper-bounded as $\ctO(T^{1 - \deff\inv})$ where $\deff = 2 d_\cS + d^\Phi_z + 3$.
\end{thm}
The detailed proofs of the above two results are delegated to Appendix~\ref{app:regret}. Here, we provide a generic \emph{proof sketch}:

\begin{proof}[Proof sketch]
    We decompose the regret as follows,
    \begin{align*}
        \cR_\Phi(T;\psi) &= \underbrace{\sum_{k=1}^{K(T)}{\sum_{t = \tau_k}^{\tau_{k+1}-1}{J_{\cM,\Phi}\ust - J_{\cM}(\phi_k)}}}_{(a)} \\
        &+ \underbrace{\sum_{k=1}^{K(T)}{\sum_{t=\tau_k}^{\tau_{k+1}-1}{J_{\cM}(\phi_k) - r(s_t,\phi_k(s_t))}}}_{(b)},
    \end{align*}
    where $K(T)$ denotes the total number of episodes till time $T$.~We bound the terms (a) and (b) separately.
    
    \textbf{Bounding} (a): This term is further decomposed into the sum of the regrets arising due to playing policies from the sets $\Phi_\gm$, where $\gm$ assumes the values $2^{-i},~i = 1, 2, \ldots, \ceil{\log{\br{1/\eps}}}$, and $\eps = T^{-\deff\inv}$. Cumulative regret arising from playing policies not in the set $\cup_{i=1}^{\ceil{\log (1\slash \eps)} } \Phi_{2^{-i}}$ is bounded by $\eps T$. Regret due to playing policies from $\Phi_\gm$ is bounded in the following three steps:
        
        (1) First, we derive a condition under which a $\gm$-suboptimal policy is no longer played.
        
        (2) Then, we deduce an upper bound of the number of plays of a policy $\phi$ in terms of its suboptimality gap by concluding that the condition stated in (1) holds when $\phi$ has been played sufficiently many times.
        
        (3) Then, we establish an upper bound on the number of policies that are activated by the algorithms from $\Phi_\gm$.
        
    The product of two upper bounds discussed in (2) and (3), when multiplied with $2 \gm$, yields regret from playing policies in $\Phi_\gm$. We then add these regret terms corresponding to different sets $\Phi_{\gamma}$, where $\gm = 2^{-i}$ and $i=1,2,\ldots,\ceil{\log{\br{1/\eps}}}$; to this we add the regret arising due to playing policies with suboptimality less than $\eps$, which is bounded by $\eps T$.
    
    \textbf{Bounding} (b): Upper bound on the term $(b)$ is derived in Proposition~\ref{prop:fluc_ub}, and shows that we must pay a constant penalty in regret each time we switch to playing a different policy; hence this term is bounded as $\cO(K(T))$. Then we show that the total number of episodes for \pzrlmf~as well as~\pzrlmb~are bounded above by $\cO\Big(T^{d^\Phi_z/ \deff}\Big)$, where $\deff = d^\Phi_z + 2$ and $\deff = 2 d_\cS + d^\Phi_z + 3$, respectively.
        
    We obtain the desired regret bound after summing the upper bounds on (a) and (b).
\end{proof}

\begin{remark}[Discontinuous Policies]
    The regret analysis of~\pzrlmb~extends to policies with discontinuities, provided all discontinuities lie on the boundaries of active cells. This condition can be enforced by defining the cells in such a way that the policy discontinuities lie only on the cell boundaries. For simplicity, we define the cells using dyadic cubes in our exposition~(Definition~\ref{def:cell}).
\end{remark}

\begin{remark}[Regarding Assumption~\ref{assum:stn_dist}]
    We need Assumption~\ref{assum:stn_dist} in order to ensure a sufficiently large number of visits to regions of $\cS$ which have been assigned positive measure under the stationary distribution. This condition ensures that the model-based UCB bonuses diminish to $0$ when the corresponding policy is played infinitely often. Similar assumptions appear in prior works. For example,~\citet{ormoneit2002kernel} assumes the transition kernel has a strictly positive Radon-Nikodym derivative. \citet{wang2023optimal} requires the $m$-step transition kernel to be lower bounded.~In yet another work, \citet{wei2021learning} bounds the regret for average reward RL that uses linear function approximation. They assume that under every policy, the integral of cross-product of the feature vectors w.r.t. the stationary measure has all the eigenvalues bounded away from zero. This assumption ensures that upon playing any policy, the confidence ball shrinks in each direction. All serve a purpose analogous to Assumption~\ref{assum:stn_dist}.
\end{remark}
The next result quantifies upper-bound on $\deff$ for an important class of policies; its proof is delegated to Appendix~\ref{app:pol_space}.
\begin{cor}[Finite parameterization]\label{cor:param_pol}
    We now consider a set $\Phi$ that consists of policies that have been parameterized by finitely many parameters from the set $W \subset \bR^{d_W}$. For each $w \in W$, let $\phi(\cdot; w): \cS \to \cA$ be the policy parameterized by $w$.~Assume that the policies satisfy $L_W \rho_\Phi(\phi(\cdot; w),\phi(\cdot; w\up)) \ge \norm{w - w\up}_{2}$ for all $w, w\up \in W$. We have $\deff \leq d_W + 2$ for \pzrlmf~and $\deff \leq 2 d_\cS + d_W + 3$ for~\pzrlmb.
\end{cor}
\begin{cor}[Bi-Lipschitz MDPs]\label{cor:sqrt_reg}
    Consider a bi-Lipschitz  MDP, i.e., the average reward function $J_\cM: \Phi \to \bR$ satisfies the following properties: there exist two constants $L_1 \geq L_2 > 0$ such that for every $\phi_1, \phi_2$
    \begin{align*}
        L_2 \rho_\Phi(\phi_1,\phi_2) \leq |J_\cM(\phi_1) - J_\cM(\phi_2)| &\leq L_1 \rho_\Phi(\phi_1,\phi_2).
    \end{align*}
    Then, the regret of \pzrlmf~w.r.t. the policy class scales as $\ctO(\sqrt{T})$ on a high probability set.
\end{cor}
We note that the assumption made in Corollary~\ref{cor:sqrt_reg} commonly made in continuum bandits literature such as~\citet{cope2009regret,yu2011unimodal,combes2020unimodal}.




\section{Simulations}\label{sec:sim}
We evaluate the proposed algorithms on the following two systems.

\textbf{Transmission scheduling for remote estimation of a stochastic dynamic process.} Consider a process $\{x_t\}$ that evolves as $x_{t+1} = \beta x_t + w_t$, where $|\beta| < 1$ and $\{w_t\}$ is i.i.d., $w_t \sim \cN(0,1)$ for all $t$. A sensor observes $\{x_t\}$, encodes it into data packets and transmits them to a remote estimator across an unreliable wireless channel. $c_t \in \{0,1\}$ denotes the channel state at time $t$. $c_t = 1$ ($0$) denotes that the channel state is good (bad). $\{c_t\}$ is a Markov process with transition probabilities $p_{ij} := \bP(c_{t + 1} = j \mid c_t = i)$, $i,j \in \{0,1\}$, 
where $p_{01}, p_{11} > 0$. $a_t \in \{0,1\}$ denotes the decision made by the sensor regarding whether or not a packet transmission is attempted at time $t$; $a_t=1$ denotes that transmission is attempted. The estimator state $\hat{x}_t$ evolves as $\hat{x}_{t+1} = x_{t+1} c_t a_t + (1 - c_t a_t) \beta \hat{x}_t$.~The estimation error $\{e_t\}$ evolves as $e_{t+1} = (\beta e_t + w_t) - \beta c_t a_t e_t$.~The agent's estimate of $c_t$, denoted by $b_t$ can be updated recursively.~The actions $\{a_t\}$ are to be chosen so as to minimize the error with a minimal amount of transmission power. The agent earns a reward $r_t := -e_t^2 - \lambda a_t$, where $\lambda>0$ is the number of units of resource required for transmission.~\cite{dutta2023optimal} shows that a threshold policy is optimal in this setup, one which transmits only when $b_t$ exceeds a certain threshold (that is allowed to depend upon $e_t$). Hence, the optimal policy can be described by specifying the threshold curve, which in turn can be approximated by a curve with finitely many parameters. This problem does not fit into the class of Linear MDPs or Tabular MDPs. However, it can be shown that the average reward function is Lipschitz when the comparator policy class consists only of stable policies, and hence fits within our framework. We compare the empirical performance of the proposed algorithms, \pzrlmf~and~\pzrlmb~(Algorithm~\ref{algo:pzrl}) with that of a heuristic algorithm Policy UCB (Algorithm~\ref{algo:pucb}), which discretizes the policy space uniformly at time $t=0$ and plays the policy with the highest model-free UCB index from the set of finite set of policies in every episode.~For both \pzrlmf~and \pzrlmb, we use the following parameterization: $\phi(s;w) = \ind{w(1) + w(2)e_t < b_t}$, $w = (w(1),w(2)) \in [1,3]\times[-1,-0.01]$. We plot the cumulative reward minus the average performance of the policy that suggests transmission irrespective of system state, averaged over $50$ runs in Figure~\ref{fig:ts_lr}, and observe that \pzrlmf~and \pzrlmb~outperform the fixed discretization-based algorithm, Policy UCB. 
    
\textbf{Continuous RiverSwim.}~We modify the RiverSwim MDP~\citep{strehl2008analysis} to obtain its continuous version.~The state $s_t$ describes the location of the agent in the river and evolves as follows upon the application of action $a_t$ at time $t$:

\begin{align*}
    s_{t+1} =
    \begin{cases}
        (0 \vee (s_t - \frac{1}{2}(1 + \frac{w_t}{2}))) \wedge 6 &\mbox{w.p. } \frac{2(1-a_t)}{5}\\
        s_t &\mbox{w.p. } 0.2\\
        (0 \vee (s_t + \frac{1}{2}(1 + \frac{w_t}{2}))) \wedge 6 &\mbox{w.p. } \frac{2(1+a_t)}{5},
    \end{cases}
\end{align*}

where $\{w_t\}$ is i.i.d. and $w_t \sim \cN(0,0.5)$, and $t \in \{0\} \cup \bN$. Here, $\cS = [0,6]$ and $\cA = [0,1]$. The reward function is given by $r(s,a) = 0.005(((s-6)/6)^4 + ((a-1)/2)^4) + 0.5((s/6)^4 + ((a+1)/2)^4)$.~Note that the policy that chooses the action $1$ at all times, irrespective of the current state, is optimal.~For both \pzrlmf~and \pzrlmb, we use the following parameterizations:
\begin{enumerate}
    \item $1$ parameter: $\phi(s;w) = w$, $w \in [-1,1]$.
    \item $2$ parameters: $\phi(s;w) = w(1) + w(2)s$, $w = (w(1),w(2)) \in [-1,1]\times[-0.5,0.5]$.
    \item $3$ parameters: $\phi(s;w) = w(1) + w(2) s + w(3) s^2$, $w = (w(1),w(2),w(3)) \in [-1,1]\times[-0.5,0.5]^2$.
\end{enumerate}
We compare the empirical performance of \pzrlmf~and~\pzrlmb~(Algorithm~\ref{algo:pzrl}) with that of ZoRL~\cite{kar2024provably}, UCRL2~\citep{jaksch2010near}~and TSDE~\citep{ouyang2017learning}.~Since these competitor policies are designed for finite state-action spaces, we apply them on a uniform discretization of $\cS \times \cA$.~We plot the logarithm of the cumulative regret averaged over $50$ runs for the Continuous RiverSwim environment in Figure~\ref{fig:crs_lr}, and observe that \pzrlmf~and \pzrlmb~outperforms every other algorithm, and amongst \pzrlmf~and \pzrlmb,~\pzrlmb~has the edge over \pzrlmf.~Policy classes with $2$ and $3$ parameterizations outperform the single parameter policy classes for both the proposed algorithms.

\begin{figure}[!t]
    \centering
    \subfigure[Transmission Scheduling]{
    \label{fig:ts_lr}
    \includegraphics[height=1.1in]{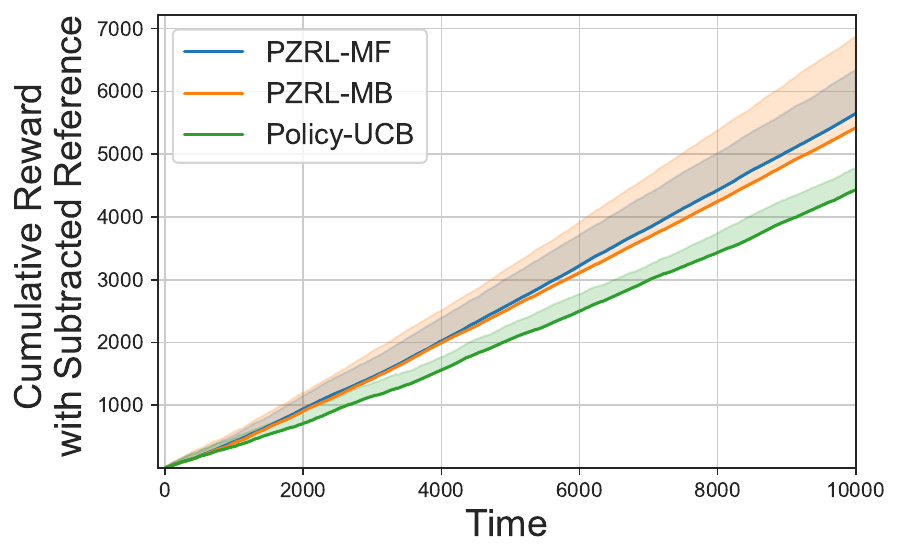}}
    \subfigure[Continuous~RiverSwim]{
    \label{fig:crs_lr}
    \includegraphics[height=1.1in]{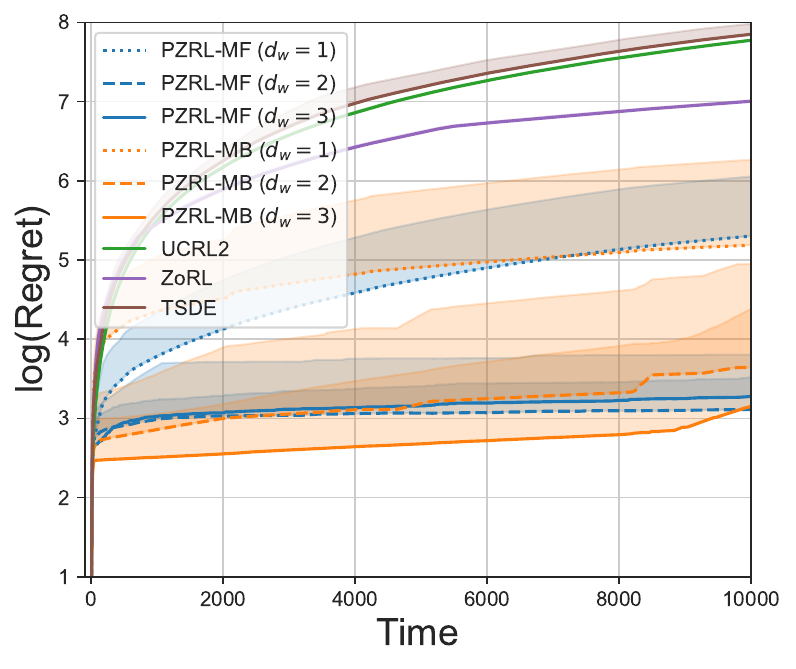}}
\end{figure}
\section{Conclusion}\label{sec:conc}
The central idea of zooming-based algorithms is to capitalize on the adaptive nature of the algorithm.~We identify an important problem encountered while employing zooming to infinite horizon average reward RL setup.~To rectify this, we define the zooming dimension for this setup in terms of coverings of the policy space, and also allow it to depend upon the class of policies $\Phi$ used by the agent.~We propose zooming-based algorithms \pzrlmf~and \pzrlmb, and prove that they exhibit adaptivity gains. We show that their regret can be bounded as $\ctO(T^{1-\deff\inv})$ where $\deff = d_z^\Phi + 2$ and $\deff = 2d_\cS + d_z^\Phi + 3$, respectively.~Simulation results confirm our theoretical findings.

\bibliography{refs}

@string{cdc = {{IEEE Conf. on Decision and Control}}}

@article{mitrophanov2005sensitivity,
	title={Sensitivity and convergence of uniformly ergodic Markov chains},
	author={Mitrophanov, A Yu},
	journal={Journal of Applied Probability},
	volume={42},
	number={4},
	pages={1003--1014},
	year={2005},
	publisher={Cambridge University Press}
}

@book{sutton2018reinforcement,
	title={Reinforcement learning: An introduction},
	author={Sutton, Richard S and Barto, Andrew G},
	year={2018},
	publisher={MIT press}
}

@inproceedings{abbasi2011regret,
  title={Regret bounds for the adaptive control of linear quadratic systems},
  author={Abbasi-Yadkori, Yasin and Szepesv{\'a}ri, Csaba},
  booktitle={Proceedings of the 24th Annual Conference on Learning Theory},
  pages={1--26},
  year={2011},
  organization={JMLR Workshop and Conference Proceedings}
}

@article{jaksch2010near,
  title={Near-optimal regret bounds for reinforcement learning},
  author={Jaksch, Thomas and Ortner, Ronald and Auer, Peter},
  journal={Journal of Machine Learning Research},
  volume={11},
  number={Apr},
  pages={1563--1600},
  year={2010}
}

@inproceedings{ouyang2017learning,
  title={Learning unknown {M}arkov decision processes: A {T}hompson sampling approach},
  author={Ouyang, Yi and Gagrani, Mukul and Nayyar, Ashutosh and Jain, Rahul},
  booktitle={Advances in Neural Information Processing Systems},
  pages={1333--1342},
  year={2017}
}

@book{folland2013real,
  title={Real analysis: modern techniques and their applications},
  author={Folland, Gerald B},
  year={2013},
  publisher={John Wiley \& Sons}
}

@book{puterman2014markov,
  title={{M}arkov decision processes: discrete stochastic dynamic programming},
  author={Puterman, Martin L},
  year={2014},
  publisher={John Wiley \& Sons}
}

@article{osband2014model,
  title={Model-based reinforcement learning and the eluder dimension},
  author={Osband, Ian and Van Roy, Benjamin},
  journal={Advances in Neural Information Processing Systems},
  volume={27},
  year={2014}
}

@article{ortner2012online,
	title={Online regret bounds for undiscounted continuous reinforcement learning},
	author={Ortner, Ronald and Ryabko, Daniil},
	journal={Advances in Neural Information Processing Systems},
	volume={25},
	year={2012}
}

@article{kakade2020information,
  title={Information theoretic regret bounds for online nonlinear control},
  author={Kakade, Sham and Krishnamurthy, Akshay and Lowrey, Kendall and Ohnishi, Motoya and Sun, Wen},
  journal={Advances in Neural Information Processing Systems},
  volume={33},
  pages={15312--15325},
  year={2020}
}

@book{lattimore2020bandit,
	title={Bandit Algorithms},
	author={Lattimore, Tor and Szepesv{\'a}ri, Csaba},
	year={2020},
	publisher={Cambridge University Press}
}

@article{arapostathis1993discrete,
	title={Discrete-time controlled {M}arkov processes with average cost criterion: a survey},
	author={Arapostathis, Aristotle and Borkar, Vivek S and Fern{\'a}ndez-Gaucherand, Emmanuel and Ghosh, Mrinal K and Marcus, Steven I},
	journal={SIAM Journal on Control and Optimization},
	volume={31},
	number={2},
	pages={282--344},
	year={1993},
	publisher={SIAM}
}

@book{hernandez2012further,
	title={Further topics on discrete-time Markov control processes},
	author={Hern{\'a}ndez-Lerma, On{\'e}simo and Lasserre, Jean B},
	volume={42},
	year={2012},
	publisher={Springer Science \& Business Media}
}

@article{strehl2008analysis,
  title={An analysis of model-based interval estimation for {M}arkov decision processes},
  author={Strehl, Alexander L and Littman, Michael L},
  journal={Journal of Computer and System Sciences},
  volume={74},
  number={8},
  pages={1309--1331},
  year={2008},
  publisher={Elsevier}
}

@book{meyn2012markov,
	title={{M}arkov chains and stochastic stability},
	author={Meyn, Sean P and Tweedie, Richard L},
	year={2012},
	publisher={Springer Science \& Business Media}
}

@article{abbasi2011improved,
	title={Improved algorithms for linear stochastic bandits},
	author={Abbasi-Yadkori, Yasin and P{\'a}l, D{\'a}vid and Szepesv{\'a}ri, Csaba},
	journal={Advances in neural information processing systems},
	volume={24},
	pages={2312--2320},
	year={2011}
}

@InProceedings{domingues2021kernel,
  title = 	 {Kernel-Based Reinforcement Learning: A Finite-Time Analysis},
  author =       {Domingues, Omar Darwiche and Menard, Pierre and Pirotta, Matteo and Kaufmann, Emilie and Valko, Michal},
  booktitle = 	 {Proceedings of the 38th International Conference on Machine Learning},
  pages = 	 {2783--2792},
  year = 	 {2021},
  editor = 	 {Meila, Marina and Zhang, Tong},
  volume = 	 {139},
  series = 	 {Proceedings of Machine Learning Research},
  publisher =    {PMLR}
}

@article{kleinberg2019bandits,
  title={Bandits and experts in metric spaces},
  author={Kleinberg, Robert and Slivkins, Aleksandrs and Upfal, Eli},
  journal={Journal of the ACM (JACM)},
  volume={66},
  number={4},
  pages={1--77},
  year={2019},
  publisher={ACM New York, NY, USA}
}

@book{billingsley2017probability,
  title={Probability and measure},
  author={Billingsley, Patrick},
  year={2017},
  publisher={John Wiley \& Sons}
}

@inproceedings{kleinberg2008multi,
  title={Multi-armed bandits in metric spaces},
  author={Kleinberg, Robert and Slivkins, Aleksandrs and Upfal, Eli},
  booktitle={Proceedings of the fortieth annual ACM symposium on Theory of computing},
  pages={681--690},
  year={2008}
}

@article{ormoneit2002kernel,
  title={Kernel-based reinforcement learning in average-cost problems},
  author={Ormoneit, Dirk and Glynn, Peter},
  journal={IEEE Transactions on Automatic Control},
  volume={47},
  number={10},
  pages={1624--1636},
  year={2002},
  publisher={IEEE}
}

@article{sinclair2023adaptive,
  title={Adaptive discretization in online reinforcement learning},
  author={Sinclair, Sean R and Banerjee, Siddhartha and Yu, Christina Lee},
  journal={Operations Research},
  volume={71},
  number={5},
  pages={1636--1652},
  year={2023},
  publisher={INFORMS}
}

@article{cao2020provably,
  title={Provably adaptive reinforcement learning in metric spaces},
  author={Cao, Tongyi and Krishnamurthy, Akshay},
  journal={Advances in Neural Information Processing Systems},
  volume={33},
  pages={9736--9744},
  year={2020}
}

@inproceedings{jin2020provably,
  title={Provably efficient reinforcement learning with linear function approximation},
  author={Jin, Chi and Yang, Zhuoran and Wang, Zhaoran and Jordan, Michael I},
  booktitle={Conference on Learning Theory},
  pages={2137--2143},
  year={2020},
  organization={PMLR}
}

@inproceedings{chowdhury2019online,
  title={Online learning in kernelized {M}arkov decision processes},
  author={Chowdhury, Sayak Ray and Gopalan, Aditya},
  booktitle={The 22nd International Conference on Artificial Intelligence and Statistics},
  pages={3197--3205},
  year={2019},
  organization={PMLR}
}

@inproceedings{ayoub2020model,
  title={Model-based reinforcement learning with value-targeted regression},
  author={Ayoub, Alex and Jia, Zeyu and Szepesvari, Csaba and Wang, Mengdi and Yang, Lin},
  booktitle={International Conference on Machine Learning},
  pages={463--474},
  year={2020},
  organization={PMLR}
}

@article{qian2019exploration,
	title={Exploration bonus for regret minimization in discrete and continuous average reward mdps},
	author={Qian, Jian and Fruit, Ronan and Pirotta, Matteo and Lazaric, Alessandro},
	journal={Advances in Neural Information Processing Systems},
	volume={32},
	year={2019}
}

@inproceedings{wei2021learning,
  title={Learning infinite-horizon average-reward MDPs with linear function approximation},
  author={Wei, Chen-Yu and Jahromi, Mehdi Jafarnia and Luo, Haipeng and Jain, Rahul},
  booktitle={International Conference on Artificial Intelligence and Statistics},
  pages={3007--3015},
  year={2021},
  organization={PMLR}
}

@inproceedings{he2023sample,
  title={Sample-efficient Learning of Infinite-horizon Average-reward MDPs with General Function Approximation},
  author={He, Jianliang and Zhong, Han and Yang, Zhuoran},
  booktitle={The Twelfth International Conference on Learning Representations},
  year={2023}
}

@book{van1996weak,
  title={Weak convergence},
  author={Van Der Vaart, Aad W and Wellner, Jon A and van der Vaart, Aad W and Wellner, Jon A},
  year={1996},
  publisher={Springer}
}

@inproceedings{wang2023optimal,
title={Optimal Sample Complexity for Average Reward {M}arkov Decision Processes},
author={Shengbo Wang and Jose Blanchet and Peter Glynn},
booktitle={The Twelfth International Conference on Learning Representations},
year={2024},
url={https://openreview.net/forum?id=jOm5p3q7c7}
}

@article{tossou2019near,
	title={Near-optimal optimistic reinforcement learning using empirical {B}ernstein inequalities},
	author={Tossou, Aristide and Basu, Debabrota and Dimitrakakis, Christos},
	journal={arXiv preprint arXiv:1905.12425},
	year={2019}
}

@article{mouhoubi2021perturbation,
  title={Perturbation and stability bounds for ergodic general state {M}arkov chains with respect to various norms},
  author={Mouhoubi, Zahir},
  journal={Le Matematiche},
  volume={76},
  number={1},
  pages={243--276},
  year={2021}
}

@article{nair2022r3m,
  title={R3m: A universal visual representation for robot manipulation},
  author={Nair, Suraj and Rajeswaran, Aravind and Kumar, Vikash and Finn, Chelsea and Gupta, Abhinav},
  journal={Confernce on Robot Learning},
  year={2023}
}

@article{kumar2021rma,
  title={Rma: Rapid motor adaptation for legged robots},
  author={Kumar, Ashish and Fu, Zipeng and Pathak, Deepak and Malik, Jitendra},
  journal={arXiv preprint arXiv:2107.04034},
  year={2021}
}

@article{maran2024no,
  title={No-Regret Reinforcement Learning in Smooth MDPs},
  author={Maran, Davide and Metelli, Alberto Maria and Papini, Matteo and Restell, Marcello},
  journal={The Forty-first International Conference on Machine Learning},
  year={2024}
}

@article{maran2024projection,
  title={Projection by Convolution: Optimal Sample Complexity for Reinforcement Learning in Continuous-Space MDPs},
  author={Maran, Davide and Metelli, Alberto Maria and Papini, Matteo and Restelli, Marcello},
  journal={The 37th Annual Conference on Learning Theory},
  year={2024}
}

@article{song2019efficient,
  title={Efficient model-free reinforcement learning in metric spaces},
  author={Song, Zhao and Sun, Wen},
  journal={arXiv preprint arXiv:1905.00475},
  year={2019}
}

@misc{rakhlin,
  author        = {Alexander Rakhlin and Karthik Sridharan},
  title         = {Lecture notes for STAT928: Statistical Learning and Sequential Prediction},
  month         = {September},
  year          = {2014},
howpublished = {\url{www.mit.edu/~rakhlin/courses/stat928/stat928_notes.pdf}}
}

@book{douc2018markov,
  title={Markov chains: Basic definitions},
  author={Douc, Randal and Moulines, Eric and Priouret, Pierre and Soulier, Philippe},
  year={2018},
  publisher={Springer}
}

@article{metivier1984applications,
  title={Applications of a {K}ushner and {C}lark lemma to general classes of stochastic algorithms},
  author={Metivier, Michel and Priouret, Pierre},
  journal={IEEE Transactions on Information Theory},
  volume={30},
  number={2},
  pages={140--151},
  year={1984},
  publisher={IEEE}
}

@inproceedings{kar2024provably,
  title={Provably Adaptive Average Reward Reinforcement Learning for Metric Spaces},
  author={Kar, Avik and Singh, Rahul},
  booktitle={Proceedings of the Forty-first Conference on Uncertainty in Artificial Intelligence},
  pages={1924--1964},
  year={2025},
  publisher={PMLR}
}

@article{hazan2022introduction,
  title={Introduction to online nonstochastic control},
  author={Hazan, Elad and Singh, Karan},
  journal={arXiv preprint arXiv:2211.09619},
  year={2022}
}

@article{wu2024computationally,
  title={Computationally Efficient RL under Linear Bellman Completeness for Deterministic Dynamics},
  author={Wu, Runzhe and Sekhari, Ayush and Krishnamurthy, Akshay and Sun, Wen},
  journal={CoRR},
  year={2024}
}

@article{cope2009regret,
  title={Regret and convergence bounds for a class of continuum-armed bandit problems},
  author={Cope, Eric W},
  journal={IEEE Transactions on Automatic Control},
  volume={54},
  number={6},
  pages={1243--1253},
  year={2009},
  publisher={IEEE}
}

@inproceedings{yu2011unimodal,
  title={Unimodal Bandits.},
  author={Yu, Jia Yuan and Mannor, Shie},
  booktitle={ICML},
  pages={41--48},
  year={2011}
}

@article{combes2020unimodal,
  title={Unimodal bandits with continuous arms: Order-optimal regret without smoothness},
  author={Combes, Richard and Prouti{\`e}re, Alexandre and Fauquette, Alexandre},
  journal={Proceedings of the ACM on Measurement and Analysis of Computing Systems},
  volume={4},
  number={1},
  pages={1--28},
  year={2020},
  publisher={ACM New York, NY, USA}
}

@article{bubeck2011x,
  title={X-Armed Bandits.},
  author={Bubeck, S{\'e}bastien and Munos, R{\'e}mi and Stoltz, Gilles and Szepesv{\'a}ri, Csaba},
  journal={Journal of Machine Learning Research},
  volume={12},
  number={5},
  year={2011}
}

@inproceedings{dutta2023optimal,
  title={Optimal scheduling policies for remote estimation of autoregressive Markov processes over time-correlated fading channel},
  author={Dutta, Manali and Singh, Rahul},
  booktitle={2023 62nd IEEE Conference on Decision and Control (CDC)},
  pages={6455--6462},
  year={2023},
  organization={IEEE}
}

@article{modi2024model,
  title={Model-free representation learning and exploration in low-rank mdps},
  author={Modi, Aditya and Chen, Jinglin and Krishnamurthy, Akshay and Jiang, Nan and Agarwal, Alekh},
  journal={Journal of Machine Learning Research},
  volume={25},
  number={6},
  pages={1--76},
  year={2024}
}


\appendix
\onecolumn
\section{Detailed Discussion on~\pzrlmb}
\label{app:pzrlmb}
In the first section of this Appendix, we provide definitions of some useful objects used in \pzrlmb~and discuss the adaptive discretization procedure.~In the second section, the estimate of the transition kernel is introduced.~We provide an algorithm to evaluate an approximation of the diameter of policies~\eqref{def:diamb} in the third section.~The next section contains a proof of the existence of the model-based UCB index of policies~\eqref{def:index_mb}.~Then, in the fifth section we prove a few properties of the model-based UCB index which is used in the regret analysis of \pzrlmb~later.~Lastly, in the sixth section, we show that the approximate diameter evaluated by the algorithm proposed in the second section is a tight lower bound of the true diameter.

\subsection{Adaptive Discretization}
\label{app:adap_disc}
We decompose the state space into ``cells'' to enable discretization.
\begin{defn}[Cells]\label{def:cell}
    A cell is intersection of the state space and a dyadic cube with vertices from the set $\{(N_1, N_2, \ldots, N_{d_\cS}) + 2^{-\ell}(v_1, v_2, \ldots, v_{d_\cS}): v_j \in \{0,1,\ldots,2^\ell\}, N_j \in \bN, j=1,2,\ldots,d_\cS\}$ with sides of length $2^{-\ell}$, where $\ell\in\bN$. The quantity $\ell$ is called the level of the cell. We also denote the collection of all the level $\ell$ cells by $\cP^{(\ell)}$.~For a cell $\zeta$, we let $q(\zeta)$ to be a point from $\zeta$ that is its unique representative point. $q\inv$ maps a representative point to the cell that the point is representing, i.e., $q\inv(z) = \zeta$ such that $q(\zeta) = z$.
\end{defn}
The cells are interrelated through a structure called the ``partition tree.''
\begin{defn}[Partition tree]\label{def:part_tree}
    A partition tree of depth $\ell$ is a tree in which
    (i) Each node at depth $m$, where $m \leq \ell$, of the tree is a cell of level $m$. (ii) If $\zeta$ is a cell of level $m$, where $m<\ell$, then, a) all the cells of level $m+1$ that collectively generate a partition of $\zeta$, are the child nodes of $\zeta$. The corresponding cells are called child cells, and we use $\textit{Child}(\zeta)$ to denote the child cells of $\zeta$. b) $\zeta$ is called the parent cell of these child nodes.~Any ancestor node of cell $\zeta$ is called an ancestor of $\zeta$.
\end{defn}
\pzrlmb~maintains a set of ``active cells'' for every active policy.~The following rule is used for activating and deactivating cells corresponding to an active policy $\phi$.
\begin{defn}[Activation rule]\label{def:activationrule}
    For a cell $\zeta$ denote, 
    \begin{align}
        N_{\max}(\zeta) := \frac{c\uc{b}_d 2^{d_\cS+2} \log{\br{\frac{T}{\delta}}}}{\diamc{\zeta}^{d_\cS+2}}, \mbox{ and }
        N_{\min}(\zeta) := \begin{cases}
            ~~0 &\mbox{ if } \ell(\zeta) = 0\\
            \frac{c\uc{b}_d \log{\br{\frac{T}{\delta}}}}{\diamc{\zeta}^{d_\cS+2}}, &\mbox{otherwise,}\label{Nmin}
        \end{cases}
    \end{align}
    where $c\uc{b}_d>0$ is a constant that is discussed in Lemma~\ref{lem:conc_ineq}, and $\delta \in (0,1)$ is the confidence parameter.~The number of visits to $\zeta$ while playing policy $\phi$ is denoted $N_{\phi,t}(\zeta)$ and is computed using the following rule:
    \begin{enumerate}
        \item The initial active partition for policy $\phi$, denoted by $\cP_{\phi,0}$ is a partition of $\cS$ made of cells of diameter $1$.
        \item $N_{\phi,t}(\zeta) := 0$ for every cell $\zeta$ if $\phi$ not active at time $t$.
        \item A cell $\zeta$ is said to be active if $N_{\min}(\zeta) \leq N_{\phi,t}(\zeta) < N_{\max}(\zeta)$.
        \item $N_{\phi,t}(\zeta)$ is defined as the number of times either $\zeta$, or any of its ancestors, has been visited while being active until time $t$ by the policy $\phi$, i.e.,
        \begin{align}\label{def:visitcounter}
            N_{\phi,t}(\zeta) &:= \sum_{i=0}^{t-1}{\ind{s_i \in \zeta_i, a_i = \phi(s_i)}},
        \end{align}
        where $\zeta_i$ is the unique cell that is active at time $i$ and satisfies $\zeta \subseteq \zeta_i$.
    \end{enumerate}
    Denote the set of active cells at time $t$ corresponding to the policy $\phi$ by $\cP_{\phi,t}$.
\end{defn}
Denote the set of the representative points of the cells in $\cP_{\phi,t}$ by $S_{\phi,t}$. $q\inv_{\phi,t}(s)$ denotes the cell in $\cP_{\phi,t}$ that contains $s$. Denote $\ell^{\max}_{\phi,t}$ to be the level of the smallest cell in $\cP_{\phi,t}$ and denote $\bar{S}_{\phi,t}$ to be the set of representative points of the cells in $\cP\uc{\ell^{\max}_{\phi,t}}$. 

\subsection{Discretized Transition Kernel Estimate}
\label{subapp:discker}
Let $\bar{S}$ be the set of representative points of a set of cells that partitions the state space. Then, for a continuous transition kernel $\tilde{p}: \cS \times \cS \to [0,1]$ and for any $S \subseteq \cS$, we define $\wp_{S \to \bar{S},\tilde{p}}(s,\cdot): S \mapsto [0,1]^{\bar{S}}$ as follows,
\begin{align}\label{def:disc_p}
    \wp_{S \to \bar{S},\tilde{p}}(s,s\up) := \tilde{p}(s,q\inv(s\up)),~\forall s \in S, s \in \bar{S}.
\end{align}
The kernel $\wp_{S \to \bar{S},\tilde{p}}$ can be viewed as a discretization of $\tilde{p}$. Denote the transition kernel induced by the policy $\phi$ by $p_\phi$. The discretized transition kernel $\wp_{\cS \to \bar{S}_{\phi,t},p_\phi}$ will be of interest to us.

Let $\tilde{S}_{\phi,t}(s)$ denote the set of representative states of the cells of level $\ell(q\inv_t(s))$, i.e., $\tilde{S}_{\phi,t}(s) := \{q(\xi) : \xi \in \cP\uc{\ell(q\inv_t(s))}\}$. Now, we are all set to introduce the discrete estimate of the transition probability kernel of the CMP induced under the application of policy $\phi$.
\begin{align}\label{def:p_hat_disc}
    \hat{p}\uc{d}_{\phi,t}(s,s\up) := \frac{\sum_{i=0}^{t-1}{\ind{(s_i, s_{i+1}) \in {q\inv_{\phi,t}(s)}_i \times q\inv(s\up), \phi_i = \phi}}}{\sum_{i=0}^{t-1}{\ind{s_i \in {q\inv_{\phi,t}(s)}_i, \phi_i = \phi}}}, ~
    \forall s \in \cS, s\up \in \tilde{S}_{\phi,t}(s).
\end{align}
Recall that ${q\inv_{\phi,t}(s)}_i$ is the cell that was active at $i$-th time step ($i \leq t$) and that contains $q\inv_{\phi,t}(s)$. We denote a continuous extension of $\hat{p}\uc{d}_{\phi,t}$ by $\hat{p}_{\phi,t}$ and it is defined as follows:
\begin{align}\label{def:p_hat}
    \hat{p}_{\phi,t}(s,B) := &\sum_{s\up \in \tilde{S}_{\phi,t}(s)}{\frac{\lambda(B \cap q\inv(s\up))}{\lambda(q\inv(s\up))} \hat{p}\uc{d}_t(s,s\up)},~\forall s \in \cS, B \in \cB_\cS.
\end{align}
Here $\lambda(\cdot)$ is the Lebesgue measure on $(\cS,\cB_\cS)$. In Lemma~\ref{lem:conc_ineq}, we bound the total variation distance between $\wp_{\cS \to \bar{S}_{\phi,t}, p_\phi}(s,\cdot)$ and $\wp_{S_{\phi,t} \to \bar{S}_{\phi,t},\hat{p}_{\phi,t}}(q(\zeta),\cdot)$ where $s \in \zeta$ and $\zeta \in \cP_{\phi,t}$.

\subsection{Approximating the Diameter~\eqref{def:diamb} of a Policy}
\label{app:eval_diam}
We propose the following algorithm for the evaluation of the diameter of a policy. To estimate the diameter of a policy $\phi \in \Phi$, we perform the following iterations,
\begin{align}
    \ovl{D}^{\phi,t}_0(s) &= 0,\notag\\
    \ovl{D}^{\phi,t}_{i+1}(s) &= \diamc{q_{\phi,t}\inv(s)} + \max_{\te \in \cC_{\phi,t}}{\sum_{s\up \in S_{\phi,t}}{\te(s, s\up) \ovl{D}^{\phi,t}_{i},(s\up)}}\label{eq:epdi}
\end{align}
$s \in S_{\phi,t}$, $i \in \bN$.~Note its similarity to~\eqref{eq:epvib}.~The approximate diameter is defined as, 
\begin{align}\label{def:diam_est_mb}
    \underbar{\mbox{diam}}\uc{b}_t(\phi) := \frac{1}{c_{\mbox{diam}}}\lim_{i \to \infty}{\br{\ovl{D}^{\phi,t}_{i+1}(s) - \ovl{D}^{\phi,t}_{i}(s)}},
\end{align}
where $c_{\mbox{diam}}$ is defined in \eqref{def:c_diam}.~The existence of the limit in \eqref{def:diam_est_mb} is shown in Corollary~\ref{cor:conv_diamit}. We show in Appendix~\ref{app:approx_diam_prop} that $\underbar{\mbox{diam}}\uc{b}_t(\phi)$ is a tight lower bound of $\diamb{t}{\phi}$. If $\phi$ is never played $\diamb{t}{\phi}$ is to be set at $1$.
 
\subsection{Existence of Model-based UCB Index}
\label{app:index_exist}
We recall that for~\pzrlmb, the UCB index of an active policy is defined in terms of a limit~\eqref{def:index_mb}. In this section, we show that this limit exists. We begin by introducing a certain Extended MDP~\citep{jaksch2010near}.

\noindent\textbf{Extended MDP.}~Recall that $\bar{S}_{\phi,t}$ is the discrete state space at time $t$ corresponding to policy $\phi$.~Define the discrete action space at time $t$ is given by 
\al{
    A_{\phi,t} := \{\{\phi(s)\} : s \in \bar{S}_{\phi,t}\}.
}
Consider the following modified reward function $\tilde{r}_{\phi,t}: \bar{S}_{\phi,t} \to [0,1]$,
\begin{align*}
    \tilde{r}_{\phi,t}(s) &= r(s,\phi(s)) + (1 + L_\phi) L_r~\diamc{q_{\phi,t}\inv(s)},
\end{align*}
where $L_\phi$ is the Lipschitz constant associated with $\phi$. Consider the following collection of MDPs $\cM^{+}_{\phi,t} = \{(\bar{S}_{\phi,t}, A_{\phi,t}, \tilde{p}, \tilde{r}_{\phi,t}) : \tilde{p} \in \cC_{\phi,t}\}$.~Such a collection of MDPs are often called as an extended MDP in literature~\citep{jaksch2010near} as one may view $\cM^{+}_{\phi,t}$ as an MDP with the finite state space $\bar{S}_{\phi,t}$ and an extended action space, where an element from the extended action space has two components, a control input from $A_{\phi,t}$ and a transition kernel from $\cC_{\phi,t}$.~Observe that the value iteration algorithm on this MDP is simply the index evaluation algorithm~\eqref{eq:epvib} for the policy $\tilde{\phi} : \bar{S}_{\phi,t} \to \cA$ with $\tilde{\phi} = \phi$ on $\bar{S}_{\phi,t}$.~The next lemma shows that the limit in \eqref{def:index_mb} exists and is equal to the value of the extended MDP $\cM^{+}_{\phi,t}$ with the restricted action space that we just discussed.

\begin{lem}\label{lem:conv_eqpvi}
    Let $\phi \in \Phi$, and let $\tilde{\phi} : \bar{S}_{\phi,t} \to \cA$ be such that $\tilde{\phi} = \phi$ on $\bar{S}_{\phi,t}$.~Define
    \begin{align*}
        J_{\cM^{+}_{\phi,t}}(\tilde{\phi}) := \max_{\tilde{p} \in \cC_{\phi,t}}{\flbr{J_{\tilde{M}}(\tilde{\phi}) : \tilde{M} = (\bar{S}_{\phi,t}, A_{\phi,t}, \tilde{p}, \tilde{r}_{\phi,t})}}.
    \end{align*}
    Then,
    \begin{align*}
        \lim_{i \to \infty}{\br{\ovl{V}^{\phi,t}_{i+1}(s) - \ovl{V}^{\phi,t}_{i}(s)}} = J_{\cM^{+}_{\phi,t}}(\tilde{\phi}).
    \end{align*}
\end{lem}
\begin{proof}
    Consider the $i$-th step of the iteration in \eqref{eq:epvib},
    \nal{
        \te_i \in \underset{\te \in \cC_{\phi,t}}{\arg\max} \flbr{\tilde{r}_{\phi,t}(s) + \sum_{s\up \in \bar{S}_{\phi,t}}{\te(s, s\up) \ovl{V}^{\phi,t}_{i}(s\up)}}.
    }
    Let $s\ust := \arg\max_{s \in \bar{S}_{\phi,t}}{\ovl{V}^{\phi,t}_{i}}(s)$. Then since $\te_i(s,\cdot)$ should assign maximum probability to $s\ust$ that is permissible in the set $\cC_{\phi,t}$, we must have $\te_i(s, s\ust) \geq \min{\flbr{1, \frac{1}{2}\eta_{\phi,t}(\zeta)}}$ for all $q_{\phi,t}\inv(s)$. Now, construct the stochastic matrix $\Te = \{\Te_{s,s\up}\}_{s,s\up\in \bar{S}_{\phi,t}}$ by letting $\Te_{s,s\up} = \te_i(s, s\up)$. It is evident that the Markov chain associated with $\Te$ is aperiodic. The proof then follows from \citet[Theorem $9.4.4$]{puterman2014markov}.
\end{proof}

From the previous lemma, it follows that the limit in \eqref{def:diam_est_mb} also exists, which is stated next.
\begin{cor}\label{cor:conv_diamit}
    Define the following extended MDP,
    \begin{align*}
        \tilde{\cM}^{+}_{\phi,t} = \{(\bar{S}_{\phi,t}, A_{\phi,t}, \tilde{p}, d_{\phi,t}) : \tilde{p} \in \cC_{\phi,t}\},
    \end{align*}
    where $d_{\phi,t}(s) = \diamc{q\inv_{\phi,t}(s)}$.~Then,
    \begin{align*}
        \lim_{i \to \infty}{\br{\ovl{D}^{\phi,t}_{i+1}(s) - \ovl{D}^{\phi,t}_{i}(s)}} = J_{\tilde{\cM}^{+}_{\phi,t}}(\tilde{\phi}),
    \end{align*}
    where the discrete policy $\tilde{\phi} : \bar{S}_{\phi,t} \to A_{\phi,t}$ be such that $\tilde{\phi} = \phi$ on $\bar{S}_{\phi,t}$.
\end{cor}

\subsection{Properties of \pzrlmb}
\label{subapp:prop_pzrlmb}
In this section, we show that the model-based UCB index~\eqref{def:index_mb} of an active policy is an optimistic estimate of its average reward; we derive an upper bound of the span of the iterates in~\eqref{eq:epvib}, and lastly, we derive an upper bound on the index of the active policies in terms of their average reward and the diameter.

\begin{lem}[Optimism of Model-based UCB Index]\label{lem:optimism}
    On the set $\cG\uc{\mbox{Model}}$, 
    \begin{align*}
        J_\cM(\phi\up) \leq \idxb_t(\phi),~t \in \{0,1, \ldots, T-1\}, \mbox{ where } \phi \in \Phi^{act.}_t, \phi\up \in B_{\phi,t},
    \end{align*}
    and where $\idxb_t(\phi)$ is defined in \eqref{def:index_mb}.
\end{lem}

\begin{proof}
    Consider the policy evaluation algorithm for the true MDP, i.e.,
    \begin{align*}
        V^\phi_0(s) &= 0,\notag\\
        V^\phi_{i+1}(s) &= r(s,\phi(s)) + \int_{\cS}{p(s,\phi(s),s\up) V^\phi_i(s\up) ds\up}.
    \end{align*}
    Note that the value of the policy $\phi$ evaluated on the MDP $\cM$ is
    \begin{align*}
        J_\cM(\phi) = \lim_{i \to \infty}{\br{V^\phi_{i+1}(s) -  V^\phi_i(s)}},~\forall s \in \cS.
    \end{align*}
    As the limit of a sequence is equal to the limit of its Ces\`aro mean, we can write,
     \begin{align*}
        J_\cM(\phi) = \lim_{i \to \infty}{\frac{1}{i} V^\phi_i(s)},~\forall s \in \cS.
    \end{align*}
    Similarly, we have,
    \begin{align*}
        \lim_{i \to \infty}{\br{\ovl{V}^{\phi,t}_{i+1}(s) - \ovl{V}^{\phi,t}_i(s)}} = \lim_{i \to \infty}{\frac{1}{i} \ovl{V}^{\phi,t}_i(s)},~\forall s \in \bar{S}_{\phi,t},
    \end{align*}
    where $\{\ovl{V}^{\phi,t}_i\}$ is the sequence defined in~\eqref{eq:epvib}.~We will first show that
    \begin{align}
        V^\phi_i(s\up) \leq \ovl{V}^{\phi,t}_i(s),~\forall (s,s\up) \in \bar{S}_{\phi,t} \times \cS \mbox{ such that } s\up \in q\inv_{\phi,t}(s),~\forall i \in \bN, \label{lem:opt;ineq:1}
    \end{align}
    using mathematical induction.~The base case is trivially true as $V^\phi_i(s) = \ovl{V}^{\phi,t}_i(s) = 0$ for all $s \in \cS$.~Let us assume that
    \begin{align}
        V^\phi_j(s\up) \leq \ovl{V}^{\phi,t}_j(s),~\forall (s,s\up) \in \bar{S}_{\phi,t} \times \cS \mbox{ such that } s\up \in q\inv_{\phi,t}(s), \forall , j = 0, 1, \ldots, i.\label{lem:opt;ineq:2}
    \end{align}
    Then we have,
    \begin{align*}
        \ovl{V}^{\phi,t}_{i+1}(s) &= r(s,\phi(s)) + (1+L_\phi) L_r~ \diamc{q_{\phi,t}\inv(s)} + \max_{\te \in \cC_{\phi,t}}{\sum_{s\upp \in \bar{S}_{\phi,t}}{\te(q(q\inv_{\phi,t}(s)), s\upp) \ovl{V}^{\phi,t}_{i}(s\upp)}} \\
        &\geq r(s\up,\phi(s\up)) + \sum_{s\upp \in \bar{S}_{\phi,t}}{p(s\up,\phi(s\up),q\inv_t(s\upp)) \ovl{V}^{\phi,t}_i(s\upp)}\\
        &\geq r(s\up,\phi(s\up)) + \sum_{s\upp \in \bar{S}_{\phi,t}}{p(s\up,\phi(s\up),q\inv_t(s\upp)) \max_{\bar{s}\upp \in q\inv_{\phi,t}(s\upp)}{V^\phi_i(\bar{s}\upp)}}\\
        &\geq r(s\up,\phi(s\up)) + \int_{\cS}{p(s,\phi(s),s\up) V^\phi_i(s\up) ds\up}\\
        &= V^\phi_{i+1}(s),
    \end{align*}
    where the first inequality follows from Assumption~\ref{assum:lip}\eqref{assum:lip_r} and from Lemma~\ref{lem:conc_ineq}, while the second inequality follows from \eqref{lem:opt;ineq:2}.~Upon dividing both sides of \eqref{lem:opt;ineq:1} by $i$, and then taking limit $i \to \infty$, we obtain,
    \begin{align*}
        J_\cM(\phi) \leq \lim_{i \to \infty}{\br{\ovl{V}^{\phi,t}_{i+1}(s) - \ovl{V}^{\phi,t}_i(s)}}, s \in \bar{S}_{\phi,t}
    \end{align*}
    Upon invoking Theorem~\ref{thm:lip_avg_reward} we obtain the following for each $\phi\up \in B_{\phi,t}$,
    \begin{align*}
        J_\cM(\phi\up) &\leq J_\cM(\phi) + L_J \diamb{t}{\phi} \\
        &\leq \lim_{i \to \infty}{\br{\ovl{V}^{\phi,t}_{i+1}(s) - \ovl{V}^{\phi,t}_i(s)}} + L_J \diamb{t}{\phi}\\
        &= \idxb_t(\phi).
    \end{align*}
    This concludes the proof.
\end{proof}

\begin{lem}\label{lem:bd_span_epvib}
    Consider the iterates in \eqref{eq:epvib}.~On the set $\cG\uc{\mbox{Model}}$, we have 
    \al{
    \spn{\ovl{V}^{\phi,t}_i} \leq C_{\ovl{V}},~\forall i \in \bN, t \in \bN,
    }
where,
    \begin{align}
        C_{\ovl{V}} &:= \max{\flbr{\frac{\bar{m} (\bar{m} + 5)}{2} + \frac{6}{\kappa\up} \br{\frac{C_{\eta,\phi} \sqrt{d_\cS} (m\ust+1) }{1 - \alpha}}^{d_\cS} + \frac{2 (m\ust+1)}{1 - \alpha}, \frac{1 + \frac{L_r (1 - \alpha)}{(m\ust+1) C_{\eta,\phi}}}{(1 - \alpha) \br{1 - \alpha^{\frac{1}{m\ust}}}}}},\label{def:C_V}\\
        \bar{m} &:= \ceil{\log_{\frac{1}{\alpha}}{\br{\frac{2C}{\kappa\up} \br{\frac{C_{\eta,\phi} (m\ust+1) \sqrt{d_\cS}}{1-\alpha}}^{d_\cS}}}},\label{def:m_bar} \mbox{ and } \\
        m\ust &:= \ceil{\log_{\frac{1}{\alpha}}{(C)}}+1. \label{def:m_ust}
    \end{align}
    and
    \begin{align*}
        \spn{\ovl{V}^{\phi,t}_i} := \max_{s \in \bar{S}_{\phi,t}}{\ovl{V}^{\phi,t}_i(s)} - \min_{s \in \bar{S}_{\phi,t}}{\ovl{V}^{\phi,t}_i(s)}.
    \end{align*}
\end{lem}
\begin{proof}
    We first note that $\ovl{V}^{\phi,t}_i(s)$ is the optimal value of the expected reward for the extended MDP $\cM_t^+$ that is accumulated during the first $i$ steps when the process starts in state $s$. The first component of the extended action of the extended MDP (refer to Appendix~\ref{app:index_exist}) is chosen according to policy $\phi$, while the second component is chosen so as to maximize the rewards.~We consider the following two cases separately. The first occurs when 
    \begin{align}
        \max{\{\diamc{q_{\phi,t}\inv(s)} : s \in \bar{S}_{\phi,t}\}} \geq \frac{1 - \alpha}{C_{\eta,\phi} \br{m\ust+1}}, \label{cond:1}
    \end{align}
    while the second occurs when \eqref{cond:1} does not hold.

    Case 1:~Let $\zeta$ be the cell with the largest diameter from the set $\cP_{\phi,t}$.~We first show that $\{s_i\}_{i=0}^{\infty}$, the CMP induced by the transition kernel $p$ under the application of policy $\phi$, hits $\zeta$ in 
    \begin{align*}
        \frac{\bar{m} (\bar{m} + 5)}{2} + \frac{6}{\kappa\up} \br{\frac{C_{\eta,\phi} \sqrt{d_\cS} (m\ust+1) }{1 - \alpha}}^{d_\cS}
    \end{align*}
    steps in expectation, where $\bar{m}$ is as defined in~\eqref{def:m_bar}.~From Assumption~\ref{assum:unif_ergodic}, Assumption~\ref{assum:stn_dist} and~\eqref{cond:1}, we have that for any $s\up \in \cS$,
    \begin{align*}
        \mu\uc{i}_{\phi,p,s\up}(\zeta) \geq \frac{1}{2}\mu\uc{\infty}_{\phi,p}(\zeta) \mbox{ and } \mu\uc{i}_{\phi,p,s\up}(\zeta) \leq \frac{3}{2}\mu\uc{\infty}_{\phi,p}(\zeta)~\forall i \geq \bar{m}.
    \end{align*}
    Now, consider the process $\{x_i\}_{i=0}^{\infty}$ that is independent across time; $x_i$ assumes the value $1$ with a probability $\mu\uc{i}_{\phi,p,s\up}(\zeta)$, and $0$ with a probability $1 - \mu\uc{i}_{\phi,p,s\up}(\zeta)$.~Define the following random variables $T\uc{x}_{\{1\}}$ and $T\uc{s}_{\zeta,s\up}$,
    \begin{align*}
        T\uc{x}_{\{1\}} &:= \inf{\{i\geq 0 \mid x_i = 1\}}, \mbox{ and}\\
        T\uc{s}_{\zeta,s\up} &:= \inf{\{i\geq 0 \mid s_i \in \zeta, s_0 = s\up\}}.
    \end{align*}
    We note that the distributions of $T\uc{x}_{\{1\}}$ and $T\uc{s}_{\zeta,s\up}$ are identical, and hence, $\bE\sqbr{T\uc{x}_{\{1\}}} = \bE\sqbr{T\uc{s}_{\zeta,s\up}}$.~We derive an upper bound on $\bE\sqbr{T\uc{x}_{\{1\}}}$, and this would also serve as the upper bound on $\bE\sqbr{T\uc{s}_{\zeta,s\up}}$.~We have,
    \begin{align*}
        \bE\sqbr{T\uc{x}_{\{1\}}} &= \sum_{i=0}^{\infty}{i \cdot \mu\uc{i}_{\phi,p}(\zeta) \prod_{j=0}^{i-1}{\br{1 - \mu\uc{j}_{\phi,p,s}(\zeta)}}} \\
        &\leq \frac{\bar{m}(\bar{m} -1)}{2} + \sum_{i=\bar{m}}^{\infty}{\frac{3i}{2} \mu\uc{\infty}_{\phi,p}(\zeta) \prod_{j=\bar{m}}^{i-1}{\br{1 - \frac{1}{2}\mu\uc{\infty}_{\phi,p}(\zeta)}}} \\
        &\leq \frac{\bar{m}(\bar{m} -1)}{2} + \frac{3}{2}\mu\uc{\infty}_{\phi,p}(\zeta)\sum_{i=0}^{\infty}{i \br{1 - \frac{1}{2}\mu\uc{\infty}_{\phi,p}(\zeta)}}^i + \frac{3 \bar{m}}{2}\mu\uc{\infty}_{\phi,p}(\zeta)\sum_{i=0}^{\infty}{\br{1 - \frac{1}{2}\mu\uc{\infty}_{\phi,p}(\zeta)}^i} \\
        &\leq \frac{\bar{m} (\bar{m} + 5)}{2} + \frac{6}{\mu\uc{\infty}_{\phi,p}(\zeta)}.
    \end{align*}
    Furthermore, from Assumption~\ref{assum:stn_dist}, and since $\bE\sqbr{T\uc{x}_{\{1\}}}=\bE\sqbr{T\uc{s}_{\zeta,s\up}}$, we get, 
    \begin{align*}
        \bE\sqbr{T\uc{s}_{\zeta,s\up}} \leq \frac{\bar{m} (\bar{m} + 5)}{2} + \frac{6}{\kappa\up} \br{\frac{\sqrt{d_\cS}}{\diamc{\zeta}}}^{d_\cS}.
    \end{align*}
    Now, consider two states $\bar{s} \in S_{\phi,t}$, and $\tilde{s} \in q\inv(s)$.~We note that on the set $\cG\uc{\mbox{Model}}$, for the extended MDP $\cM_t^+$ whenever the state is $\bar{s}$, there is an extended action such that the next state transition distribution is $p(\tilde{s},\phi(\tilde{s}),\cdot)$.~Hence, on the set $\cG\uc{\mbox{Model}}$, there is a sequence of extended actions such that starting from any state, in expectation, within $\frac{\bar{m} (\bar{m} + 5)}{2} + \frac{6}{\kappa\up} \br{\frac{\sqrt{d_\cS}}{\diamc{\zeta}}}^{d_\cS}$ steps the process hits $S_{\phi,t} \cap \zeta$.
    
    Now, consider the state process of the extended MDP starting from state $s \in S_{\phi,t}$.~We claim that for any state $s\up$, there exists a sequence of extended actions where the first components of the extended actions are chosen by $\phi$ such that $s\up$ can be reached in $\frac{2}{C_{\eta,\phi} \diamc{q_{\phi,t}\inv(s)}}$ steps in expectation.~This is true because there is a transition kernel in $\cC_{\phi,t}$ that assigns at least $\frac{C_{\eta,\phi}}{2} \diamc{q_{\phi,t}\inv(s)}$ transition probability to $s\up$ when the current state is from $s$.~To summarize, starting from any state using a sequence of actions the state process can reach $q(\zeta)$ in $\frac{\bar{m} (\bar{m} + 5)}{2} + \frac{6}{\kappa\up} \br{\frac{\sqrt{d_\cS}}{\diamc{\zeta}}}^{d_\cS}$ steps in expectation, and from $q(\zeta)$, again it can reach any other state using a sequence of actions in $\frac{2}{C_{\eta,\phi} \diamc{q_{\phi,t}\inv(s)}}$.~Therefore, there cannot be state $s\up$ such that $\max_{s \in S_{\phi,t}}{\ovl{V}^{\phi,t}_i(s)} > \ovl{V}^{\phi,t}_i(s\up) + \frac{\bar{m} (\bar{m} + 5)}{2} + \frac{6}{\kappa\up} \br{\frac{\sqrt{d_\cS}}{\diamc{\zeta}}}^{d_\cS} + \frac{2}{C_{\eta,\phi} \diamc{\zeta}}$.~Now, from the lower bound on $\diamc{\zeta}$~\eqref{cond:1}, we obtain that
    \begin{align}
        \spn{\ovl{V}^{\phi,t}_i} \leq \frac{\bar{m} (\bar{m} + 5)}{2} + \frac{6}{\kappa\up} \br{\frac{C_{\eta,\phi} \sqrt{d_\cS} (m\ust+1) }{1 - \alpha}}^{d_\cS} + \frac{2 (m\ust+1)}{1 - \alpha}. \label{ub:case_1}
    \end{align}

    Case 2:~In this case, we have that
    \begin{align}
        \max{\{\diamc{q_{\phi,t}\inv(s)} : s \in S_{\phi,t}\}} < \frac{1 - \alpha}{C_{\eta,\phi} \br{m\ust+1}}. \label{cond:2}
    \end{align}
    Hence, for any $\te \in \cC_{\phi,t}$,
    \begin{align*}
        \max_{s\in S_{\phi,t}}{\norm{\te(s,\cdot) - p_{t}(s,\cdot;\phi)}_1} \leq \frac{1 - \alpha}{m\ust+1},~\forall s \in S_{\phi,t}.
    \end{align*}
    Let $p\uc{m}_t(\cdot,\cdot;\phi)$ and $\te\uc{m}$ denotes the $m$-step transition kernel of the CMP induced by $p_t(\cdot,\cdot;\phi)$ and $\te$ where $m \in \bN$.~From Lemma~\ref{lem:diff_kern_comp}, we have that, for $m\ust = \ceil{\log_{\frac{1}{\alpha}}(C)} + 1$,
    \begin{align*}
        \max_{s\in S_{\phi,t}}{\norm{\te\uc{m\ust}(s,\cdot) - p\uc{m\ust}_t(s,\cdot;\phi)}_1} \leq \frac{1 - \alpha}{2}.
    \end{align*}
    Also, from Lemma~\ref{lem:pn_contra}, we have that
    \begin{align*}
        \max_{s,s\up\in S_{\phi,t}}{\norm{p\uc{m\ust}_t(s,\cdot;\phi) - p\uc{m\ust}_t(s\up,\cdot;\phi)}_1} \leq 2\alpha.
    \end{align*}
    Hence, for any $\te \in \cC_{\phi,t}$,
    \begin{align*}
        \max_{s,s\up\in S_{\phi,t}}{\norm{\te\uc{m\ust}(s,\cdot) - \te\uc{m\ust}(s\up,\cdot)}_1} &\leq \max_{s,s\up\in S_{\phi,t}}\bigg\{\norm{\te\uc{m\ust}(s,\cdot) - p\uc{m\ust}_t(s,\cdot;\phi)}_1 + \norm{p\uc{m\ust}_t(s,\cdot;\phi) - p\uc{m\ust}_t(s\up,\cdot;\phi)}_1 \\
        &\quad + \norm{p\uc{m\ust}_t(s\up,\cdot;\phi) - \te\uc{m\ust}(s\up,\cdot)}_1 \bigg\}\\
        &\leq \frac{1 - \alpha}{2}+ 2\alpha + \frac{1 - \alpha}{2} \\
        &= 1+\alpha.
    \end{align*}
    Again, from Lemma~\ref{lem:pn_contra} and from Lemma~\ref{lem:bdd_pval_spn}, we have that
    \begin{align}
        \spn{\ovl{V}^{\phi,t}_i} \leq \frac{1 + \frac{L_r (1 - \alpha)}{(m\ust+1) C_{\eta,\phi}}}{(1 - \alpha) \br{1 - \alpha^{\frac{1}{m\ust}}}}. \label{ub:case_2}
    \end{align}
    Combining the upper bounds from \eqref{ub:case_1} and \eqref{ub:case_2}, we obtain the desired upper bound.
\end{proof}

\begin{lem}[Upper-bound on Model-based UCB Index]\label{lem:ub_opt_index}
    On the set $\cG\uc{\mbox{Model}}$, 
    \begin{align*}
        \idxb_t(\phi) \leq J_\cM(\phi) + C_{ub} \int_{\cS}{\diamc{q_{\phi,t}\inv(s)} \mu\uc{\infty}_{\phi,p}(s) ds} + L_J \diamb{t}{\phi}, ~\forall t,\phi \mbox{ such that } \phi \in A_t, t \in \{0,1, \ldots, T-1\}.
    \end{align*}
    where 
    \begin{align}
        C_{ub} := \frac{C_{\eta,\phi} C_{\ovl{V}}}{2} + 2 (1 + L_\phi) L_r, \label{def:C_ub}
    \end{align}
    where $C_{\eta,\phi} = 3(1 + (1 + L_\phi) L_p)$ and $C_{\ovl{V}}$ is as defined in~\eqref{def:C_V}.
\end{lem}

\begin{proof}
    In order to prove this result, we will show that on the set $\cG\uc{\mbox{Model}}$, for every $n \in \bN$, for every $s \in S_{\phi,t}$ and for every $s\up \in q\inv(s)$, the following holds,
    \begin{align}\label{eq:lb_Vindex}
        \bar{V}^{\phi,t}_i(s) \leq V^\phi_i(s\up) + C_{ub}~ \bE_{p,\phi}\sqbr{\sum_{j=0}^{i-1}{\diamc{q_t\inv(s_j, \phi(s_j))}} \middle| s_0 = s\up},
    \end{align}
    where $\bE_{p,\phi}$ denotes that the expectation is taken with respect to the measure induced by $\phi$ when it is applied to MDP with transition kernel $p$.~We prove this using induction. The base case~$(n=0)$ is seen to hold trivially. Next, we assume that the following holds for $i \in [n]$, where $i\up \in \{0,1,\ldots,i\}$,
    \begin{align}\label{indhyp}
        \ovl{V}^{\phi,t}_{i\up}(s) \leq V^\phi_{i\up}(s\up) + C_{ub}~ \bE_{p,\phi}\sqbr{\sum_{j=0}^{{i\up}-1}{\diamc{q_t\inv(s_j, \phi(s_j))}} \middle| s_0 = s\up},
    \end{align}
    for every $s \in S_{\phi,t}$ and for every $s\up \in q\inv(s)$.~Let us fix $s \in S_{\phi,t}$ and $s\up \in q\inv(s)$ arbitrarily, then from~\eqref{eq:epvib} we obtain the following,
    \begingroup
    \allowdisplaybreaks
    \begin{align*}
        \ovl{V}^{\phi,t}_{i+1}(s) &= r(s,\phi(s)) + \max_{\te \in \cC_{\phi,t}}{\sum_{\ts \in S_{\phi,t}}{\te(s, \ts) \ovl{V}^{\phi,t}_i(\ts)}} + (1+L_\phi) L_r~ \diamc{q_{\phi,t}\inv(s)}\\
        &= r(s,\phi(s)) + \sum_{\ts \in S_{\phi,t}}{\te_i(s, \ts) \ovl{V}^{\phi,t}_i(\ts)} + (1+L_\phi) L_r~ \diamc{q_{\phi,t}\inv(s)}\\
        &\leq r(s\up, \phi(s\up)) + \sum_{\ts \in S_{\phi,t}}{p_t(s\up, \ts;\phi) ~\ovl{V}^{\phi,t}_i(\ts)} + \eta_t(q_{\phi,t}\inv(s)) \spn{\ovl{V}^{\phi,t}_i} + 2 (1 + L_\phi) L_r~ \diamc{q_{\phi,t}\inv(s)} \\
        &\leq r(s\up, \phi(s\up)) + \int_\cS{p(s\up, \phi(s\up), \ts) V^\phi_i(\ts) d\ts} + C_{ub}~ \bE_{p,\phi}\sqbr{\sum_{j=1}^{i}{\diamc{q_t\inv(s_j, \phi(s_j))}} \middle| s_0 = s\up} \\
        &\quad + \br{\frac{C_{\eta,\phi} C_{\ovl{V}}}{2} + 2 (1 + L_\phi) L_r} \diamc{q_{\phi,t}\inv(s)} \\
        &= V^\phi_{i+1}(s\up) + C_{ub}~ \bE_{p,\phi}\sqbr{\sum_{j=0}^{i}{\diamc{q_t\inv(s_j, \phi(s_j)))}} \middle| s_0 = s},
    \end{align*}
    \endgroup
    where $\te_i$ is a transition kernel belonging to the set $\cC_{\phi,t}$ that maximizes the expression in the r.h.s. of the first equality.~The first inequality follows from Lipschitz continuity of the reward function, the definition of event $\cG\uc{\mbox{Model}}$ and from Lemma~\ref{lem:bdd_dotdifLv}.~The second inequality is obtained by invoking the induction hypothesis~\eqref{indhyp}, and by using the upper bound on $\spn{\ovl{V}^{\phi,t}_i}$.~This concludes the induction argument, and proves~\eqref{eq:lb_Vindex}.~The proof of the claim follows by dividing both side of~\eqref{eq:lb_Vindex} by $i$ and taking limit $i \to \infty$.
\end{proof}
\subsection{Properties of Approximate Diameter}\label{app:approx_diam_prop}
In this section, we use results of Appendix~\ref{subapp:prop_pzrlmb} and show that $\underbar{\mbox{diam}}\uc{b}_t(\phi)$~\eqref{def:diam_est_mb} serves as a tight lower bound on $\diamb{t}{\phi}$.~The next result follows from Lemma~\ref{lem:optimism}.

\begin{cor}\label{cor:diam_opt}
    On the set $\cG\uc{\mbox{Model}}$, 
    \begin{align*}
        \diamb{t}{\phi} \leq \lim_{i \to \infty}{\br{\ovl{D}^{\phi,t}_{i+1}(s) - \ovl{D}^{\phi,t}_{i}(s)}}, ~\forall t,\phi \mbox{ such that } t \in \{0,1, \ldots, T-1\}, \phi \in \Phi^{act.}_t,
    \end{align*}
    where $\diamb{t}{\phi}$ is defined in \eqref{def:diamb}, and $\ovl{D}^{\phi,t}_i$ are the iterates defined in~\eqref{eq:epdi}.
\end{cor}
\begin{proof}
    The proof follows using the same arguments as the proof of Lemma~\ref{lem:optimism}.
\end{proof}

The next result follows from Lemma~\ref{lem:bd_span_epvib}.
\begin{cor}\label{cor:bd_span_epdi}
    Consider the iterates in \eqref{eq:epdi}.~On the set $\cG\uc{\mbox{Model}}$, we have,
    \nal{
    \spn{\ovl{D}^{\phi,t}_i} \leq C_{\ovl{V}},~\forall i \in \bN, t \in \bN,
    }
    where $C_{\ovl{V}}$ is as defined in~\eqref{def:C_V}, and $m\ust = \ceil{\log_{\frac{1}{\alpha}}{(C)}}+1$.
\end{cor}
\begin{proof}
    The proof follows from Lemma~\ref{lem:bd_span_epvib}.
\end{proof}
The next result is a corollary of Lemma~\ref{lem:ub_opt_index}.
\begin{cor}\label{cor:ub_diam_est}
    On the set $\cG\uc{\mbox{Model}}$, 
    \begin{align*}
        \lim_{i \to \infty}{\br{\ovl{D}^{\phi,t}_{i+1}(s) - \ovl{D}^{\phi,t}_{i}(s)}} \leq c_{\mbox{diam}} \cdot \diamb{t}{\phi},~\forall t \in \{0,1, \ldots, T-1\},\phi \mbox{ such that } \phi \in A_t.
    \end{align*}
    where 
    \begin{align}
        c_{\mbox{diam}} := \frac{C_{\eta,\phi} C_{\ovl{V}}}{2} + 2 (1 + L_\phi) L_r, \label{def:c_diam}
    \end{align}
    where $C_{\eta,\phi} = 3(1 + (1 + L_\phi)L_p)$ and $C_{\ovl{V}}$ is as defined in~\eqref{def:C_V}.
\end{cor}
\begin{proof}
    The proof follows from Lemma~\ref{lem:ub_opt_index}.
\end{proof}

Upon combining the results of Corollary~\ref{cor:diam_opt} and \ref{cor:ub_diam_est} we get,
\begin{align*}
    \frac{\diamb{t}{\phi}}{c_{\mbox{diam}}} \leq \frac{1}{c_{\mbox{diam}}}\lim_{i \to \infty}{\br{\ovl{D}^{\phi,t}_{i+1}(s) - \ovl{D}^{\phi,t}_{i}(s)}} \leq \diamb{t}{\phi},
\end{align*}
which shows that $\underbar{\mbox{diam}}\uc{b}_t(\phi)$ is indeed a tight lower bound on $\diamb{t}{\phi}$.
\section{Properties of Lipschitz and Ergodic MDPs}
\label{app:mdp_prop}
Consider an MDP $\cM = (\cS, \cA, p, r)$ and a policy $\phi \in \Phi$.~Assume that $\cM$ satisfies Assumption~\ref{assum:unif_ergodic}.~Hence, there exists a unique invariant distribution $\mu\uc{\infty}_{\phi,p}$ for the CMP induced by the transition kernel $p$ under the application of policy $\phi$.~One can show that there exists a solution to the following Poisson equation~\citep{hernandez2012further},
\begin{align}
    J + h(s) = r(s,\phi(s)) + \int_{\cS}{h(s\up) ~p(s,\phi(s),ds\up)},~\forall s \in \cS.\label{eq:pois}
\end{align}
Specifically, $(J_\cM(\phi), h_\cM(\cdot; \phi)) \in \bR \times \bR^\cS$ satisfy \eqref{eq:pois}, where
\begin{align}
    J_\cM(\phi) &= \underset{T\to\infty}{\lim\inf}{\frac{1}{T} \bE\sqbr{\sum_{t=0}^{T-1}{r(s_t,\phi(s_t))} \mid s_0=s}} = \int_{\cS}{r(s,\phi(s)) ~d\mu\uc{\infty}_{\phi,p}(s)}, \label{def:Jphi}\\
  \mbox{ and }  h_\cM(s\up; \phi) &= \sum_{t=0}^{\infty}{\int_{\cS}{r(s,\phi(s)) ~d(\mu\uc{\infty}_{\phi,p} - \mu\uc{t}_{\phi,p,s\up})(s)}},~\forall s\up \in \cS. \label{def:hphi}
\end{align}
Recall that $\mu\uc{t}_{\phi,p,s}$ denotes the distribution of $s_t$ given $s_0 = s$ where $\{s_t\}$ is the CMP induced by the transition kernel $p$ under the application of policy $\phi$.~$h_\cM(\cdot; \phi)$ is the relative value function of the policy $\phi$~\citep{puterman2014markov}.~Later, in Lemma~\ref{lem:bdd_rvf_spn}, we derive a bound on the span of the relative value function of any policy from the set $\Phi$ if the MDP satisfies Assumption~\ref{assum:unif_ergodic}. 

We now establish two results related to the CMP $\{s_t\}$, which will be used subsequently to derive a novel sensitivity result~(Proposition~\ref{prop:sens}) for Markov chains on general state space. We will then end this section with the proof of Theorem~\ref{thm:lip_avg_reward}, which shows the Lipschitz property of $J_\cM(\phi)$.~Let us denote the $t$-stage transition kernel under the application of policy $\phi$ by $p\uc{t}_\phi$, i.e., for every $t \in \bN$,
\begin{align}
    p\uc{t}_\phi(s,B) = \bP(s_{\tau+t} \in B \mid s_\tau = s, a_{t\up} = \phi(s_{t\up}), t\up = \tau, \tau + 1, \ldots, \tau+t-1), \label{def:p_tstage}
\end{align}
for all $s \in \cS, B \in \cB_\cS$ and $\tau \in \bN$.~Our next result shows when $t$ is sufficiently large, then Assumption~\ref{assum:unif_ergodic} is equivalent to saying that $p\uc{t}_\phi$ has the ``contractive property,'' ~\eqref{def:contractive}. 

\begin{lem}\label{lem:pn_contra}
    Consider an MDP $\cM = (\cS, \cA, p, r)$ that satisfies Assumption~\ref{assum:unif_ergodic}, and consider a policy $\phi \in \Phi$.
    Then,
    \begin{align}
        \norm{p\uc{i}_\phi(s, \cdot) - p\uc{i}_\phi(s\up,\cdot)}_{TV} \leq 2 \alpha,~\forall s, s\up \in \cS, i \geq m\ust, \label{def:contractive}
    \end{align}
    where $p\uc{i}_\phi$ is the $i$-stage transition probability of the CMP induced by the transition kernel $p$ under the application of policy $\phi$ as defined in~\eqref{def:p_tstage}, and $m\ust = \ceil{\log_\frac{1}{\alpha}{(C)}} + 1$. Conversely, if for every $\phi \in \Phi$,
    \begin{align*}
        \norm{p\uc{m}_\phi(s, \cdot) - p\uc{m}_\phi(s\up,\cdot)}_{TV} \leq 2 \alpha\up,~\forall s, s\up \in \cS,
    \end{align*}
    for some $m \in \bN$, then Assumption~\ref{assum:unif_ergodic} holds with $C = \frac{2}{\alpha\up}$ and $\alpha = {\alpha\up}^{\frac{1}{m}}$.
\end{lem}
\begin{proof}
    We first note that $p\uc{i}_\phi(s, \cdot) = \mu\uc{i}_{\phi,p,s}$ for every $s \in \cS$.~Hence, for any $s, s\up \in \cS$,
    \begin{align*}
        \norm{p\uc{i}_\phi(s, \cdot) - p\uc{i}_\phi(s\up, \cdot)}_{TV} \leq \norm{\mu\uc{i}_{\phi,p,s} - \mu\uc{\infty}_{\phi,p}}_{TV} + \norm{\mu\uc{i}_{\phi,p,s\up} - \mu\uc{\infty}_{\phi,p}}_{TV}.
    \end{align*}
    Also, $C \alpha^i \leq \alpha$ for $i \geq \log_\frac{1}{\alpha}{(C)} + 1$.~Now, using Assumption~\ref{assum:unif_ergodic}, we have that when $i \geq m\ust$, then the following holds,
    \begin{align*}
        \norm{p\uc{i}_\phi(s, \cdot) - p\uc{i}_\phi(s\up, \cdot)}_{TV} &\leq \norm{\mu\uc{i}_{\phi,p,s} - \mu\uc{\infty}_{\phi,p}}_{TV} + \norm{\mu\uc{i}_{\phi,p,s\up} - \mu\uc{\infty}_{\phi,p}}_{TV}\\
        &\leq 2 \alpha.
    \end{align*}
    This concludes the proof of the first claim.
    
    Now, we prove the second claim. Consider the CMP that is described by the transition kernel $p$ and evolves under the application of the policy $\phi$. Consider two copies of this CMP, where these copies differ in the distribution of the initial state. Denote these distributions by $\mu\uc{0}_1$ and $\mu\uc{0}_2$. Denote the distributions of $s_i$ in the corresponding processes by $\mu\uc{i}_1$ and $\mu\uc{i}_2$, respectively. We show the following:
    \begin{align}\label{ineq:geom_close}
        \norm{\mu\uc{i}_{1} - \mu\uc{i}_{2}}_{TV} \leq \tilde{C} \cdot \tilde{\alpha}^t \norm{\mu\uc{0}_{1} - \mu\uc{0}_{2}}_{TV},~\forall i \in \bN,
    \end{align}
    where $\tilde{C} = \frac{1}{\alpha\up}$ and $\tilde{\alpha} = {\alpha\up}^{\frac{1}{m}}$.~The claim then follows by letting $\mu\uc{0}_{1} = \delta_s$ and $\mu\uc{0}_{2} = \mu\uc{\infty}_{\phi,p}$.~Note that,
    \begin{align}
        \norm{\mu\uc{m}_{1} - \mu\uc{m}_{2}}_{TV} &= 2~ \sup_{A \subseteq \cS}{\flbr{(\mu\uc{m}_{1} - \mu\uc{m}_{2})(A)}} \notag\\
        &= 2~\sup_{A \subseteq \cS}{\flbr{\int_{\cS}{p\uc{m}_\phi(s,A) ~d(\mu\uc{0}_{1} - \mu\uc{0}_{2})(s)}}} \notag\\
        &\leq \sup_{\substack{A \subseteq \cS \\ s,s\up \in \cS}}{\flbr{p\uc{m}_\phi(s,A) - p\uc{m}_\phi(s\up,A)}} \norm{\mu\uc{0}_{1} - \mu\uc{0}_{2}}_{TV} \notag\\
        &\leq \alpha\up \norm{\mu\uc{0}_{1} - \mu\uc{0}_{2}}_{TV}. \label{bdd:mu1m-mu2m}
    \end{align}
    Also, note that for any $i \in \bN$,
    \begin{align}
        \norm{\mu\uc{i}_{1} - \mu\uc{i}_{2}}_{TV} &= 2~ \sup_{A \subseteq \cS}{\flbr{(\mu\uc{i}_{1} - \mu\uc{i}_{2})(A)}} \notag\\
        &= 2~\sup_{A \subseteq \cS}{\flbr{\int_{\cS}{p(s,\phi(s),A) ~d(\mu\uc{i-1}_{1} - \mu\uc{i-1}_{2})(s)}}} \notag\\
        &\leq \sup_{\substack{A \subseteq \cS \\ s,s\up \in \cS}}{\flbr{p(s,\phi(s),A) - p(s\up,\phi(s\up),A)}} \norm{\mu\uc{i-1}_{1} - \mu\uc{i-1}_{2}}_{TV} \notag\\
        &\leq \norm{\mu\uc{i-1}_{1} - \mu\uc{i-1}_{2}}_{TV}, \label{bdd:mu1t-mu2t}
    \end{align}
    where the first step follows from the definition of the total variation norm, while the third step follows from Lemma~\ref{lem:bdd_dotdifLv}.~Combining \eqref{bdd:mu1m-mu2m} and \eqref{bdd:mu1t-mu2t}, we can write
    \begin{align*}
        \norm{\mu\uc{i}_{1} - \mu\uc{i}_{2}}_{TV} &\leq {\alpha\up}^{\floor{\frac{i}{m}}} \norm{\mu\uc{0}_{1} - \mu\uc{0}_{2}}_{TV}\\
        &\leq \frac{1}{\alpha\up}\br{{\alpha\up}^{\frac{1}{m}}}^i \norm{\mu\uc{0}_{1} - \mu\uc{0}_{2}}_{TV},~\forall i \in \bN.
    \end{align*}
    This concludes the proof of the lemma.
\end{proof}

Consider two CMPs $\{s_{1,i}\}$ and $\{s_{2,i}\}$, both of which are induced by $\phi$ operating on the MDP $\cM$ that has transition kernel $p$. Their initial state distributions are $\mu\uc{0}_1$ and $\mu\uc{0}_2$ respectively. Next, we upper-bound on the cumulative sum of distances of the distributions of $s_{1,i}$ and $s_{2,i}$.
\begin{lem}\label{lem:sum_tv_dist}
    Consider an MDP $\cM = (\cS, \cA, p, r)$ that satisfies Assumption~\ref{assum:unif_ergodic}, and a policy $\phi \in \Phi$.~Let $\{s_{1,i}\}$ and $\{s_{2,i}\}$ be two CMPs induced by $\phi$ when it is applied to $\cM$. Let $\mu\uc{i}_1$ and $\mu\uc{i}_2$ denote the distributions of $s_{1,i}$ and $s_{2,i}$, respectively.~Then,
    \begin{align*}
        \sum_{i=0}^{\infty}{\norm{\mu\uc{i}_1 - \mu\uc{i}_2}_{TV}} \leq \frac{m\ust}{1 - \alpha} \norm{\mu\uc{0}_{1} - \mu\uc{0}_{2}}_{TV},
    \end{align*}
    where $m\ust = \ceil{\log_\frac{1}{\alpha}{(C)}} + 1$.
\end{lem}
\begin{proof}
    From Lemma~\ref{lem:pn_contra}, we have that,
    \begin{align}\label{ineq:pn_contra}
        \norm{\mu\uc{i}_{1} - \mu\uc{i}_{2}}_{TV} \leq \alpha \norm{\mu\uc{0}_{1} - \mu\uc{0}_{2}}_{TV}, \mbox{ for } i \geq m\ust.
    \end{align}
    Also, for $i \in \bN$ we have,
    \begin{align*}
        \norm{\mu\uc{i}_{1} - \mu\uc{i}_{2}}_{TV} &= 2~ \sup_{A \subseteq \cS}{\flbr{(\mu\uc{i}_{1} - \mu\uc{i}_{2})(A)}} \\
        &= 2~\sup_{A \subseteq \cS}{\flbr{\int_{\cS}{p(s,\phi(s),A) ~d(\mu\uc{i-1}_{1} - \mu\uc{i-1}_{2})(s)}}}\\
        &\leq \sup_{\substack{A \subseteq \cS \\ s,s\up \in \cS}}{\flbr{p(s,\phi(s),A) - p(s\up,\phi(s\up),A)}} \norm{\mu\uc{i-1}_{1} - \mu\uc{i-1}_{2}}_{TV} \\
        &\leq \norm{\mu\uc{i-1}_{1} - \mu\uc{i-1}_{2}}_{TV},
    \end{align*}
    where the first step follows from the definition of the total variation norm, and the third step follows from Lemma~\ref{lem:bdd_dotdifLv}.~Hence, 
    \begin{align}\label{ineq:non_expan}
        \norm{\mu\uc{i}_{1} - \mu\uc{i}_{2}}_{TV} \leq \norm{\mu\uc{0}_{1} - \mu\uc{0}_{2}}_{TV},~\forall i \in \bN
    \end{align}
    Using \eqref{ineq:pn_contra} iteratively, and using \eqref{ineq:non_expan}, we can write,
    \begin{align*}
        \sum_{t=0}^{\infty}{\norm{\mu\uc{i}_1 - \mu\uc{i}_2}_{TV}} &= \sum_{m=0}^{m\ust-1}{\sum_{i=0}^{\infty}{\norm{\mu\uc{m+i\cdot m\ust}_{1} - \mu\uc{m+i\cdot m\ust}_{2}}_{TV}}}\\
        &\leq \frac{m\ust}{1 - \alpha} \norm{\mu\uc{0}_{1} - \mu\uc{0}_{2}}_{TV}.
    \end{align*}
    where $m\ust : = \ceil{\log_\frac{1}{\alpha}{(C)}} + 1$.~This concludes the proof.
\end{proof}
We shall now derive a novel sensitivity result for the CMPs induced by applying different policies to the MDP $\cM$.
\begin{prop}[Sensitivity]\label{prop:sens}
    Consider a MDP $\cM = (\cS, \cA, p, r)$ that satisfies Assumption~\ref{assum:lip}\eqref{assum:lip_p} and~\ref{assum:unif_ergodic}. Let $\phi_1, \phi_2 \in \Phi$ be two stationary deterministic policies. Then,
    \begin{align*}
        \norm{\mu^{(\infty)}_{\phi_1,p} - \mu^{(\infty)}_{\phi_2,p}}_{TV} \leq \frac{m\ust L_p}{1 - \alpha} \rho_{\Phi,\mu^{(\infty)}_{\phi_1,p}}(\phi_1, \phi_2) \leq \frac{m\ust L_p}{1 - \alpha} \rho_{\Phi,\infty}(\phi_1, \phi_2).
    \end{align*}
\end{prop}
\begin{proof}
    First note that $\rho_{\Phi,\mu^{(\infty)}_{\phi_1,p}}(\cdot, \cdot) \leq \rho_{\Phi,\infty}(\cdot, \cdot)$.~It remains to prove the first inequality.~We introduce a few shorthand notation to simplify the presentation. Denote $\mu\uc{m}_{\phi_i,p}$ by $\mu\uc{m}_i$, $\mu^{(\infty)}_{\phi_i,p}$ by $\mu_i$ for $i = 1,2$. Let $p_i$ denote the Markov transition kernel induced by policy $\phi_i$, $i = 1, 2$. Denote $\Delta p := p_1 - p_2$ and $\delta\uc{m} := \mu\uc{m}_1 - \mu\uc{m}_2$.~Note that $\int{p_i ~d\mu\uc{m}_i}$ are the distributions at $m+1$. We have the following,
    \begin{align}\label{eq:recur_delta}
        \delta\uc{m} &= \int{p_1 ~d\mu\uc{m-1}_1} - \int{p_2 ~d\mu\uc{m-1}_2} \notag\\
        &= \int{\Delta p ~d\mu\uc{m-1}_1} + \int{p_2 ~d \delta\uc{m-1}} \notag\\
        &= \int{\Delta p ~d\mu\uc{m-1}_1} + \int{p_2 ~d \br{\int{\Delta p ~d\mu\uc{m-2}_1}}} + \int{p_2 ~d \br{\int{p_2 ~d \delta\uc{m-2}}}} \notag\\
        &= \int{\Delta p ~d\mu\uc{m-1}_1} + \int{p_2 ~d \br{\int{\Delta p ~d\mu\uc{m-2}_1}}} + \int{p_2 ~d \br{\int{p_2 ~d \br{\int{\Delta p ~d\mu\uc{m-3}_1}}}}} \notag\\
        &\quad + \int{p_2 ~d \br{\int{p_2 ~d  \br{\int{p_2 ~d \delta\uc{m-3}}}}}}.
    \end{align}
    Note that if the difference between the distributions of states at time $m-3$ is $\delta\uc{m-3}$, then the last term is the difference between distributions of states at time $m$ of two CMPs whose evolutions are governed by transition kernel $p_2$ from time $m-3$ to $m$, i.e.,
    \begin{align}
        &\int{p_2 ~d \br{\int{p_2 ~d  \br{\int{p_2 ~d \delta\uc{m-3}}}}}} = \bP_{\phi_2}(s_m \mid s_{m-3} \sim \mu\uc{m-3}_1) - \bP_{\phi_2}(s_m \mid s_{m-3} \sim \mu\uc{m-3}_2),
    \end{align}
    where $\bP_{\phi_2}$ denotes the probability measure associated with the CMP with transition kernel $p_2$.~Letting $m \to \infty$ in \eqref{eq:recur_delta} and unrolling the terms up to $k$ terms we have,
    \begin{align}\label{eq:recur_delta_1}
        \mu_1 - \mu_2 &= \int{\Delta p ~d\mu_1} + \int{p_2 ~d \br{\int{\Delta p ~d\mu}_1}} + \int{p_2 ~d \br{\int{p_2 ~d \br{\int{\Delta p ~d\mu_1}}}}} + \ldots \notag\\
        &\quad + \br{\bP_{\phi_2}(s_m \mid s_{m-k} \sim \mu_1) - \bP_{\phi_2}(s_m \mid s_{m-k} \sim \mu_2)}.
    \end{align}
    Now, letting $k \to \infty$, we have the following,
    \begin{align}\label{eq:recur_delta_2}
        \mu_1 - \mu_2 = \int{\Delta p ~d\mu_1} + \int{p_2 ~d \br{\int{\Delta p ~d\mu}_1}} + \int{p_2 ~d \br{\int{p_2 ~d \br{\int{\Delta p ~d\mu_1}}}}} + \ldots.
    \end{align}
    Note that $\int{\Delta p ~d\mu_1}=\int{p_1 ~d\mu_1}-\int{p_2 ~d\mu_1}$. Taking total variation norm of the measures present on both the sides of \eqref{eq:recur_delta_2}, then invoking the triangle inequality and Lemma~\ref{lem:sum_tv_dist}, we obtain,
    \begin{align}\label{eq:recur_delta_3}
        \norm{\mu_1 - \mu_2}_{TV} &\leq \frac{m\ust}{1-\alpha}\norm{\int{\Delta p ~d\mu_1}}_{TV} \notag\\
        &= \frac{m\ust}{1-\alpha}\sup_{A \subseteq \cS}{\flbr{\int_\cS{\Delta p(s,A) ~d\mu_1(s)}}} \notag\\
        &\leq \frac{m\ust}{1-\alpha}\int_\cS{\sup_{A \subseteq \cS}{\{\Delta p(s,A)\}} ~d\mu_1(s)} \notag\\
        &= \frac{m\ust}{1-\alpha}\int{\norm{\Delta p}_{TV} ~d\mu_1}.
    \end{align}
    Also,
    \begin{align*}
        \norm{\mu^{(\infty)}_{\phi_1,p} - \mu^{(\infty)}_{\phi_2,p}}_{TV} &\leq \frac{m\ust}{1-\alpha}\int_\cS{\norm{p(s,\phi_1(s),\cdot) - p(s,\phi_2(s),\cdot)}_{TV} ~d\mu_1(s)}\\
        &\leq \frac{m\ust L_p}{1-\alpha} \int_\cS{\rho_\cA(\phi_1(s),\phi_2(s)) ~d\mu_1(s)} \\
        &= \frac{m\ust L_p}{1 - \alpha} \rho_{\Phi,\mu^{(\infty)}_{\phi_1,p}}(\phi_1, \phi_2),
    \end{align*}
    where the second inequality follows from Assumption~\ref{assum:lip}\eqref{assum:lip_p}. This concludes the proof.
\end{proof}

\noindent\textbf{Proof of Theorem~\ref{thm:lip_avg_reward}}.~For $\phi \in \Phi_{SD}$ we have,
    \begin{align*}
        J_\cM(\phi) = \int_\cS{r(s,\phi(s))~ d\mu^{(\infty)}_{\phi,p}(s)}.
    \end{align*}
    Hence,
    \begin{align*}
        |J_\cM(\phi_1) - J_\cM(\phi_2)| &= \abs{\int_\cS{r(s,\phi_1(s))~ d\mu^{(\infty)}_{\phi_1,p}(s)} - \int_\cS{r(s,\phi_2(s))~ d\mu^{(\infty)}_{\phi_2,p}(s)}}\\
        &\leq \abs{\int_\cS{r(s,\phi_1(s))~ d\br{\mu^{(\infty)}_{\phi_1,p} - \mu^{(\infty)}_{\phi_2,p}}(s)}} \\
        & \quad + \abs{\int_\cS{\br{r(s,\phi_1(s)) - r(s,\phi_2(s))}~ d\mu^{(\infty)}_{\phi_2,p}(s)}} \\
        &\leq \frac{1}{2} \spn{r} \norm{\mu^{(\infty)}_{\phi_1,p} - \mu^{(\infty)}_{\phi_2,p}}_{TV} + L_r \rho_{\Phi,\mu^{(\infty)}_{\phi_2,p}}(\phi_1, \phi_2)\\
        &\leq \br{L_r + \frac{m\ust L_p}{2 (1 - \alpha)}} \rho_{\Phi,\mu^{(\infty)}_{\phi_2,p}}(\phi_1, \phi_2),
    \end{align*}
    where the second inequality follows from Lemma~\ref{lem:bdd_dotdifLv}, Assumption~\ref{assum:lip}\eqref{assum:lip_r}, \eqref{def:policy_metric} and the sub-additivity property of the metric $\rho$. The third inequality follows from Proposition~\ref{prop:sens}. Since, $\rho_{\Phi,\mu^{(\infty)}_{\phi_2,p}}(\cdot, \cdot) \leq \rho_{\Phi,\infty}(\cdot, \cdot)$,
    \begin{align*}
        |J_\cM(\phi_1) - J_\cM(\phi_2)|\leq \br{L_r + \frac{m\ust L_p}{2 (1 - \alpha)}} \rho_{\Phi,\infty}(\phi_1, \phi_2).
    \end{align*}
    If $\mu^{(\infty)}_{\phi_2,p}(\xi) \leq \kappa \nu(\xi)$, for all $\xi \in \cB_\cS, \phi \in \Phi$ for some probability measure $\nu$ and a constant $\kappa > 0$, then $\rho_{\Phi,\mu^{(\infty)}_{\phi_2,p}}(\cdot, \cdot) \leq \kappa \rho_{\Phi,\nu}(\cdot, \cdot)$. Hence,
    \begin{align*}
        |J_\cM(\phi_1) - J_\cM(\phi_2)| \leq \kappa \br{L_r + \frac{m\ust L_p}{2 (1 - \alpha)}} \rho_{\Phi,\nu}(\phi_1, \phi_2).
    \end{align*}
    \hfill$\square$

\subsection{Span-related Lemmata}
In this section, for MDPs that satisfy Assumption~\ref{assum:unif_ergodic}, we derive an upper bound on the span of the relative value function and also derive an upper bound on the span of the iterates of the policy evaluation algorithm for a stationary deterministic policy.

\begin{lem}[Bound on the span of relative value function]\label{lem:bdd_rvf_spn}
    Consider an MDP $\cM = (\cS, \cA, p, r)$ such that $p$ satisfies Assumption~\ref{assum:unif_ergodic}. For any policy $\phi \in \Phi$, the span of the corresponding relative value function $h_\cM(\cdot;\phi)$ is at most $\frac{C \spn{r}}{2 (1 - \alpha)}$.
\end{lem}
\begin{proof}
    Consider a policy $\phi \in \Phi$.~Using the definition of $h_\cM(\cdot;\phi)$ from \eqref{def:hphi}, we obtain,
    \begin{align*}
        \spn{h_\cM(\cdot;\phi)} &= \spn{\sum_{t=0}^{\infty}{\int_{\cS}{r(s,\phi(s)) ~d\br{\mu\uc{\infty}_{\phi,p} - \mu\uc{t}_{\phi,p,\cdot}}(s)}}} \\
        & \leq \sum_{t=0}^{\infty}{\spn{\int_{\cS}{r(s,\phi(s)) ~d\br{\mu\uc{\infty}_{\phi,p} - \mu\uc{t}_{\phi,p,\cdot}}(s)}}} \\
        & \leq \frac{1}{2} \sum_{t=0}^{\infty}{\max_{s}\norm{\mu\uc{\infty}_{\phi,p} - \mu\uc{t}_{\phi,p,s}}_{TV}} \spn{r}\\
        & \leq \frac{C \spn{r}}{2} \sum_{t=0}^{\infty}{\alpha^t} \\
        & = \frac{C \spn{r}}{2(1-\alpha)},
    \end{align*}
    where the third and the fourth step follows from Lemma~\ref{lem:bdd_dotdifLv} and from Assumption~\ref{assum:unif_ergodic}, respectively.
\end{proof}

\begin{remark}
    Observe that if $(\bar{J}, \bar{h})$ satisfies the Poisson equation~\eqref{eq:pois}, then $(\bar{J}, \bar{h}+c)$ is also a solution to the Poisson equation where $c \in \bR$.~In what follows, without loss of generality, we will always assume that $h_\cM(\cdot;\phi)$ belongs to $\sqbr{0, \frac{C \spn{r}}{2(1-\alpha)}}$.
\end{remark}

\begin{lem}[Bound on the span of policy evaluation iterates]\label{lem:bdd_pval_spn}
    Consider an MDP $\cM = (\cS, \cA, p, r)$ such that $p$ satisfies Assumption~\ref{assum:unif_ergodic}, and consider the policy evaluation algorithm for policy $\phi \in \Phi$ on $\cM$ i.e.,
    \begin{align*}
        V^\phi_0(s) &= 0,\notag\\
        V^\phi_{i+1}(s) &= r(s,\phi(s)) + \int_{\cS}{p(s,\phi(s),s\up) V^\phi_i(s\up) ds\up}.
    \end{align*}
    Then $\spn{V^\phi_i} \leq \spn{r} \br{m\ust +1} \frac{2 - \alpha}{1 - \alpha}$, where $m\ust = \ceil{\log_{\frac{1}{\alpha}}(C)} + 1$.
\end{lem}
\begin{proof}
    From Lemma~\ref{lem:pn_contra}, we have that under Assumption~\ref{assum:unif_ergodic},
    \begin{align*}
        \norm{p\uc{m\ust}_\phi(s, \cdot) - p\uc{m\ust}_\phi(s\up,\cdot)}_{TV} \leq 2 \alpha,~\forall s, s\up \in \cS,
    \end{align*}
    where $p\uc{m}_\phi$ is the $m$-step transition kernel of the CMP induced by the transtion kernel $p$ under the application of policy $\phi$.~Also, note that
    \begin{align*}
        V^\phi_{i+m\ust}(s) = \bE\sqbr{\sum_{j=0}^{m\ust}{r(s_{i+j}, \phi(s_{i+j}))} \mid S_i = s} + \int_{\cS}{p\uc{m\ust}_\phi(s,s\up) V^\phi_i(s\up) ds\up}.
    \end{align*}
    Hence,
    \begin{align*}
        \spn{V^\phi_{i+m\ust}} &\leq \spn{\bE\sqbr{\sum_{j=0}^{m\ust}{r(s_{i+j}, \phi(s_{i+j}))} \mid S_i = s}} + \spn{\int_{\cS}{p\uc{m\ust}_\phi(s,s\up) V^\phi_i(s\up) ds\up}}\\
        &\leq \spn{r} \br{m\ust +1} + \frac{1}{2} \spn{V^\phi_i}  \norm{p\uc{m\ust}_\phi(s, \cdot) - p\uc{m\ust}_\phi(s\up,\cdot)}_{TV} \\
        &\leq \spn{r} \br{m\ust +1} + \alpha \spn{V^\phi_i},
    \end{align*}
    where the second inequality follows from Lemma~\ref{lem:bdd_dotdifLv}.~Using the above inequality we can write that, for every $k \leq m\ust$
    \begin{align*}
        \spn{V^\phi_{i\cdot m\ust + k}} &\leq \spn{r} \br{m\ust +1} \sum_{j=0}^{i-1}{\alpha^j} + \alpha^i \spn{V^\phi_{k}}\\
        &\leq \spn{r} \br{m\ust +1} \frac{2 - \alpha}{1 - \alpha}.
    \end{align*}
    This concludes the proof of the lemma.
\end{proof}

\section{Metric Space of Policies}
\label{app:pol_space}
The following definitions are some widely used notions of the size of sets in a metric space, and closely follow Definition $1$ of \citet{cao2020provably}.
\begin{defn}\label{def:sizeofset}
    Let $(X, \rho_X)$ be a metric space.
    \begin{itemize}
        \item An $\eps$-covering of $X$ is a collection of subsets of $X$ such that each element of the collection has a diameter at most $\eps$.~The cardinality of the smallest $\eps$-covering is called the $\eps$-covering number, denoted by $\cN_\eps(X)$.
        \item The covering dimension of $X$ is defined as $d_X := \inf\{d \geq 0: \cN_\eps(X) \leq c~ \eps^{-d}~\forall \eps >0, \mbox{ for some constant } c > 0\}$.
        \item A $\eps$-packing is a collection of points $\bar{X} \subset X$ such that $\min_{x \neq x\up\in \bar{X}}{\rho_X(x,x\up)} \geq \eps$. The cardinality of the largest $\eps$-packing is called the $\eps$-packing number, denoted by $\cN^{\mbox{pack}}_\eps(X)$.
        \item An $\eps$-net of $X$ is an $\eps$-packing $X_\eps \subset X$ for which $\{\{x\up \in X: \rho_X(x,x\up) \leq \eps\} ; x \in X_\eps\}$ covers $X$.
        \item The doubling constant of $(X,\rho_X)$ is defined as $\Lambda(X):= \max_{\eps>0, x \in X}\cN_{\eps/2}(\{x\up \in X: \rho_X(x,x\up) \leq \eps\})$, which is the maximum $\eps/2$-covering number of balls of radius $\eps$ in $X$.
    \end{itemize}
\end{defn}

We note that 
\nal{
\cN^{\mbox{pack}}_{2\eps}(X) \leq \cN_\eps(X) \leq \cN^{\mbox{pack}}_{\eps}(X).
}
Next, we shall show that the metric $\rho_{\Phi_{SD},\nu}(\cdot,\cdot)$~\eqref{def:policy_metric} induces a metric space on the set of stationary deterministic policies, $\Phi_{SD}$.
\begin{lem}\label{lem:policy_metric_space}
    $(\Phi_{SD}, \rho_{\Phi_{SD},\nu})$ is a metric space, where two policies $\phi$ and $\phi\up$ are equivalent if $\rho_{\Phi_{SD},\nu}(\phi,\phi\up) = 0$, i.e.,
    \begin{enumerate}
        \item Non-negativity: For any $\phi, \phi\up \in \Phi_{SD}$, $\rho_{\Phi_{SD},\nu}(\phi,\phi\up) \geq 0$,\label{prop:non-neg}
        \item Symmetry: For any $\phi, \phi\up \in \Phi_{SD}$, $\rho_{\Phi_{SD},\nu}(\phi,\phi\up) = \rho_{\Phi_{SD},\nu}(\phi\up,\phi)$, and\label{prop:symmetry}
        \item Triangle inequality: For any $\phi, \phi\up, \phi\upp \in \Phi_{SD}$, $\rho_{\Phi_{SD},\nu}(\phi,\phi\upp) \leq \rho_{\Phi_{SD},\nu}(\phi,\phi\up) + \rho_{\Phi_{SD},\nu}(\phi\up,\phi\upp)$.\label{prop:tri_ineq}
    \end{enumerate}
\end{lem}
\begin{proof}
    Non-negativity~\eqref{prop:non-neg} follows from the fact that $\rho_\cA$ is a metric and $\nu$ is a probability measure. Symmetry~\eqref{prop:symmetry} follows from the symmetry property of the metric $\rho_\cA$, i.e., $\rho_\cA(\phi(s),\phi\up(s)) = \rho_\cA(\phi\up(s),\phi(s))$. Similarly, \eqref{prop:tri_ineq} follows from triangle inequality of the metric $\rho_\cA$, i.e., $\rho_\cA(\phi(s),\phi\upp(s)) \leq \rho_\cA(\phi(s),\phi\up(s)) + \rho_\cA(\phi\up(s),\phi\upp(s))$.
\end{proof}
Consider a class of policies, $\Phi \subseteq \Phi_{SD}$.~Note that $(\Phi,\rho_{\Phi,\nu})$ is the induced metric space from $(\Phi_{SD},\rho_{\Phi_{SD},\nu})$.~In this paper we deal with the metric space $(\Phi, \rho_{\Phi,\nu})$, where $\nu$ satisfies Assumption~\ref{assum:stn_dist}.

\section{Concentration Inequalities}
\label{app:conc_ineq}
We will derive few concentration results for the estimates of the average reward, and for the model estimates.

\noindent\textbf{Concentration inequality for average reward estimate.}~Here we obtain concentration results which quantify the deviations of the empirical estimate of the average reward from the average reward $J_\cM(\phi)$, when a policy $\phi$ is played. This is then used later in order to construct the model-free UCB indices of the active policies.~This concentration result is also used in the analysis of the model-based algorithm since it yields us a desired lower bound on the number of visits to certain ``desired'' regions of the state-action space, in the event that a certain policy has been played at least so many times.~Let $\cF_t$ be the $\sigma$-algebra generated by the random variables $s_0,a_0,s_1,a_1,\ldots, s_t, a_t$.

\begin{lem}\label{lem:conc_ineq_avg_rew}
    Consider an MDP $\cM = (\cS, \cA, p, r)$ that satisfies Assumption~\ref{assum:unif_ergodic} and let $\{\Phi_i : i= 1,2,\ldots,N\}$ be a collection of subsets of $\Phi$.~Let $\phi_t$ denote the policy played at time $t$.~Then, with a probability greater than $1 - \delta$, we have,
    \begin{align*}
        &\frac{1}{N_\tau(\Phi_i)}\abs{\sum_{t=1}^{\tau}{\ind{\phi_t \in \Phi_i}\br{J_\cM(\phi_t) - r(s_t,\phi_t(s_t))}}} \\
        &\leq \frac{C \spn{r}}{1 - \alpha} \br{\sqrt{\frac{1}{N_\tau(\Phi_i)} \log{\br{\frac{(1 + N_\tau(\Phi_i))N}{\delta}}}} + \frac{1 + K_\tau(\Phi_i)}{N_\tau(\Phi_i)}},~\forall \tau \in \{1,2,\ldots,T-1\}, \Phi_i, i = 1, 2, \ldots,N,
    \end{align*}
    where $N_\tau(\Phi_i) = \sum_{t=1}^{\tau}{\ind{\phi_t \in \Phi_i}}$, and $K_\tau(\Phi_i)$ is the number of episodes till time $\tau$ that played policies from $\Phi_i$.
\end{lem}
\begin{proof}
    We begin by converting ``Markovian noise'' to ``martingale noise'' using the Poisson equation~\eqref{eq:pois} \citep{metivier1984applications} as follows.
    \begingroup
        \allowdisplaybreaks
        \begin{align*}
            &\sum_{t=0}^{\tau}{\ind{\phi_t \in \Phi_i}\br{J_\cM(\phi_t) - r(s_t,\phi_t(s_t))}}\\
            &= \sum_{t=0}^{\tau}{\ind{\phi_t \in \Phi_i}\br{\int_{\cS}{h_\cM(s;\phi_t) p(s_t,\phi_t(s_t), ds)} - h_\cM(s_t,\phi_t)}} \\
            &= \sum_{t=1}^{\tau}{\ind{\phi_t \in \Phi_i}\br{\int_{\cS}{h_\cM(s;\phi_t) p(s_{t-1},\phi_{t-1}(s_{t-1}), ds)} - h_\cM(s_t,\phi_t)}} \\
            &\quad + \ind{\phi_0 \in \Phi_i}\br{\int_{\cS}{h_\cM(s;\phi_0) p(s_0,\phi_0(s_0), ds)} - h_\cM(s_0,\phi_0)}\\
            &\quad + \sum_{t=1}^{\tau}{\int_{\cS}{\ind{\phi_t \in \Phi_i} h_\cM(s;\phi_t) p(s_t,\phi_t(s_t), ds)} - \int_{\cS}{\ind{\phi_t \in \Phi_i} h_\cM(s;\phi_t) p(s_{t-1},\phi_{t-1}(s_{t-1}), ds)}} \\
            &= \sum_{t=1}^{\tau}{\ind{\phi_t \in \Phi_i}\br{\int_{\cS}{h_\cM(s;\phi_t) p(s_{t-1},\phi_{t-1}(s_{t-1}), ds)} - h_\cM(s_t,\phi_t)}} \\
            &\quad + \sum_{t=1}^{\tau}{\int_{\cS}{\br{\ind{\phi_{t-1} \in \Phi_i} h_\cM(s;\phi_{t-1}) - \ind{\phi_t \in \Phi_i} h_\cM(s;\phi_t)} p(s_{t-1},\phi_{t-1}(s_{t-1}), ds)}} \\
            &\quad + \int_{\cS}{\ind{\phi_\tau \in \Phi_i} h_\cM(s;\phi_\tau) p(s_\tau,\phi_\tau(s_\tau), ds)} - h_\cM(s;\phi_0).
        \end{align*}
    \endgroup
    This yields,
    \begin{align}
        &\abs{\sum_{t=1}^{\tau}{\ind{\phi_t \in \Phi_i}\br{J_\cM(\phi_t) - r(s_t,\phi_t(s_t))}}} \notag\\
        &\leq \abs{\sum_{t=1}^{\tau}{\ind{\phi_t \in \Phi_i}\br{\int_{\cS}{h_\cM(s;\phi_t) p(s_{t-1},\phi_{t-1}(s_{t-1}), ds)} - h_\cM(s_t,\phi_t)}}} \notag\\
        &\quad + \abs{\sum_{t=1}^{\tau}{\int_{\cS}{\br{\ind{\phi_{t-1} \in \Phi_i} h_\cM(s;\phi_{t-1}) - \ind{\phi_t \in \Phi_i} h_\cM(s;\phi_t)} p(s_{t-1},\phi_{t-1}(s_{t-1}), ds)}}} \notag\\
        &\quad + \ind{\phi_0 \in \Phi_i}\spn{h_\cM(\cdot;\phi_0)} + \ind{\phi_\tau \in \Phi_i}\spn{h_\cM(\cdot;\phi_\tau)}. \label{ineq:abs_dif_1}
    \end{align}
    Consider the first term in the r.h.s. of \eqref{ineq:abs_dif_1}.~Let $\eta_t : = \int_{\cS}{h_\cM(s;\phi_t) p(s_{t-1},\phi_{t-1}(s_{t-1}), ds)} - h_\cM(s_t,\phi_t)$.~We note that $\flbr{\eta_t}_t$ is a martingale difference sequence.~Also, from the bound on the span of $h_\cM(\cdot;\phi)$ derived in Lemma~\ref{lem:bdd_rvf_spn}, we have that $\eta_t \in \sqbr{-\frac{C \spn{r}}{2(1-\alpha)}, \frac{C \spn{r}}{2(1-\alpha)}}$. Hence, by Hoeffding's lemma, $\eta_t$ is conditionally $\frac{C \spn{r}}{2(1-\alpha)}$ sub-Gaussian.~Since $\ind{\phi_t \in \Phi_i}$ is $\cF_{t-1}$ measurable, we can invoke Corollary~\ref{cor:self_norm_vec} and obtain that for each $\delta \in (0,1)$,  the following holds 
    with a probability greater than $1 - \delta$,
    \begin{align*}
        &\abs{\sum_{t=1}^{\tau}{\ind{\phi_t \in \Phi_i}\br{\int_{\cS}{h_\cM(s;\phi_t) p(s_{t-1},\phi_{t-1}(s_{t-1}), ds)} - h_\cM(s_t,\phi_t)}}} \\
        &\leq \frac{C \spn{r}}{2(1-\alpha)} \sqrt{2 \br{1 + N_\tau(\Phi_i)} \log{\br{\frac{1 + N_\tau(\Phi_i)}{\delta}}}} \\
        &\leq \frac{C \spn{r}}{1-\alpha} \sqrt{N_\tau(\Phi_i) \log{\br{\frac{1 + N_\tau(\Phi_i)}{\delta}}}}, \mbox{ where } \tau \in \{1,2,\ldots\}.
    \end{align*}
    Now, consider the second term in the r.h.s. of \eqref{ineq:abs_dif_1}. The $t$-th element of the sum could be non-zero if $\ind{\phi_{t-1} \in \Phi_i} \neq \ind{\phi_t \in \Phi_i}$. Since both the algorithms play policies in an episodic manner, $\sum_{t=1}^{\tau}{\abs{\ind{\phi_{t-1} \in \Phi_i} - \ind{\phi_t \in \Phi_i}}} \leq 2 K_\tau(\Phi_i)$. So, the second term in \eqref{ineq:abs_dif_1} can be upper-bounded by $\frac{C \spn{r}}{1 - \alpha} K_\tau(\Phi_i)$ using Lemma~\ref{lem:bdd_rvf_spn}. The third term in the r.h.s. of \eqref{ineq:abs_dif_1} is also bounded similarly by $\frac{C \spn{r}}{1 - \alpha}$.

    Summing the bounds obtained on individual terms in the r.h.s. of \eqref{ineq:abs_dif_1}, and taking a union bound over $\Phi_i, i = 1,2, \ldots, N$, we have that for each $\delta \in (0,1)$, the following holds with a probability greater than or equal to $1 - \delta$,
    \begin{align}
        \abs{\sum_{t=1}^{\tau}{\ind{\phi_t \in \Phi_i}\br{J_\cM(\phi_t) - r(s_t,\phi_t(s_t))}}} &\leq \frac{C \spn{r}}{1 - \alpha} \sqrt{N_\tau(\Phi_i) \log{\br{\frac{(1 + N_\tau(\Phi_i))N}{\delta}}}} \notag\\
        &\quad+ \frac{C \spn{r}}{1 - \alpha} (1 + K_\tau(\Phi_i)), \mbox{ where } \tau \in \{1,2,\ldots\}, i= 1, 2, \ldots, N 
    \end{align}
    The desired bound is then obtained by dividing both the sides by $N_\tau(\Phi_i)$.
\end{proof}

\noindent\textbf{Concentration inequality for the model estimate.}~Consider a policy $\phi$ and the transition kernel induced by it, $p_\phi$.~Let $\tilde{S}$ be a set of representative points of a set of disjoint cells.~Recall the discretization of $p_\phi$, $\wp_{\cS \to \tilde{S}, p_\phi}$.~Denote its continuous extension by $\bar{\wp}_{\cS \to \tilde{S}, p_\phi}$, i.e.,
\begin{align*}
    \bar{\wp}_{\cS \to \tilde{S}, p_\phi}(s,B) := \sum_{s\up \in \tilde{S}}{\frac{\lambda(B \cap q\inv(s\up))}{\lambda(q\inv(s\up))} \wp_{\cS \to \tilde{S}, p_\phi}(s,s\up)},
\end{align*}
for every $s \in \cS, B \in \cB_\cS$.~Recall the confidence ball $\cC_{\phi,t}$~\eqref{def:confball} associated with the policy $\phi$ and the estimate $\hat{p}_t$. Consider the following event, 
\begin{align}\label{def:G_1}
    \cG\uc{\mbox{Model}} := \cap_{t=0}^{T-1}{\flbr{\norm{\wp_{S_{\phi,t} \to \bar{S}_{\phi,t}, \hat{p}_{\phi,t}}(s,\cdot) - \wp_{\cS \to \bar{S}_{\phi,t}, p_\phi}(s\up, \cdot)}_1 \leq \eta_{\phi,t}(\zeta) \mbox{ for every } s \in S_{\phi,t}, s\up \in q\inv(s),~\forall \phi \in \Phi^{act.}_t}}.
\end{align}
We show that $\cG\uc{\mbox{Model}}$ holds with a high probability.~We use this result in the analysis of \pzrlmb.
\begin{lem}\label{lem:conc_ineq}
    $\bP\br{\cG\uc{\mbox{Model}}} \geq 1 - \frac{\delta}{3}$, where $\cG\uc{\mbox{Model}}$ is as in~\eqref{def:G_1}.
\end{lem}

\begin{proof}
    Fix a policy $\phi$, $t\in\{0,1,\ldots, T-1\}$, and consider a cell $\zeta$ that is active at time $t$,~and let $s\up \in \zeta$.~Denote the level of cell $\zeta$ by $\ell$ and the point $q(\zeta)$ by $s$. We want to get a high probability upper bound on $\norm{\wp_{S_{\phi,t} \to \bar{S}_{\phi,t}, \hat{p}_{\phi,t}}(s,\cdot) - \wp_{\cS \to \bar{S}_{\phi,t}, p_\phi}(s\up, \cdot)}_1$, where $s \in S_{\phi,t}$ and $s\up \in q\inv(s)$. We have that
    \begin{align}
        &\norm{\wp_{S_{\phi,t} \to \bar{S}_{\phi,t}, \hat{p}_{\phi,t}}(s,\cdot) - \wp_{\cS \to \bar{S}_{\phi,t}, p_\phi}(s\up, \cdot)}_1 \notag\\
        = & \norm{\hat{p}_{\phi,t}(s,\cdot) - \bar{\wp}_{\cS \to \bar{S}_{\phi,t}, p_\phi}(s\up,\cdot)}_{TV} \notag\\
        \leq & \norm{\hat{p}_{\phi,t}(s,\cdot) - \bar{\wp}_{\cS \to S\uc{\ell}, p_\phi}(s\up,\cdot)}_{TV} + \norm{\bar{\wp}_{\cS \to S\uc{\ell}, p_\phi}(s\up,\cdot) - \bar{\wp}_{\cS \to \bar{S}_{\phi,t}, p_\phi}(s\up,\cdot)}_{TV} \notag\\
        \leq & \norm{\hat{p}\uc{d}_{\phi,t}(s,\cdot) - \wp_{\cS \to S\uc{\ell}, p_\phi}(s\up,\cdot)}_{1} + \norm{\bar{\wp}_{\cS \to S\uc{\ell}, p_\phi}(s\up,\cdot) - \bar{\wp}_{\cS \to \bar{S}_{\phi,t}, p_\phi}(s\up,\cdot)}_{TV}. \label{ineq:cb_decomp}
    \end{align}
    From Lemma~I.2 of \citet{kar2024provably}, we have that $\norm{\bar{\wp}_{\cS \to \bar{S}_{\phi,t}, p_\phi}(s\up,\cdot) - \bar{\wp}_{\cS \to S\uc{\ell}, p_\phi}(s\up,\cdot)}_{TV} \leq \cO(\diamc{\zeta})$.~Now let us focus on the first term.~Note that bounding if $\zeta$ is not visited by $\phi$, the bound in $\cC_{\phi,t}$ holds trivially.~Note that both $\wp_{\cS \to S\uc{\ell}, p_\phi}(s\up,\cdot)$ and $\hat{p}\uc{d}_{\phi,t}(s,\cdot)$ have the same support $\tilde{S}_{\phi,t}(s)$. By Assumption~\ref{assum:stn_dist}, i.e., the stationary distribution has bounded support, by Assumption~\ref{assum:unif_ergodic}, and by the Borel-Cantelli Lemma, we have that the total number of time steps when the corresponding CMP is outside the support of the stationary distribution is at most finite. Due to this and the cell activation rule, one can find that $\abs{\tilde{S}_{\phi,t}(s)} \leq \cO\br{\log{(N_t(\phi))} + \diamc{\zeta}^{-d_\cS}}$ w.p. $1$.
    
    Let $\tilde{S}_{\phi,t}^+(s)$ denote the set of points from $\tilde{S}_{\phi,t}(s)$ such that for each $s_+ \in {\tilde{S}_{\phi,t}^+}(s)$, we have $\hat{p}\uc{d}_{\phi,t}(s,s_+) - \wp_{\cS \to S\uc{\ell}, p_\phi}(s\up,s_+) > 0$.~We can write the following:
    \begin{align}
        \bP\br{\norm{\hat{p}\uc{d}_{\phi,t}(s,\cdot) - \wp_{\cS \to S\uc{\ell}, p_\phi}(s\up,\cdot)}_1 \geq \iota} &= \bP\br{\max_{S\up \subset \tilde{S}_{\phi,t}^+(s)}{\sum_{s_+ \in S\up}{\hat{p}\uc{d}_{\phi,t}(s,s_+) - \wp_{\cS \to S\uc{\ell}, p_\phi}(s\up,s_+)}} \geq \frac{\iota}{2}} \notag\\
        &= \bP\br{\cup_{S\up \subset \tilde{S}_{\phi,t}^+}{\flbr{\sum_{s_+ \in S\up}{\hat{p}\uc{d}_{\phi,t}(s,s_+) - \wp_{\cS \to S\uc{\ell}, p_\phi}(s\up,s_+)} \geq \frac{\iota}{2}}}}.\label{cineq:union}
    \end{align}
    Note that if $S\up \subset \tilde{S}_{\phi,t}^+$, then $\tilde{S}_{\phi,t} \setminus S\up \not\subset \tilde{S}_{\phi,t}^+$. Hence the number of subsets of $\tilde{S}_{\phi,t}^+$ is at most $2^{|\tilde{S}_{\phi,t}|-1}$.~If $\bP\br{\sum_{s_+ \in S\up}{\hat{p}\uc{d}_{\phi,t}(s,s_+) - \wp_{\cS \to S\uc{\ell}, p_\phi}(s\up,s_+)} \geq \frac{\iota}{2}} \leq b_\iota,~\forall S\up \subset \tilde{S}_{\phi,t}^+$, then by an application of union bound on ~\eqref{cineq:union}, we obtain that the following must hold,
    \begin{align}
        \bP\br{\norm{\hat{p}\uc{d}_{\phi,t}(s,\cdot) - \wp_{\cS \to S\uc{\ell}, p_\phi}(s\up,\cdot)}_1 \geq \iota} \leq 2^{|\tilde{S}_{\phi,t}|-1} b_\iota.\label{ineq:l1byunion}
    \end{align}
    Consider a fixed $\xi \subseteq \cS$.~Define the following random processes,
    \begin{align}
        v_i(z) &:= \ind{s_i,\in \zeta_i, \phi_i = \phi},\\
        v_i(z,\xi) &:= \ind{(s_i, s_{i+1}) \in \zeta_i \times \xi, \phi_i = \phi},\\
        w_i(z,\xi) &:= v_i(z,\xi) - p(s_i,\phi(s_i),\xi) v_i(z),
    \end{align}
    where $i = 0, 1, \ldots, T-1$.~Let $S\up \subset \tilde{S}_{\phi,t}^+$ and $\xi = \cup_{s \in S\up}{q\inv_{\phi,t}(s)}$. Also, denote $N_t\br{\zeta}$ to be the number of visits to $\zeta$ till time $t$ while playing policy $\phi$. Similarly, denote $N_t\br{\zeta, \xi}$ to be the number of transitions from $\zeta$ to $\xi$ till time $t$ while playing policy $\phi$. Then we have,
    \begin{align}
        \sum_{s_+ \in S\up}{\hat{p}\uc{d}_{\phi,t}(s,s_+) - \wp_{\cS \to S\uc{\ell}, p_\phi}(s\up,s_+)} &= \frac{N_t\br{\zeta, \xi}}{N_t\br{\zeta}} - p(s\up,\phi(s\up),\xi) \notag\\
        &= \frac{N_t\br{\zeta, \xi} - p(s\up,\phi(s\up),\xi) N_t\br{\zeta}}{N_t\br{\zeta}} \notag\\
        &\leq \frac{1}{N_t\br{\zeta}}\br{\sum_{i = 0}^{t - 1}{w_i(z,\xi)}} + \frac{(1 + L_\phi) L_p}{2 N_t\br{\zeta}} \sum_{i=0}^{N_t(\zeta)}{\diamc{\zeta_{t_i}}}\notag\\
        &\leq \frac{1}{N_t\br{\zeta}}\br{\sum_{i = 0}^{t - 1}{w_i(z,\xi)}} + 1.5 (1 + L_\phi) L_p~ \diamc{\zeta}, \label{ineq:determ}
    \end{align}
    where the last step follows from Lemma~\ref{lem:avg_diam}.~Note that $\flbr{w_i(z,\zeta)}_{i \in [T-1]}$ is martingale difference sequence w.r.t. $\flbr{\cF_i}_{i \in [T-1]}$.~Moreover, $\abs{w_i(z,\zeta)} \leq 1$. Hence from Lemma~\ref{lem:ah_ineq} we have,
    \begin{align*}
        \bP\br{\flbr{\frac{\sum_{i=0}^{t-1}{w_i(z,\xi)}}{N_t\br{\zeta}} \geq \sqrt{\frac{2}{N_t(\zeta)} \log{\br{\frac{3}{\delta}}}}, N_t(\zeta) = N}} \leq \frac{\delta}{3},
    \end{align*}
    which when combined with~\eqref{ineq:determ} yields,
    \begin{align*}
        \bP\br{\flbr{\sum_{s_+ \in S\up}{\hat{p}\uc{d}_{\phi,t}(s,s_+) - \wp_{\cS \to S\uc{\ell}, p_\phi}(s\up,s_+)} \geq \sqrt{\frac{2}{N_t(\zeta)} \log{\br{\frac{3}{\delta}}}} + 1.5 (1 + L_\phi) L_p~ \diamc{\zeta}, N_t(\zeta) = N}}  \leq \frac{\delta}{3}.
    \end{align*}
    Upon using~\eqref{ineq:l1byunion} in the above, and taking a union bound over all possible values of $N$, we obtain,
    \begin{align*}
        \bP\br{\flbr{\norm{\hat{p}\uc{d}_{\phi,t}(s,\cdot) - \wp_{\cS \to S\uc{\ell}, p_\phi}(s\up,\cdot)}_1 \geq \sqrt{\frac{2 |\tilde{S}_{\phi,t}|}{N_t(\zeta)} \log{\br{\frac{3T}{\delta}}}} + 3 (1 + L_\phi) L_p~ \diamc{\zeta}, N_t(\zeta) = N}} \leq \frac{\delta}{3}.
    \end{align*}
    Note that since on a set of probability $1$, $\abs{\tilde{S}_{\phi,t}} = \ctO\br{\diamc{\zeta}^{-d_\cS}}$, we can replace $|\tilde{S}_{\phi,t}|$ by its upper bound, i.e.,
    \begin{align}
        \bP\br{\norm{\hat{p}\uc{d}_{\phi,t}(s,\cdot) - \wp_{\cS \to S\uc{\ell}, p_\phi}(s\up,\cdot)}_1  \geq \diamc{\zeta}^{-\frac{d_\cS}{2}} \sqrt{\frac{const\cdot \log{\br{\frac{3 T}{\delta}}}}{N_t(\zeta)}} + 3 (1 + L_\phi) L_p~ \diamc{\zeta}} \leq \frac{\delta}{3}.
    \end{align}
    Let $\cN_1 := 2 d_\cS^\frac{d_\cS}{2} \br{\frac{T}{c\uc{b}_d \log{\br{T/\delta}}}}^\frac{d}{d_\cS + 2}$. $\cN_1$ is an upper bound on the total number of cells \pzrlmb~can activate on any sample path.~Also note that there could be at most $T$ active policies.~Upon taking union bound over all the cells that could possibly be activated in all possible sample paths at some $t$ and over all active policies and using the fact that $N_t(\zeta) \geq N_{\min}(\zeta)$, we obtain that there exists a constant $c\uc{b}_d$ such that with a probability at least $1 - \frac{\delta}{3}$, 
    \begin{align}
        \norm{\hat{p}\uc{d}_{\phi,t}(s,\cdot) - \wp_{\cS \to S\uc{\ell}, p_\phi}(s\up,\cdot)}_1  \geq \br{\frac{c\uc{b}_d \log{\br{\frac{3 T}{\delta}}}}{N_t(\zeta)}}^\frac{1}{d_\cS + 2} + 3 (1 + L_\phi) L_p~ \diamc{\zeta},
    \end{align}
    for all $s \in S_{\phi,t}$, $s\up \in q\inv_{\phi,t}(s)$, $\phi \in \Phi^{act.}_t$, $t \in \{0,1,\ldots,T-1\}$.~Now, invoking the above in \eqref{ineq:cb_decomp}, we get what we intend to prove.
\end{proof}
\section{Analysis of \pzrlmf}
\label{app:analysis_mf}
In this section, we derive an upper bound on the total number of policies that can be activated by \pzrlmf, and also derive an upper bound on the number of plays of a policy in terms of its sub-optimality w.r.t. the policy class $\Phi$.~The proof of our main result, the regret bound~(Theorem~\ref{thm:reg_ub_mf}) crucially relies upon this result.~Corollary~\ref{cor:con_ineq_mf} specializes Lemma~\ref{lem:conc_ineq_avg_rew} for it to be used in the analysis of \pzrlmf.~In Lemma~\ref{lem:lb_phistr_diam} we derive a lower bound on $\diamf{t}{\phi}$ of an active policy $\phi$, in the event that $B_{t,\phi}$ contains an optimal policy.~Building upon these results, we prove the main result of this section in Lemma~\ref{lem:ub_nplay}.

Let $i\uc{f}:= \ceil{\frac{\log_2{\br{T}}}{d^\Phi_z + 2}}$. Let $\bar{\Phi}_\gm$ be the set consisting of the centers of balls that yield the smallest $\gm/2(2+L_J)$-covering of $\Phi_\gm$ for $\gm = 2^{-i}$, $i = 1, 2, \ldots, i\uc{f}$.~It follows from the definition of the zooming dimension~\eqref{def:zoomingdim} after performing few algebraic manipulations that the cardinality of the set $\cup_{i=1}^{i\uc{f}}{\bar{\Phi}\uc{f}_{2^{-i}}}$ is at most $c_{z_1} 2^{d^\Phi_z+1} T^{\frac{d^\Phi_z}{d^\Phi_z + 2}}$.~Let us define the following covering of $\Phi$,
\begin{align}
    \Phi\uc{f}_{\mbox{cover}} := \{B_\phi(2^{-i}/2(2+L_J)) \cap \Phi_{2^{-i}} : \phi \in \bar{\Phi}\uc{f}_{2^{-i}}, i = 1, 2, \ldots, i\uc{f}\},\label{def:phif_cover}
\end{align}
and the following set,\footnote{where the superscript denotes that this set is used for the analysis of model-free algorithm.}
\begin{align}
    \cG\uc{\mbox{MF}} &:= \Bigg\{\frac{1}{N_\tau(\Phi\up)}\abs{\sum_{t=1}^{\tau}{\ind{\phi_t \in \Phi\up} J_\cM(\phi_t)} -\sum_{t=1}^{\tau}{\ind{\phi_t \in \Phi\up} r(s_t,\phi_t(s_t))}} \leq \diamf{\tau}{\phi}, \notag\\
    &\qquad~\forall \Phi\up \in \Phi\uc{f}_{\mbox{cover}}, \forall \tau \in \{1,2,\ldots,T-1\}\Bigg\}, \label{def:G_mf}
\end{align}
where $\diamf{t}{\phi} = \frac{C}{1 - \alpha} \br{\sqrt{\frac{c\uc{f}_d \log{\br{\frac{T}{\delta}}}}{N_t(\phi)}} + \frac{1 + K_t(\phi)}{N_t(\phi)}}$ and $c\uc{f}_d$ be a constant such that,
\begin{align}
    c\uc{f}_d \geq \frac{\log{\br{\frac{c_z}{\delta} 2^{d^\Phi_z+2} T^{\frac{d^\Phi_z + 3}{d^\Phi_z + 2}}}}}{\log{\br{\frac{T}{\delta}}}}. \label{def:cdf}
\end{align}
The next result is a corollary from Lemma~\ref{lem:conc_ineq_avg_rew}, and shows that $\bP\br{\cG\uc{\mbox{MF}}}$ holds with a high probability.

\begin{cor}\label{cor:con_ineq_mf}
    $\bP\br{\cG\uc{\mbox{MF}}} \geq 1 - \frac{\delta}{2}.$
\end{cor}
\begin{proof}
    First, we note that the cardinality of $\Phi\uc{f}_{\mbox{cover}}$ is at most $c_{z_1} 2^{d^\Phi_z+1} T^{\frac{d^\Phi_z}{d^\Phi_z + 2}}$.~We get from Lemma~\ref{lem:conc_ineq_avg_rew}, that the following holds with a probability greater than $1 - \frac{\delta}{2}$,
    \begin{align}
        &\frac{1}{N_\tau(\Phi\up)}\abs{\sum_{t=1}^{\tau}{\ind{\phi_t \in \Phi\up} J_\cM(\phi_t)} -\sum_{t=1}^{\tau}{\ind{\phi_t \in \Phi\up} r(s_t,\phi_t(s_t))}} \notag\\
        &\leq \frac{C}{1 - \alpha} \br{\sqrt{\frac{c\uc{f}_d}{N_\tau(\Phi\up)} \log{\br{\frac{T}{\delta}}}} + \frac{1 + K_\tau(\Phi\up)}{N_\tau(\Phi\up)}} = \diamf{\tau}{\phi}, \forall \tau \in \{1,2,\ldots,T-1\}, \forall \Phi\up \in \Phi\uc{f}_{\mbox{cover}}, \label{ineq:l11_1}
    \end{align}
    where $c\uc{f}_d$ satisfies~\eqref{def:cdf}.~This concludes the proof.
\end{proof}

\begin{lem}\label{lem:lb_phistr_diam}
    Let $\phi$ denote an active policy at time $t$ such that the optimal policy w.r.t. the policy class $\Phi$ belongs to $B_{t,\phi}$. Then, $\Delta_\Phi(\phi) \leq L_J \cdot \diamf{t}{\phi}$, where $L_J$ is as in Theorem~\ref{thm:lip_avg_reward}.
\end{lem}
\begin{proof}
    Let us assume that $\Delta_\Phi(\phi) > L_J \cdot \diamf{t}{\phi}$. By definition the optimal policy is $\diamf{t}{\phi}$-close to $\phi$. Hence, by Theorem~\ref{thm:lip_avg_reward}, $J\ust_{\cM,\Phi} \leq J_{\cM}(\phi) + L_J\cdot \diamf{t}{\phi} < J_{\cM}(\phi) +  \Delta_\Phi(\phi)$, which is not true by definition of $\Delta_\Phi(\phi)$.~Thus, $L_J \cdot \diamf{t}{\phi}$ must be greater than $\Delta_\Phi(\phi)$.
\end{proof}

\begin{lem}\label{lem:ub_nplay}
    On the set $\cG\uc{\mbox{MF}}$, \pzrlmf~has the following properties. 
    \begin{enumerate}
        \item It cannot activate more than one policy from the sets $\Phi\uc{f}_{\mbox{cover}}$.\label{prop:1}
        \item It does not play a policy from $\Phi_{2^{-i}}$ for more than $c\uc{f}_d \br{\frac{2C(2 + L_J)}{1 - \alpha}}^2 \log{\br{\frac{T}{\delta}}} ~2^{2i}$ time steps, where $i = 1, 2, \ldots, i\uc{f}$.\label{prop:2}
    \end{enumerate}
\end{lem}
\begin{proof}
    First, we will prove \ref{prop:1}.~Let us assume that at time $\tau$, each of the sets in $\Phi\uc{f}_{\mbox{cover}}$ contains at most one active policy.~Hence, on the set $\cG\uc{\mbox{MF}}$, for every policy $\phi \in A_\tau$ the following holds,
    \begin{align}
        J_\cM(\phi) &\leq \frac{1}{N_\tau(\phi)}\sum_{t=1}^{\tau}{\ind{\phi_t = \phi} r(s_t,\phi_t(s_t))} + \diamf{\tau}{\phi}, \mbox{ and}\label{ineq:l11_2}\\
        J_\cM(\phi) &\geq \frac{1}{N_\tau(\phi)}\sum_{t=1}^{\tau}{\ind{\phi_t = \phi} r(s_t,\phi_t(s_t))} - \diamf{\tau}{\phi},\label{ineq:l11_3}
    \end{align}
    where $\diamf{\tau}{\phi}$ is as defined in~\eqref{def:diamf}.~Let $\phi\up$ be a policy from $B_\tau(\phi)$. Using Theorem~\ref{thm:lip_avg_reward} in conjunction with \eqref{ineq:l11_2}, we have that for all such $\phi, \tau, \phi\up$ on $\cG\uc{\mbox{MF}}$,
    \begin{align}
        J_\cM(\phi\up) &\leq \frac{1}{N_\tau(\phi)}\sum_{t=1}^{\tau}{\ind{\phi_t = \phi} r(s_t,\phi_t(s_t))} + (1+L_J) \diamf{\tau}{\phi} = \idxf_{\tau}(\phi). \label{ineq:l11_4}
    \end{align}
    Upon using \eqref{ineq:l11_3} in the definition of index~(Definition~\ref{def:index_mf}), on $\cG\uc{\mbox{MF}}$ we get the following upper bound of $\idxf_\tau(\phi)$,
    \begin{align}
        \idxf_\tau(\phi) \leq J_\cM(\phi) + (2 + L_J) \diamf{\tau}{\phi},~\forall \phi \in A_\tau. \label{ineq:index_ub}
    \end{align}
    Also, by the covering invariance property and by \eqref{ineq:l11_4}, there is an active policy $\phi\ust_\tau$ such that on $\cG\uc{\mbox{MF}}$, 
    \begin{align}
        \idxf_\tau(\phi\ust_\tau) \geq J\ust_{\cM,\Phi}.\label{ineq:index_lb}
    \end{align}
    Note that in order to have two active policies from some $\Phi\up \in \Phi\uc{f}_{\mbox{cover}}$, we require that the diameter of at least one active policy must be less than $2^{-i}/(2+L_J)$, where $\Phi\up \subset \Phi_{2^{-i}}$.~Next, we show that this is not possible.

    Let $\phi \in \Phi_{2^{-i}}$ be an active policy and let $\diamf{t}{\phi} < 2^{-i}/(2+L_J)$. We can safely assume that $B_{t,\phi}$ does not contain any optimal policy since from Lemma~\ref{lem:lb_phistr_diam} we have that $\diamf{t}{\phi} \geq 2^{-i}/L_J$ if $B_{t,\phi}$ contains an optimal policy.~For every $\tau$ when the diameter of every active policy belonging to $\Phi_{2^{-i}}$ is larger than $2^{-i}/(2+L_J)$, \eqref{ineq:index_ub} and \eqref{ineq:index_lb} hold on the set $\cG\uc{\mbox{MF}}$.~We observe that if $J_\cM(\phi) + (2 + L_J) \diamf{\tau}{\phi} \leq J\ust_{\cM,\Phi}$, then $\idxf_\tau(\phi) \leq \idxf_\tau(\phi\ust_\tau)$. Consequently, the policy $\phi$ will never be chosen by~\pzrlmf; after performing some algebraic calculations, we conclude that the radius of $B_t(\phi)$ will never drop below $\Delta_\Phi(\phi)/(2+L_J)$. Since initially at $t=0$ there is only a single active policy, each set in $\Phi\uc{f}_{\mbox{cover}}$ contains at most one active policy for every $\tau \in \{0,1,\ldots, T-1\}$. This concludes the proof of \ref{prop:1}.

    Now we will prove the second statement.~For the diameter of a policy $\phi \in \Phi_{2^{-i}}$ to go below $\Delta_\Phi(\phi)/(2 + L_J)$, it enough to play it $N$ times where $N$ is the smallest natural number that satisfies
    \begin{align*}
        \frac{C}{1 - \alpha} \br{\sqrt{\frac{c\uc{f}_d}{N} \log{\br{\frac{T}{\delta}}}} + \frac{1 + \log_2{(N)}}{N}} \leq 2^{-i}/(2 + L_J),
    \end{align*}
    or,
    \begin{align*}
         N \geq c\uc{f}_d \br{\frac{2C(2 + L_J)}{1 - \alpha}}^2 \log{\br{\frac{T}{\delta}}} ~2^{2i}.
    \end{align*}
    This proves the second claim.
\end{proof}
\section{Analysis of \pzrlmb}
\label{app:analysis_mb}
In this section, we derive some results that are essential for the regret analysis of \pzrlmb~which is carried out in Appendix~\ref{app:regret}.~Let $i\uc{b}:= \ceil{\frac{\log_2{\br{T}}}{2 d_\cS + d^\Phi_z + 3}}$. Let $\bar{\Phi}\uc{b}_\gm$ be the set consisting of the centers of balls that yield the smallest $\frac{\gm(C_{ub}+L_J)}{2}$-covering of $\Phi_\gm$, where the radius $\gm$ is taken to be equal to $2^{-i}$, and $i \in  \{ 1, 2, \ldots, i\uc{b}\}$.~It follows from the definition of the zooming dimension~\eqref{def:zoomingdim} after performing few algebraic manipulations, that the cardinality of the set $\cup_{i=1}^{i\uc{b}}{\bar{\Phi}\uc{b}_{2^{-i}}}$ is at most $c_{z_1} 2^{d^\Phi_z+1} T^{\frac{d^\Phi_z}{2 d_\cS + d^\Phi_z + 3}}$.~Let us define the following covering of $\Phi$,
\begin{align}
    \Phi\uc{b}_{\mbox{cover}} := \{B_\phi(2^{-i}/2(C_{ub}+L_J)) \cap \Phi_{2^{-i}} : \phi \in \bar{\Phi}\uc{b}_{2^{-i}}, i = 1, 2, \ldots, i\uc{b}\}.\label{def:phib_cover}
\end{align}
In the first result of this section, we derive a threshold for a policy as a function of its sub-optimality so that if the diameter of the policy goes below that threshold, it will not be played by \pzrlmb~anymore.~We use this result to show that \pzrlmb~will activate at most one policy from each set in $\Phi\uc{b}_{\mbox{cover}}$. 

\begin{lem}\label{lem:min_diam}
    Consider a policy $\phi$ for which we have $\diamb{t}{\phi} \leq \Delta_\Phi(\phi)/ (C_{ub} + L_J)$, where $C_{ub}$ is as defined in \eqref{def:C_ub}, and $L_J$ is as defined in \eqref{def:LJ}. Then, on the set $\cG\uc{\mbox{Model}}$, $\phi$ is not played by \pzrlmb~during times $s\ge t$.~Moreover, on the set $\cG\uc{\mbox{Model}}$, \pzrlmb~cannot activate more than one policy from the sets in $\Phi\uc{b}_{\mbox{cover}}$.
\end{lem}
\begin{proof}
    From Lemma~\ref{lem:optimism} we have that on the set $\cG\uc{\mbox{Model}}$,
    \begin{align*}
        \idxb_t(\phi\up) \geq J\ust_{\cM,\Phi}, \mbox{ where } t \in \{0,1,\ldots,T-1\}, \mbox{ and }\phi \in \Phi^{act.}_t, \phi\up \in B_{t,\phi}.
    \end{align*}
    Using the covering invariance property of \pzrlmb~in the above, we can write,
    \begin{align}
        \max_{\phi \in \Phi^{act.}_t}{\idxb_t(\phi)} \geq J\ust_{\cM,\Phi},~\forall t \in \{0,1,\ldots,T-1\}.
    \end{align}
    Hence, we conclude that if the index of an active policy $\phi$ decreases below $J\ust_{\cM,\Phi}$, then $\phi$ will never be played by \pzrlmb.~From the upper bound on the index of a policy that was obtained in Lemma~\ref{lem:ub_opt_index}, we have the following: an active policy $\phi$ will never be played from time $t$ onwards, if $\diamb{t}{\phi} \leq \Delta_\Phi(\phi)/ (C_{ub} + L_J)$. This concludes the first part of the proof.

    Note that the diameter of a policy does not decrease if it is not played. Let $\phi \in \Phi_{2^{-i}}$ for some $i \in \{1, 2, \ldots, i\uc{b}\}$. Then, the diameter of $\phi$ would not decrease below $\frac{2^{-i}}{C_{ub} + L_J}$, and for this reason there will be no active policy that is $\frac{2^{-i}}{C_{ub} + L_J}$ close to $\phi$.~We note that in case the algorithm activates more than one policy from a set in $\Phi\uc{b}_{\mbox{cover}}$, this condition is violated. This proves the second claim.
\end{proof}

Now, for a policy $\phi$ that is activated by \pzrlmb, we derive a bound $N_\phi$ on the number of plays of $\phi$. We first derive a sufficient condition under which $\diamb{t}{\phi} \leq \Delta_\Phi(\phi)/ (C_{ub} + L_J)$ holds. 

Let us consider the partition of the state-action space where every cell is of level $\ell_\phi = \ceil{\log_2{\br{\frac{C_{ub} + L_J}{\Delta_\Phi(\phi)}}}}$, and let $\cP\uc{\phi}$ be the cells in this partition such that for every $\zeta \in \cP\uc{\phi}$, we have $(q(\pi_\cS(\zeta)), \phi(q(\pi_\cS(\zeta)))) \in \zeta$.~Note that the cardinality of $\cP\uc{\phi}$ is $2^{\ell_\phi d_\cS}$.~Consider those cells in $\cP\uc{\phi}$ to which the measure $\mu\uc{\infty}_{\phi,p}$ assigns a mass of at least $\br{\frac{\Delta_\Phi(\phi)}{2(C_{ub} + L_J)}}^{d_\cS+1}$.~We note that if all of these cells are deactivated by the algorithm, then we have,
\begin{align}
    &\int_{\cS}{\diamc{\zeta^\phi_s} \mu\uc{\infty}_{\phi,p}(s) ds} \notag\\
    &\leq \sum_{\substack{\zeta  \in \cP\uc{\phi}, \\\mu\uc{\infty}_{\phi,p}(\pi_\cS(\zeta)) \geq \br{\frac{\Delta_\Phi(\phi)}{2(C_{ub} + d_\cS)}}^{d_\cS+1}}}{\frac{\diamc{\zeta}}{2} \cdot \mu\uc{\infty}_{\phi,p}(\pi_\cS(\zeta))} + \sum_{\substack{\zeta  \in \cP\uc{\phi}, \\\mu\uc{\infty}_{\phi,p}(\pi_\cS(\zeta)) < \br{\frac{\Delta_\Phi(\phi)}{2(C_{ub} + d_\cS)}}^{d_\cS+1}}}{1 \cdot \mu\uc{\infty}_{\phi,p}(\pi_\cS(\zeta))} \notag\\
    &< \frac{\Delta_\Phi(\phi)}{2(C_{ub} + L_J)} + \frac{\Delta_\Phi(\phi)}{2(C_{ub} + L_J)} \notag\\
    &= \frac{\Delta_\Phi(\phi)}{C_{ub} + L_J} \label{cond:diam_cell}
\end{align}
where in order to obtain the second inequality, we use the fact that there are $2^{\ell_\phi d_\cS}$ cells in $\cP\uc{\phi}$.

Next, we derive a high probability lower bound on the number of visits to a cell, in the event that a policy that passes through this cell has been played sufficiently many times. Define,
\begin{align}
    \cG\uc{\mbox{Visit}} := &\Bigg\{\sum_{t=1}^{\tau}{\ind{\phi_t \in \Phi\up, (s_t,\phi_t(s_t)) \in \zeta_{L+}}} \geq \sum_{t=1}^{\tau}{\ind{\phi_t \in \Phi\up} \mu\uc{\infty}_{\phi_t,p}(\zeta)} - \frac{C}{1 - \alpha} \sqrt{N_t(\Phi\up)\log{\br{\frac{3T\abs{\Phi\uc{b}_{\mbox{cover}}} \abs{\cup_{i=0}^{i\uc{b}}{\cP\uc{i}}}}{\delta}}}} , \notag \\
    &\quad + \frac{C}{1 - \alpha} \br{K_\tau(\Phi\up) + 1}\mbox{ for every } t \in \{1, 2, \ldots, T-1\}, \Phi\up \in \Phi\uc{b}_{\mbox{cover}}, \zeta \in \cup_{i=0}^{i\uc{b}}{\cP\uc{i}} \Bigg\}. \label{def:Gvisit}
\end{align}

\begin{cor}\label{cor:bdd_visit}
    $\bP\br{\cG\uc{\mbox{Visit}}} \geq 1 - \frac{\delta}{3}$.
\end{cor}
\begin{proof}
    Fix a cell $\zeta \in \cup_{i=0}^{i\uc{b}}{\cP\uc{i}}$, and consider the MDP $\cM_\zeta := (\cS, \cA, p, \ind{\zeta})$ in which the agent receives a reward of $1$ unit at time $t$ in the event $(s_t,a_t) \in {\zeta}_{L+}$, and receives no reward otherwise.~Here, the subscript $L+$ denotes the $L$-neighborhood of the corresponding cell.~From Lemma~\ref{lem:conc_ineq_avg_rew}, we obtain that with a probability at least $1 - \frac{\delta}{3}$, the following holds,
    \begin{align*}
        \abs{\sum_{t=1}^{\tau}{\ind{\phi_t \in \Phi\up}\br{J_{\cM_\zeta}(\phi_t) - \ind{(s_t,\phi_t(s_t)) \in \zeta_{L+}}}}} \leq \frac{C}{1 - \alpha} \br{\sqrt{\sum_{t=1}^{\tau}{\ind{\phi_t \in \Phi\up}}\log{\br{\frac{3T\abs{\Phi\uc{b}_{\mbox{cover}}}}{\delta}}}} + K_\tau(\Phi\up) + 1},\\
        ~\forall \tau \in \{1,2,\ldots,T-1\}, \Phi\up \in \Phi\uc{b}_{\mbox{cover}},
    \end{align*}
    where $K_\tau(\Phi\up)$ is the number of episodes that started before time $t$ and played policy from $\Phi\up$, $\delta \in (0,1)$.~Noting that $\sum_{t=1}^{\tau}{\ind{\phi_t \in \Phi\up}} = N_t(\Phi)$, and $J_{\cM_\zeta}(\phi_t) = \mu\uc{\infty}_{\phi,p}(\xi)$, we rewrite the above result as follows: with probability at least $1 - \frac{\delta}{3}$,
    \begin{align*}
        \abs{\sum_{t=1}^{\tau}{\ind{\phi_t \in \Phi\up} \mu\uc{\infty}_{\phi_t,p}(\pi_\cS(\zeta))} - \sum_{t=1}^{\tau}{\ind{\phi_t \in \Phi\up, (s_t,\phi_t(s_t)) \in \zeta_{L+}}}} \leq \frac{C}{1 - \alpha} \br{\sqrt{N_t(\Phi\up)\log{\br{\frac{3T\abs{\Phi\uc{b}_{\mbox{cover}}}}{\delta}}}} + K_\tau(\Phi\up) + 1},\\
        ~\forall \tau \in \{1,2,\ldots,T-1\}, \Phi\up \in \Phi\uc{b}_{\mbox{cover}}.
    \end{align*}
    where $\delta \in (0,3)$.~Lastly, taking a union bound over all $\zeta \in \cup_{i=0}^{i\uc{b}}{\cP\uc{i}}$, we get that with probability at least $1 - \frac{\delta}{3}$,
    \begin{align}
        \abs{\sum_{t=1}^{\tau}{\ind{\phi_t \in \Phi\up} \mu\uc{\infty}_{\phi_t,p}(\pi_\cS(\zeta))} - \sum_{t=1}^{\tau}{\ind{\phi_t \in \Phi\up, (s_t,\phi_t(s_t)) \in \zeta_{L+}}}} \leq \frac{C}{1 - \alpha} \sqrt{N_t(\Phi\up)\log{\br{\frac{3T \abs{\Phi\uc{b}_{\mbox{cover}}} \abs{\cup_{i=0}^{i\uc{b}}{\cP\uc{i}}}}{\delta}}}} \notag\\
        + \frac{C}{1 - \alpha} \br{K_\tau(\Phi\up) + 1}, ~\forall \tau \in \{1,2,\ldots,T-1\}, \Phi\up \in \Phi\uc{b}_{\mbox{cover}}, \zeta \in \cup_{i=0}^{i\uc{b}}{\cP\uc{i}}. \label{ineq:visit}
    \end{align}
    This concludes the proof.
\end{proof}

Next, using the result from Lemma~\ref{lem:min_diam}, we show the an useful property of the set $\cG\uc{\mbox{Visit}} \cap \cG\uc{\mbox{Model}}$.
\begin{cor}\label{cor:pro_VcapM}
    Consider a sample path from the set $\cG\uc{\mbox{Visit}} \cap \cG\uc{\mbox{Model}}$. Then, the following holds for all times $t = 1, 2, \ldots, T-1$: If the policy $\phi_t$ satisfies $\Delta_\Phi(\phi_t) > 2^{-i\uc{b}}$, then we have,
    \begin{align*}
        \sum_{t=0}^{\tau}{\ind{(s_t,\phi_t(s_t)) \in \zeta_{L+}}} \geq N_\tau(\phi_t) \mu\uc{\infty}_{\phi,p}(\pi_\cS(\zeta)) - \frac{C}{1 - \alpha} \br{\sqrt{N_\tau(\phi_t)  c\uc{1}_v \log{\frac{T}{\delta}}} + \log_2(N_\tau(\phi_t)) + 1},
    \end{align*}
    for every cell $\zeta \in \cP\uc{\phi_t}$, where the constant $c\uc{1}_v$ is defined as follows,
    \begin{align}
        c\uc{1}_v  := \frac{\log{\br{\frac{3c_{z_1}}{\delta} 2^{d^\Phi_z + 2} T^{1+\frac{d + d^\Phi_z}{2 d_\cS + d^\Phi_z + 3}}}} }{\log{\br{\frac{T}{\delta}}}}. \label{def:cv1}
    \end{align}
\end{cor}

\begin{proof}
    We discussed the following bound previously in this section
    \begin{align}
        \abs{\Phi\uc{b}_{\mbox{cover}}} \leq c_{z_1} 2^{d^\Phi_z+1} T^{\frac{d^\Phi_z}{2 d_\cS + d^\Phi_z + 3}}. \label{ub:phi_cov}
    \end{align}
    Note that
    \begin{align}
        \abs{\cup_{i=0}^{i\uc{b}}{\cP\uc{i}}} \leq \sum_{i=0}^{i\uc{b}}{2^{i d}} \leq 2 T^{\frac{d}{2 d_\cS + d^\Phi_z + 3}}.\label{ub:cell_set}
    \end{align}
    From \eqref{ub:phi_cov} and \eqref{ub:cell_set}, we have that
    \begin{align*}
        \frac{3T \abs{\Phi\uc{b}_{\mbox{cover}}} \abs{\cup_{i=0}^{i\uc{b}}{\cP\uc{i}}}}{\delta} = \frac{3c_{z_1}}{\delta} 2^{d^\Phi_z + 2} T^{1+\frac{d + d^\Phi_z}{2 d_\cS + d^\Phi_z + 3}}.
    \end{align*}
    From Lemma~\ref{lem:min_diam} we have that \pzrlmb~activates at most one policy from each of the sets $\Phi\up \in \Phi\uc{b}_{\mbox{cover}}$.~We use these observations above in conjunction with \eqref{ineq:visit} in order to conclude that on the set $\cG\uc{\mbox{Visit}} \cap \cG\uc{\mbox{Model}}$, we have,
    \begin{align*}
        \sum_{t=0}^{\tau}{\ind{(s_t,\phi_t(s_t)) \in \zeta_{L+}}} \geq N_\tau(\phi_t) \mu\uc{\infty}_{\phi,p}(\pi_\cS(\zeta)) - \frac{C}{1 - \alpha} \br{\sqrt{N_\tau(\phi_t)  c\uc{1}_v \log{\frac{T}{\delta}}} + \log_2(N_\tau(\phi_t)) + 1},\\
        \forall t \in \{1,2,\ldots, T-1\},
    \end{align*}
    where while deriving the above, we have used the fact that $K_\tau(\phi) \leq \log_2(N_\tau(\phi))$ for any $\phi$ while playing \pzrlmb. This concludes the proof.
\end{proof}

\begin{lem}\label{lem:bd_num_play}
    On the set $\cG\uc{\mbox{Visit}} \cap \cG\uc{\mbox{Model}}$, \pzrlmb~does not play any policy $\phi \in \Phi_{2^{-i}}$ for more than
    \begin{align*}
        \br{\br{\br{\frac{2C}{1 - \alpha}}^2 c\uc{1}_v + 3 c\uc{b}_d} \log{\br{\frac{T}{\delta}}} \br{2(C_{ub} + L_J)}^{2d_\cS+3}} \cdot 2^{i(2d_\cS+3)}
    \end{align*}
    time steps, where $i = 1, 2, \ldots, i\uc{b}$.
\end{lem}
\begin{proof}
The following is seen to hold for $x\in\bN$,
\nal{
2 \sqrt{x~ c\uc{1}_v\log{\br{\frac{T}{\delta}}}} \geq \sqrt{x~ c\uc{1}_v\log{\br{\frac{T}{\delta}}}} + \log_2(x) + 1.
}
Upon using the above in conjunction with Corollary~\ref{cor:pro_VcapM}, we obtain that the following holds on the set $\cG\uc{\mbox{Visit}} \cap \cG\uc{\mbox{Model}}$: for a policy $\phi$ that is active at time $\tau$, where $\phi \in \Phi_{2^{-i}}$, $i \in \{1, 2, \ldots, i\uc{b}\}$, and for any cell $\zeta$ such that $(q(\pi_\cS(\zeta)), \phi(q(\pi_\cS(\zeta)))) \in \zeta$, we have,
    \begin{align}
        \sum_{t=0}^{\tau}{\ind{(s_t,\phi_t(s_t)) \in \zeta_{L+}}} \geq N_\tau(\phi) \mu\uc{\infty}_{\phi,p}(\pi_\cS(\zeta)) - \frac{2 C}{1 - \alpha} \sqrt{N_\tau(\phi) c\uc{1}_v\log{\br{\frac{T}{\delta}}}}.\label{ub:num_plays}
    \end{align}
    It is easily verified that in case the r.h.s. of \eqref{ub:num_plays} exceeds $\frac{c\uc{b}_d 2^{d_\cS+2} \log{\br{\frac{T}{\delta}}}}{\diamc{\zeta}^{d_\cS+2}}$, then we have $\int_{\cS}{\diamc{\zeta^\phi_s} \mu\uc{\infty}_{\phi,p}(s) ds} < \frac{\Delta_\Phi(\phi)}{C_{ub} + L_J}$, or equivalently the following holds,
    \begin{align}
        N_\tau(\phi) \cdot \mu\uc{\infty}_{\phi,p}(\pi_\cS(\zeta)) - \frac{2C}{1 - \alpha} \sqrt{c\uc{1}_v\log{\br{\frac{T}{\delta}}}} \cdot \sqrt{N_\tau(\phi)} - c\uc{b}_d \log{\br{\frac{T}{\delta}}}\br{\frac{\Delta_\Phi(\phi)}{2(C_{ub} + L_J)}}^{-(d_\cS+2)} \geq 0. \label{ineq:Num_visit}
    \end{align}
    Upon invoking Lemma~\ref{lem:fx}, we obtain that \eqref{ineq:Num_visit} is satisfied when
    \begin{align*}
        N_\tau(\phi) \geq \br{\frac{\frac{2C}{1 - \alpha} \sqrt{c\uc{1}_v\log{\br{\frac{T}{\delta}}}}}{\mu\uc{\infty}_{\phi,p}(\pi_\cS(\zeta))}}^2 + \frac{3 c\uc{b}_d \log{\br{\frac{T}{\delta}}}\br{\frac{\Delta_\Phi(\phi)}{2(C_{ub} + L_J)}}^{-(d_\cS+2)}}{\mu\uc{\infty}_{\phi,p}(\pi_\cS(\zeta))}.
    \end{align*}
    Since $\phi \in \Phi_{2^{-i}}$, we have $\Delta_\Phi(\phi) \geq 2^{-i}$.~Recall from Lemma~\ref{lem:min_diam} that it suffices to visit cells with stationary measure $\br{\frac{\Delta_\Phi(\phi)}{2(C_{ub} + L_J)}}^{d_\cS+1}$ under policy $\phi$ in order to shrink the diameter of $\phi$.~Now, replacing the lower bound of $\mu\uc{\infty}_{\phi,p}(\pi_\cS(\zeta))$ and $\Delta_\Phi(\phi)$ in the above and simplifying it further, we get that \eqref{ineq:Num_visit} is satisfied for all
    \begin{align*}
        N \geq \br{\br{\br{\frac{2C}{1 - \alpha}}^2 c\uc{1}_v + 3 c\uc{b}_d} \log{\br{\frac{T}{\delta}}} \br{2(C_{ub} + L_J)}^{2d_\cS+3}} \cdot 2^{i(2d_\cS+3)}.
    \end{align*}
    This concludes the proof.
\end{proof}

\section{Regret Analysis}\label{app:regret}
\textbf{Regret decomposition:} 
Recall the regret decomposition of regret~\eqref{def:regret},
\begin{align}
    \cR_\Phi(T;\psi) &= \sum_{t=0}^{T-1}{J\ust_{\cM,\Phi} - r(s_t,a_t)} \notag\\
    &= \underbrace{\sum_{t=0}^{T-1}{\br{J\ust_{\cM,\Phi} - J_\cM(\phi_t)}}}_{(a)} + \underbrace{\sum_{t=0}^{T-1}{\br{J_\cM(\phi_t) - r(s_t,\phi_t(s_t))}}}_{(b)}.\label{eq:decompregret}
\end{align}

The term (a) captures the regret arising due to the gap between the optimal value of the average reward and the average reward of the policies $\{\phi_k\}_{k=1}^{K}$ that are actually played in different episodes, while (b) captures the sub-optimality arising since the distribution of the induced Markov chain does not reach the stationary distribution in finite time.~(a) and (b) are bounded separately.~We firstly bound the term (b); this bound is the same for both the algorithms~\pzrlmf~and~\pzrlmf. The bounds on the term (a) vary for \pzrlmf~and \pzrlmb,~and hence are derived separately.
\begin{prop}\label{prop:fluc_ub}
    Consider,
    \begin{align}
        \sum_{t=1}^{T-1}{J_\cM(\phi_t) - r(s_t,\phi_t(s_t))} \leq \frac{C}{1-\alpha} \sqrt{\frac{T}{2} \log{\br{\frac{1}{\delta}}}} + \frac{C}{2(1 - \alpha)} (1 + K(T)), \label{bdd:fluctuation}
    \end{align}
    where $K(T)$ denotes the total number of episodes until time $T$. Denote 
\al{
\cG\uc{\mbox{Fluc.}}_{\delta} := \left\{ \omega:~\eqref{bdd:fluctuation} \mbox{ holds } \right\}.
}
For both the algorithms \pzrlmf~and~\pzrlmb~we have,
\al{
\bP\br{\cG\uc{\mbox{Fluc.}}_{\delta}} \geq 1 - \delta,~\delta \in (0,1). \label{bdd:fluc}
}
\end{prop}
\begin{proof}
    We begin by converting the Markovian noise to a martingale difference sequence, using a similar technique that was used in the proof of Lemma~\ref{lem:conc_ineq_avg_rew}, i.e.,
    \begingroup
        \allowdisplaybreaks
        \begin{align}
            &\sum_{t=0}^{T-1}{J_\cM(\phi_t) - r(s_t,\phi_t(s_t))} \notag\\
            &= \sum_{t=0}^{T-1}{\int_{\cS}{h_\cM(s;\phi_t) p(s_t,\phi_t(s_t), ds)} - h_\cM(s_t,\phi_t)} \notag\\
            &= \sum_{t=1}^{T-1}{\int_{\cS}{h_\cM(s;\phi_t) p(s_{t-1},\phi_{t-1}(s_{t-1}), ds)} - h_\cM(s_t,\phi_t)} \notag\\
            &\quad + \sum_{t=1}^{T-1}{\int_{\cS}{ h_\cM(s;\phi_t) p(s_t,\phi_t(s_t), ds)} - \int_{\cS}{ h_\cM(s;\phi_t) p(s_{t-1},\phi_{t-1}(s_{t-1}), ds)}} \notag\\
            &\quad + \int_{\cS}{h_\cM(s;\phi_0) p(s_0,\phi_0(s_0), ds)} - h_\cM(s_0,\phi_0) \notag\\
            &= \sum_{t=1}^{T-1}{\int_{\cS}{h_\cM(s;\phi_t) p(s_{t-1},\phi_{t-1}(s_{t-1}), ds)} - h_\cM(s_t,\phi_t)} \notag\\
            &\quad + \sum_{t=1}^{T-1}{\int_{\cS}{\br{h_\cM(s;\phi_t) -  h_\cM(s;\phi_{t-1})} p(s_{t-1},\phi_{t-1}(s_{t-1}), ds)}} \notag\\
            &\quad + \int_{\cS}{h_\cM(s;\phi_{T-1}) p(s_{T-1},\phi_{T-1}(s_{T-1}), ds)} - h_\cM(s_0,\phi_0). \label{ineq:abs_dif_2}
        \end{align}
    \endgroup    
    Now consider the first term in the r.h.s. of \eqref{ineq:abs_dif_2}.~Denote $\eta_t = \int_{\cS}{h_\cM(s;\phi_t) p(s_{t-1},\phi_{t-1}(s_{t-1}), ds)} - h_\cM(s_t,\phi_t)$.~Noting that for both \pzrlmf~as well as~\pzrlmb, $\phi_t$ is $\cF_{t-1}$-measurable, we obtain the following:
    \begin{align*}
        \bE\sqbr{\eta_t \mid \cF_{t-1}} &= \bE\sqbr{\int_{\cS}{h_\cM(s;\phi_t) p(s_{t-1},\phi_{t-1}(s_{t-1}), ds)} - h_\cM(s_t,\phi_t) \mid \cF_{t-1}} \\
        &= \int_{\cS}{h_\cM(s;\phi_t) p(s_{t-1},\phi_{t-1}(s_{t-1}), ds)} - \int_{\cS}{h_\cM(s;\phi_t) p(s_{t-1},\phi_{t-1}(s_{t-1}), ds)}\\
        &=0.
    \end{align*}
    Hence, $\flbr{\eta_t}$ is a martingale difference sequence.~Also, from the bound on the span of $h_\cM(\cdot;\phi)$ that was derived in Lemma~\ref{lem:bdd_rvf_spn}, we have that $\eta_t \in \sqbr{-\frac{C}{2(1-\alpha)}, \frac{C}{2(1-\alpha)}}$.~An application of Azuma-Hoeffding inequality~(Lemma~\ref{lem:ah_ineq}), yields the following: for each $\delta \in (0,1)$, with probability at least $1 - \delta$ we have,
    \begin{align*}
        \sum_{t=1}^{T-1}{\int_{\cS}{h_\cM(s;\phi_t) p(s_{t-1},\phi_{t-1}(s_{t-1}), ds)} - h_\cM(s_t,\phi_t)} \leq \frac{C}{1-\alpha} \sqrt{\frac{T}{2} \log{\br{\frac{1}{\delta}}}}.
    \end{align*}
    Now, consider the second term in the r.h.s. of \eqref{ineq:abs_dif_2}. The $t$-th element in this summation can assume a non-zero value only when a new episode starts at time $t$.~So, the second term can be upper-bounded by $\frac{C}{2(1 - \alpha)} K(T)$ using Lemma~\ref{lem:bdd_rvf_spn}, where $K(T)$ denotes the number of episodes that have been started until time $T$ by the learning algorithm.~Again by using Lemma~\ref{lem:bdd_rvf_spn}, the third term can be bounded by $\frac{C}{2(1 - \alpha)}$.

    Putting all the individual bounds together, we have that for any $\delta \in (0,1)$ with probability at least $1 - \delta$,
    \begin{align}
        \sum_{t=1}^{T-1}{J_\cM(\phi_t) - r(s_t,\phi_t(s_t))} \leq \frac{C}{1-\alpha} \sqrt{\frac{T}{2} \log{\br{\frac{1}{\delta}}}} + \frac{C}{2(1 - \alpha)} (1 + K(T)).\label{ub:b}
    \end{align}
    This concludes the proof.
\end{proof}
In the next two consecutive lemmata, we derive upper-bounds on $K(T)$ for \pzrlmf~and \pzrlmb, respectively.

\begin{lem}\label{lem:bd_num_epi_mf}
    On the set $\cG\uc{\mbox{MF}}$, the number of total episodes played by \pzrlmf~is bounded above as
    \begin{align*}
        K(T) \leq \br{\br{2(c_{z_1} + c_{z_2}) + 2^{-c\uc{f}_k}}  \cdot \br{2i\uc{f} + \ceil{c\uc{f}_k}} + 2^{-c\uc{f}_k}} T^{\frac{d^\Phi_z}{d^\Phi_z+2}},
    \end{align*}
    where,
    \begin{align}
        c\uc{f}_k := 2\log_2{\br{\frac{2C(2 + L_J)}{1 - \alpha} \sqrt{c\uc{f}_d\log{\br{\frac{T}{\delta}}}}}},\label{def:cfk}
    \end{align}
    and $i\uc{f} = \ceil{\frac{\log{\br{\frac{T}{\delta}}}}{d_z^\phi + 2}}$.
\end{lem}
\begin{proof}
    Firstly, we note that in order for \pzrlmf~to play a policy for $N$ steps, it must this policy in at least $\ceil{\log_2(N)}$ episodes.~Combining this observation with Lemma~\ref{lem:ub_nplay}, we conclude that on the set $\cG\uc{\mbox{MF}}$, \pzrlmf~can play a policy from the set $\Phi_{2^{-i}}$ in at most $2i + \ceil{c\uc{f}_k}$ episodes.~Also, from Lemma~\ref{lem:ub_nplay}, we know that on the set $\cG\uc{\mbox{MF}}$, for any $i \in \bN$ the number of active policies from $\Phi_{2^{-i}}$ is at most $c_{z_1} 2^{i d^\Phi_z}$.~Combining these two properties of \pzrlmf, we conclude that on the set $\cG\uc{\mbox{MF}}$, there are at most $c_{z_1}\sum_{i=1}^{i\uc{f}}{\br{2i + \ceil{c\uc{f}_k}} 2^{i d^\Phi_z}}$ episodes where policies from $\cup_{i=1}^{i\uc{f}}{\Phi_{2^{-i}}}$ will be played.
    
    Now, let us bound the number of episodes in which policies from the set $\Phi_{\leq 2^{-i\uc{f}}}$ are played.~Assume that $\cN^{\mbox{pack}}_{2^{-i\uc{f}}/(2 + L_J)}\br{\Phi_{\leq 2^{-i\uc{f}}}}$ policies have already been activated from the set $\Phi_{\leq 2^{-i\uc{f}}}$.~From the definition of packing number~(Definition~\ref{def:sizeofset}), for a policy from the set $\Phi_{\leq 2^{-i\uc{f}}}$ to get activated thereafter, the diameter of at least one policy from the set of active policies in $\Phi_{\leq 2^{-i\uc{f}}}$ needs to drop below $2^{-i\uc{f}}/(2 + L_J)$.~Note that a policy needs to be played for at least $2^{2i\uc{f} + c\uc{f}_k}$ time steps in order that its diameter drops below $2^{-i\uc{f}}/(2 + L_J)$.~Hence, the number of policies whose diameter can shrink below $2^{-i\uc{f}}/(2 + L_J)$, is bounded above by $T \cdot 2^{-(2i\uc{f} + c\uc{f}_k)}$. Thus, there can be at most 
    \begin{align*}
        \cN^{\mbox{pack}}_{2^{-i\uc{f}}/(2 + L_J)}\br{\Phi_{\leq 2^{-i\uc{f}}}} + T \cdot 2^{-(2i\uc{f} + c\uc{f}_k)}
    \end{align*}
    active policies in the set $\Phi_{\leq 2^{-i\uc{f}}}$.~Also, there can be at most a total of $T \cdot 2^{-(2i\uc{f} + c\uc{f}_k)}$ episodes that are longer than $2^{2i\uc{f} + c\uc{f}_k - 1}$ steps.~Each of the other active policies from $\Phi_{\leq 2^{-i\uc{f}}}$ are played in at most $2 i\uc{f} + \ceil{c\uc{f}_k}$ episodes.~Hence, there can be at most
    \begin{align*}
        \br{\cN^{\mbox{pack}}_{2^{-i\uc{f}}/(2 + L_J)}\br{\Phi_{\leq 2^{-i\uc{f}}}} + T \cdot 2^{-(2i\uc{f} + c\uc{f}_k)}} \times \br{2 i\uc{f} + \ceil{c\uc{f}_k}} + T \cdot 2^{-(2i\uc{f} + c\uc{f}_k)}
    \end{align*}
    episodes in which policies from $\Phi_{\leq 2^{-i\uc{f}}}$ are played.~Note that 
    \begin{align*}
        \cN^{\mbox{pack}}_{2^{-i\uc{f}}/(2 + L_J)}\br{\Phi_{\leq 2^{-i\uc{f}}}} &\leq \cN_{2^{-i\uc{f}}/2(2 + L_J)}\br{\Phi_{\leq 2^{-i\uc{f}}}}\\
        &\leq c_{z_2} 2^{i\uc{f} d^\Phi_z} \\
        &\leq 2c_{z_2} T^{\frac{d^\Phi_z}{d^\Phi_z+2}}.
    \end{align*}
    Also, note that~$T \cdot 2^{-2i\uc{f}} \leq T^\frac{d_z^\phi}{d_z^\phi + 2}$.
    
    Upon summing up the number of episodes corresponding to policies in $\cup_{i=1}^{i\uc{f}}{\Phi_{2^{-i}}}$ and $\Phi_{\leq 2^{-i\uc{f}}}$, and using the last two observations, we get the following bound on the number of episodes on the set $\cG\uc{\mbox{MF}}$:
    \begin{align*}
        K(T) &\leq c_{z_1}\sum_{i=1}^{i\uc{f}}{\br{2i + \ceil{c\uc{f}_k}} 2^{i d^\Phi_z}} + \br{\cN^{\mbox{pack}}_{2^{-i\uc{f}}/(2 + L_J)}\br{\Phi_{\leq 2^{-i\uc{f}}}} + T \cdot 2^{-(2i\uc{f} + c\uc{f}_k)}} \times \br{2 i\uc{f} + \ceil{c\uc{f}_k}}\\
        &\quad + T \cdot 2^{-(2i\uc{f} + c\uc{f}_k)} \\
        &\leq \br{\br{2(c_{z_1} + c_{z_2}) + 2^{-c\uc{f}_k}}  \cdot \br{2i\uc{f} + \ceil{c\uc{f}_k}} + 2^{-c\uc{f}_k}} T^{\frac{d^\Phi_z}{d^\Phi_z+2}}.
    \end{align*}
    This concludes the proof.
\end{proof}

\begin{lem}\label{lem:bd_num_epi_mb}
    On the set $\cG\uc{\mbox{Model}} \cap \cG\uc{\mbox{Visit}}$, the total number of episodes for~\pzrlmb~is bounded above as follows,
    \begin{align*}
        K(T) \leq \br{\br{2(c_{z_1} + c_{z_2}) + 2^{-c\uc{b}_k}}  \cdot \br{(2d_\cS + 3)i\uc{b} + \ceil{c\uc{b}_k}} + 2^{-c\uc{b}_k}} T^{\frac{d^\Phi_z}{2 d_\cS + d^\Phi_z + 3}},
    \end{align*}
    where,
    \begin{align}
        c\uc{b}_k :=  \log_{2}\br{\br{\br{\frac{2C}{1 - \alpha}}^2 c\uc{1}_v + 3 c\uc{b}_d} \log{\br{\frac{T}{\delta}}} \br{2(C_{ub} + L_J)}^{2d_\cS+3}}, \label{def:cbk}
    \end{align}
    and $i\uc{b} = \ceil{\frac{\log{\br{\frac{T}{\delta}}}}{2 d_\cS + d_z^\phi + 3}}$.
\end{lem}

\begin{proof}
    By Lemma~\ref{lem:bd_num_play} and by the rule of setting the episode duration, on the set $\cG\uc{\mbox{Visit}} \cap \cG\uc{\mbox{Model}}$, \pzrlmb~plays any policy from $\Phi_{2^{-i}}$ at most in 
    \begin{align*}
        \log_2{\br{\br{\br{\frac{2C}{1 - \alpha}}^2 c\uc{1}_v + 3 c\uc{b}_d} \log{\br{\frac{T}{\delta}}} \br{2(C_{ub} + L_J)}^{2d_\cS+3}} \cdot 2^{i(2d_\cS+3)}} = (2d_\cS+3)i + c\uc{b}_k
    \end{align*}
    episodes.~The remaining proof is similar to the proof of Lemma~\ref{lem:bd_num_epi_mf}.
\end{proof}

\subsection{Proof of Theorem~\ref{thm:reg_ub_mf}}
\begin{proof}
    We derive upper bounds on the terms (a) and (b) of \eqref{eq:decompregret} separately. 
    
    Firstly we bound (a).~From Lemma~\ref{lem:ub_nplay}, we have that on the set $\cG\uc{\mbox{MF}}$,
    
    (i) There will be at most $c_{z_1} 2^{i d^\Phi_z}$ active policies from $\Phi_{2^{-i}}$ for all $i = 1, 2, \ldots, i\uc{f}$, where $i\uc{f} = \ceil{\frac{\log_2{(T)}}{d^\Phi_z + 2}}$. 
    
    (ii) The active policies from $\Phi_{2^{-i}}$ will not be played more than $2^{c\uc{f}_k} ~2^{2i}$ times for all $i = 1, 2, \ldots, i\uc{f}$, where $c\uc{f}_k$ is as defined in~\eqref{def:cfk}.
    
    Hence, on the set $\cG\uc{\mbox{MF}}$, by summing the upper bounds on the cumulative sub-optimalities arising from playing policies from $\Phi_{2^{-i}}$ for $i=1, 2, \ldots, i\uc{f}$, and considering the worst possible sub-optimality from amongst the remaining the policies, we obtain an upper bound on the term (a) as follows:
    \begin{align}
        \sum_{t=0}^{T-1}{J\ust_{\cM,\Phi} - J_\cM(\phi_t)} &\leq c_{z_1} 2^{c\uc{f}_k} \sum_{i=1}^{i\uc{f}}{2^{i(d^\Phi_z+1)}} + 2^{-i\uc{f}} T \notag\\
        &\leq 2c_{z_1} 2^{c\uc{f}_k} 2^{i\uc{f}(d^\Phi_z+1)} + 2^{-i\uc{f}} T \notag\\
        &\leq (4 c_{z_1} 2^{c\uc{f}_k} + 1)T^{\frac{d^\Phi_z+1}{d^\Phi_z+2}}.\label{bd:a_mf}
    \end{align}
    Next, we bound the term (b).~From Lemma~\ref{prop:fluc_ub}, we have the following on the set $\cG\uc{\mbox{Fluc.}}_{\frac{\delta}{2}}$,
    \begin{align*}
        \sum_{t=1}^{T-1}{J_\cM(\phi_t) - r(s_t,\phi_t(s_t))} \leq \frac{C}{1-\alpha} \sqrt{\frac{T}{2} \log{\br{\frac{2}{\delta}}}} + \frac{C}{2(1 - \alpha)} (1 + K(T)).
    \end{align*}
    Upon substituting the upper bound on $K(T)$ obtained in \eqref{lem:bd_num_epi_mf} in the above, we have that the following holds on $\cG\uc{\mbox{Fluc.}}_{\frac{\delta}{2}}$,
    \begin{align}
        \sum_{t=1}^{T-1}{J_\cM(\phi_t) - r(s_t,\phi_t(s_t))} &\leq \frac{C}{1-\alpha} \sqrt{\frac{T}{2} \log{\br{\frac{2}{\delta}}}} \notag\\
        &\quad + \frac{C}{2(1 - \alpha)} \br{1 + \br{\br{2(c_{z_1} + c_{z_2}) + 1}  \cdot \br{2i\uc{f} + \ceil{c\uc{f}_k}} + 1} T^{\frac{d^\Phi_z}{d^\Phi_z+2}}}. \label{bd:b_mf}
    \end{align}
    Upon summing the upper bounds of term (a) and term (b) from \eqref{bd:a_mf} and \eqref{bd:b_mf}, respectively, we obtain that on the set $\cG\uc{\mbox{MF}} \cap \cG\uc{\mbox{Fluc.}}_{\frac{\delta}{2}}$ the following holds:
    \begin{align*}
        \cR_\Phi(T;\psi) &\leq (4 c_{z_1} 2^{c\uc{f}_k} + 1)T^{\frac{d^\Phi_z+1}{d^\Phi_z+2}} + \frac{C}{1-\alpha} \sqrt{\frac{T}{2} \log{\br{\frac{2}{\delta}}}}\\
        &\quad + \frac{C}{2(1 - \alpha)} \br{1 + \br{\br{2(c_{z_1} + c_{z_2}) + 1}  \cdot \br{2i\uc{f} + \ceil{c\uc{f}_k}} + 1} T^{\frac{d^\Phi_z}{d^\Phi_z+2}}}.
    \end{align*}
    From Corollary~\ref{cor:con_ineq_mf} and \eqref{bdd:fluc}, upon using a union bound, we obtain $\bP\br{\cG\uc{\mbox{MF}} \cap \cG\uc{\mbox{Fluc.}}_{\frac{\delta}{2}}} \geq 1- \delta$.~This concludes the proof.
\end{proof}

\subsection{Proof of Theorem~\ref{thm:reg_ub_mb}}
\begin{proof}
    We begin by bounding the term (a) in~\eqref{eq:decompregret}.~From Lemma~\ref{lem:min_diam}, we have that on the set $\cG\uc{\mbox{Model}} \cap \cG\uc{\mbox{Visit}}$, there will be at most $c_{z_1} 2^{i d^\Phi_z}$ active policies from $\Phi_{2^{-i}}$ for all $i = 1, 2, \ldots, i\uc{b}$, where $i\uc{b} = \ceil{\frac{\log_2{(T)}}{2 d_\cS + d^\Phi_z + 3}}$.~From Lemma~\ref{lem:bd_num_play}, we have that on the set $\cG\uc{\mbox{Model}} \cap \cG\uc{\mbox{Visit}}$, the active policies from $\Phi_{2^{-i}}$ will not be played more than $2^{c\uc{b}_k} ~2^{(2d_\cS + 3)i}$ times for all $i = 1, 2, \ldots, i\uc{b}$, where $c\uc{b}_k$ is defined in~\eqref{def:cbk}.~Upon summing the upper bounds of the cumulative sub-optimalities arising from playing policies from $\Phi_{2^{-i}}$ for $i=1, 2, \ldots, i\uc{b}$ and considering the worst sub-optimality from the rest of the policies, we obtain the upper bound on the term (a) on the set $\cG\uc{\mbox{Model}} \cap \cG\uc{\mbox{Visit}}$:    
    \begin{align}
        \sum_{t=0}^{T-1}{J\ust_{\cM,\Phi} - J_\cM(\phi_t)} &\leq c_{z_1} 2^{c\uc{b}_k} \sum_{i=1}^{i\uc{b}}{2^{i(2 d_\cS + d^\Phi_z+ 2)}} + 2^{-i\uc{b}} T \notag\\
        &\leq 2c_{z_1} 2^{c\uc{b}_k} 2^{i\uc{b}(2 d_\cS + d^\Phi_z+ 2)} + 2^{-i\uc{b}} T \notag\\
        &\leq (4 c_{z_1} 2^{c\uc{b}_k} + 1)T^{\frac{2 d_\cS + d^\Phi_z+ 2}{2 d_\cS + d^\Phi_z+ 3}}.\label{bd:a_mb}
    \end{align}
    Next, we will bound the term (b) in~\eqref{eq:decompregret}.~From Lemma~\ref{prop:fluc_ub}, we have that on the set $\cG\uc{\mbox{Fluc.}}_{\frac{\delta}{3}}$, the following holds,
    \begin{align*}
        \sum_{t=1}^{T-1}{J_\cM(\phi_t) - r(s_t,\phi_t(s_t))} \leq \frac{C}{1-\alpha} \sqrt{\frac{T}{2} \log{\br{\frac{3}{\delta}}}} + \frac{C}{2(1 - \alpha)} (1 + K(T)).
    \end{align*}
    Upon substituting the upper bound on $K(T)$ from Lemma~\ref{lem:bd_num_epi_mb} in the above, we obtain,
    \begin{align}
        \sum_{t=1}^{T-1}{J_\cM(\phi_t) - r(s_t,\phi_t(s_t))} &\leq \frac{C}{1-\alpha} \sqrt{\frac{T}{2} \log{\br{\frac{3}{\delta}}}} \notag\\
        + \frac{C}{2(1 - \alpha)} & \br{1 + \br{\br{2(c_{z_1} + c_{z_2}) + 1}  \cdot \br{(2d_\cS + 3)i\uc{b} + \ceil{c\uc{b}_k}} + 1} T^{\frac{d^\Phi_z}{2 d_\cS + d^\Phi_z + 3}}}, \label{bd:b_mb}
    \end{align}
    on the set $\cG\uc{\mbox{Fluc.}}_{\frac{\delta}{3}}$.~Summing the upper bounds on the terms (a) and (b) from \eqref{bd:a_mb} and \eqref{bd:b_mb}, respectively, we obtain the on the set $\cG\uc{\mbox{Model}} \cap \cG\uc{\mbox{Visit}} \cap \cG\uc{\mbox{Fluc.}}_{\frac{\delta}{3}}$, the following holds,
    \begin{align*}
        \cR_\Phi(T;\psi) &\leq (4 c_{z_1} 2^{c\uc{b}_k} + 1)T^{\frac{2 d_\cS + d^\Phi_z+ 2}{2 d_\cS + d^\Phi_z+ 3}} + \frac{C}{1-\alpha} \sqrt{\frac{T}{2} \log{\br{\frac{3}{\delta}}}} \notag\\
        &\quad + \frac{C}{2(1 - \alpha)} \br{1 + \br{\br{2(c_{z_1} + c_{z_2}) + 1}  \cdot \br{(2d_\cS + 3)i\uc{b} + \ceil{c\uc{b}_k}} + 1} T^{\frac{d^\Phi_z}{2 d_\cS + d^\Phi_z + 3}}}.
    \end{align*}
    From Lemma~\ref{lem:conc_ineq}, Corollary~\ref{cor:bdd_visit} and \eqref{bdd:fluc}, upon using a union bound, we obtain $\cG\uc{\mbox{Model}} \cap \cG\uc{\mbox{Visit}} \cap \cG\uc{\mbox{Fluc.}}_{\frac{\delta}{3}} \geq 1- \delta$.~This concludes the proof.
\end{proof}

Next, we prove Corollary~\ref{cor:param_pol} and Corollary~\ref{cor:sqrt_reg}.~For the convenience of the reader, we restate these results.

\begin{cor}[Finitely parameterized policies]
    We now consider a set $\Phi$ that consists of policies that have been parameterized by finitely many parameters from the set $W \subset \bR^{d_W}$, and for $w \in W$ let $\phi(\cdot; w): \cS \to \cA$ be the policy parameterized by $w$.~Assume that the policies satisfy
    \begin{align}
        L_W \rho_\Phi(\phi(\cdot; w),\phi(\cdot; w\up)) \ge \norm{w - w\up}_{2}, \forall w, w\up \in W,\label{inv_lip}
    \end{align}
     where $\nu$ is a measure that satisfies Assumption~\ref{assum:stn_dist}. We have $\deff \leq d_W + 2$ for \pzrlmf~and $\deff \leq 2 d_\cS + d_W + 3$ for~\pzrlmb.
\end{cor}
\begin{proof}
    It suffices to prove that $d^\Phi_z \leq d_w$. Let us fix a $w \in W$, and consider the $\gm$-radius ball $B_\gm(\phi(\cdot;w))$ in $(\Phi,\rho_{\Phi,\nu})$. From~\eqref{inv_lip}, we note that all the parameters corresponding to policies in $B_\gm(\phi(\cdot;w))$ are in $B_{\gm/L_W}(w)$.~Hence, the $\gm$-covering number of $\Phi_\gm$, $\cN_\gm(\Phi_\gm) \leq \cN_{\gm/L_W}(W_\gm)$ where $W_\gm$ is the set of parameters corresponding to policies in $\Phi_\gm$. Similarly, $\cN_\gm(\Phi_{\leq\gm}) \leq \cN_{\gm/L_W}(W_{\leq\gm})$ where $W_{\leq\gm}$ is the set of parameters corresponding to policies in $\Phi_{\leq\gm}$. The proof trivially follows from here.
\end{proof}

\begin{cor}
    Let the policy space $(\Phi, \norm{\cdot}_{\rho_\cA,\nu})$ be embedded in a finite-dimensional compact space with a doubling constant~(Definition~\ref{def:sizeofset}) given by $\Lambda$.~Let $\cN_\gamma(\Phi)$, which is the $\gamma$-covering number of the sublevel set $\{\phi \in \Phi: \Delta_\Phi(\phi) \leq \gamma\}$ of $\Delta_\Phi:\phi \to [0,1]$, be bounded by a constant.~Then the regret of \pzrlmf~w.r.t. the policy class scales as $\ctO(\sqrt{T})$ on a high probability set.
\end{cor}
\begin{proof}
    It suffices to show that in the given setup, we have $d^\Phi_z = 0$.~Firstly we note that $\cN_{\frac{\gm}{c_z}}\br{\Phi_{\leq\gm}} \leq \Lambda^{\ceil{\log_2(c_z)}} N$, which is a constant.~Note that $\cN_{\frac{2\gm}{c_z}}\br{\Phi_\gm} \leq \Lambda^{\ceil{\log_2(c_z)}} N$.~Hence,
    \begin{align*}
        \cN_{\frac{\gm}{c_z}}\br{\Phi_\gm} \leq \Lambda^{\ceil{\log_2(c_z)}+1} N.
    \end{align*}
    It trivially follows from here that $d^\Phi_z = 0$, with $c_{z_1} = c_{z_2}/\Lambda = \Lambda^{\ceil{\log_2(c_z)}} N$.~This concludes the proof.
\end{proof}
\section{Auxiliary Results}\label{app:aux_res}

In this section, we derive some useful properties of the algorithm that are used in the proof of regret upper bound.~The first lemma shows that for any active cell $\zeta$ at time $t$, the quantity $\frac{1}{N_t(\zeta)}{\sum_{i=1}^{N_t(\zeta)}{\diamc{\zeta_{t_i}}}}$ is bounded above by $3~\diamc{\zeta}$.~We use this in concentration inequality for the transition kernel estimate.

\begin{lem}\label{lem:avg_diam}
    For all $t \in [T-1]$ and $\zeta \in \cP_t$, let $t_i$ denote the time instance when $\zeta$ or any of its ancestor was visited by \pzrlmb~for the $i$-th time. Then
    \begin{align*}
        \frac{1}{N_t(\zeta)} \sum_{i=1}^{N_t(\zeta)}{\diamc{\zeta_{t_i}}} \leq 3~ \diamc{\zeta}.
    \end{align*}
\end{lem}
\begin{proof}
    By the activation rule~\eqref{def:activationrule}, a cell $\zeta\up$ can be played at most $N_{\max}(\zeta\up) - N_{\min}(\zeta\up) = \tilde{c}_1 2^{2\ell(\zeta\up)} + \frac{\tilde{c}_1}{3} \ind{\zeta\up = \cS \times \cA}$ times while being active, where $\tilde{c}_1 = 3 c_1 \sqrt{d_\cS}^{-2} \log{\br{\frac{T}{\eps \delta}}}~ \eps^{-d_\cS}$. We can write,
    \begin{align*}
        \frac{1}{N_t(\zeta)} \sum_{i=1}^{N_t(\zeta)}{\diamc{\zeta_{t_i}}} &= \frac{1}{N_t(\zeta)} \sum_{i=1}^{N_{\min}(\zeta)}{\diamc{\zeta_{t_i}}} + \frac{1}{N_t(\zeta)} \sum_{i=N_{\min}(\zeta)+1}^{N_t(\zeta)}{\diamc{\zeta_{t_i}}} \\
        &= \frac{\tilde{c}_1 \sqrt{d_\cS}}{3 N_t(\zeta)} + \frac{\tilde{c}_1 \sqrt{d_\cS}}{N_t(\zeta)} \sum_{\ell = 0}^{\ell(\zeta) - 1}{2^\ell} +  \frac{N_t(\zeta) - N_{\min}(\zeta) - 1}{N_t(\zeta)} \diamc{\zeta} \\
        &< \frac{\tilde{c}_1 \sqrt{d_\cS}}{N_t(\zeta)} 2^{\ell(\zeta)} + \frac{N_t(\zeta) - N_{\min}(\zeta) - 1}{N_t(\zeta)} \diamc{\zeta} \\
        &= \frac{3 N_{\min}(\zeta)}{N_t(\zeta)} \diamc{\zeta} + \frac{N_t(\zeta) - N_{\min}(\zeta) - 1}{N_t(\zeta)} \diamc{\zeta} \\
        &=\frac{(N_t(\zeta) + 2 N_{\min}(\zeta) - 1)~ \diamc{\zeta}}{N_t(\zeta)} \\
        &\leq 3~ \diamc{\zeta},
    \end{align*}
    where the last step is due to the fact that $N_{\min}(\zeta) \leq N_t(\zeta)$.
\end{proof}
\section{Useful Results}
\label{app:use_res}
\subsection{Concentration Inequalities}
\begin{lem}[Azuma-Hoeffding inequality]\label{lem:ah_ineq}
    Let $X_1, X_2, \ldots$ be a martingale difference sequence with $|X_i| \leq c,~\forall i$. Then for all $\eps > 0$ and $n \in \bN$,
    \begin{align}
        \bP\br{\sum_{i = 1}^{n}{X_i} \geq \eps } \leq e^{-\frac{\eps^2}{2nc^2}}.
    \end{align}
\end{lem}
The following inequality is Proposition $A.6.6$ of \cite{van1996weak}.
\begin{lem}[Bretagnolle-Huber-Carol inequality]\label{lem:bhl_ineq}
    If the random vector $\br{X_1, X_2, \ldots, X_n}$ is multinomially distributed with parameters $N$ and $\br{p_1, p_2, \ldots, p_n}$, then for $\eps > 0$
    \begin{align}
        \bP\br{\sum_{i=1}^{n}{\abs{X_i - N p_i}} \geq 2\sqrt{N} \eps} \leq 2^n e^{-2\eps^2}.
    \end{align}
    Alternatively, for $\delta > 0$
    \begin{align}
        \bP\br{\sum_{i=1}^{n}{\abs{\frac{X_i}{N} - p_i}} < \sqrt{\frac{2n}{N} \log{\br{\frac{2}{\delta^\frac{1}{n}}}}}} \geq 1 - \delta.
    \end{align}
\end{lem}
The following is essentially Theorem~1 of~\cite{abbasi2011improved}.
\begin{thm}[Self-Normalized Tail Inequality for Vector-Valued Martingales] \label{thm:self_norm}
    Let $\{\cF_t\}_{t=0}^{\infty}$ be a filtration. Let $\{\eta_t\}_{t=1}^{\infty}$ be a real-valued stochastic process such that $\eta_t$ is $\cF_t$ measurable and $\eta_t$ is conditionally $R$ sub-Gaussian for some $R>0$, i.e., 
    \begin{align*}
        \bE\left[ \exp(\lambda \eta_t) | \cF_{t-1}  \right] \le \exp\left( \lambda^2 R^2 \slash 2  \right), \forall \lambda\in \bR.
    \end{align*}
    Let $\{X_t\}_{t=1}^{\infty}$ be an $\bR^{d}$ valued stochastic process such that $X_t$ is $\cF_{t-1}$ measurable. Assume that $V$ is a $d\times d$ positive definite matrix. For $t\ge 0$ define
    \begin{align*}
        \bar{V}_t := V + \sum_{s=1}^{t} X_s X^\top_s,
    \end{align*}
    and
    \begin{align*}
        S_t := \sum_{s=1}^{t} \eta_s X_s.
    \end{align*}
    Then, for any $\delta>0$, with a probability at least $1-\delta$, for all $t\ge 0$,
    \begin{align*}
        \|S_t\|^{2}_{\bar{V}^{-1}_t} \le 2 R^{2} \log{\br{\frac{\det(\bar{V}_t)^{1\slash 2} \det(V)^{-1\slash 2}}{\delta}}}.
    \end{align*}
\end{thm}

\begin{cor}[Self-Normalized Tail Inequality for Martingales] \label{cor:self_norm_vec}
	Let $\{\cF_i\}_{i=0}^{\infty}$ be a filtration. Let $\{\eta_i\}_{i=1}^{\infty}$ be a $\{\cF_i\}_{i=0}^{\infty}$ measurable stochastic process and $\eta_t$ is conditionally $R$ sub-Gaussian for some $R > 0$. Let $\{ X_i \}_{i=1}^{\infty}$ be a $\{0,1\}$-valued $\cF_{i-1}$ measurable stochastic process.
	
	Then, for any $\delta>0$, with a probability at least $1-\delta$, for all $k \geq 0$,
	\begin{align*}
		\left|\sum_{i=1}^{k}{\eta_i X_i}\right| \leq R \sqrt{2 \left(1 + \sum_{i=1}^{k}{X_i}\right) \log{\br{\frac{1 + \sum_{i=1}^{k}{X_i}}{\delta}}}} .
	\end{align*}
\end{cor}
\begin{proof}
	Upon taking $V = 1$, we have that $\bar{V}_t = 1 + \sum_{s = 1}^{t}{X_s}$. The claim follows from Theorem~\ref{thm:self_norm}.
\end{proof}

\subsection{Other Useful Results}

\begin{lem}\label{lem:fx}
    Consider the following function $f(x)$ such that $a_0, a_1, a_2 > 0$,
    \begin{align*}
        f(x) = a_0 x - a_1 \sqrt{x} - a_2.
    \end{align*}
    Then for all $x \geq \br{\frac{a_1}{a_0}}^2 + 3\frac{a_2}{a_0}$, $f(x) \geq 0$.
\end{lem}
\begin{proof}
    Taking $y = x^2$ and applying quadratic formula, we have that
    \begin{align*}
        a_0 y^2 - a_1 y - a_2 \geq 0,
    \end{align*}
    if $y \geq \frac{a_1 + \sqrt{a_1^2 + 4a_0 a_2}}{2 a_0}$. Therefore, $f(x) \geq 0$ if $x \geq \br{\frac{a_1 + \sqrt{a_1^2 + 4a_0 a_2}}{2 a_0}}^2 > \frac{a_1^2}{a_0^2} + \frac{3 a_2}{a_0}$. This concludes the proof.
\end{proof}

\begin{lem}\label{lem:bdd_dotdifLv}
    Let $\mu_1$ and $\mu_2$ be two probability measures on $Z$ and let $v$ be an $\bR$-valued bounded function on $Z$. Then, the following holds.
    \begin{align*}
        \abs{\int_{Z}{v(z) (\mu_1 - \mu_2)(dz)}} \leq \frac{1}{2}\norm{\mu_1 - \mu_2}_{TV} \spn{v}.&
    \end{align*}
\end{lem}
\begin{proof}
    Denote $\lm(\cdot) := \mu_1(\cdot) - \mu_2(\cdot)$. Now let $Z_+,Z_- \subset Z$ be such that $\lm(B) \geq 0$ for every $B \subseteq Z_+$ and $\lm(B) < 0$ for every $B \subseteq Z_-$. We have that
    \begin{align}
        \lm(Z) = \lm(Z_+) + \lm(Z_-) = 0.
    \end{align}
    Also,
    \begin{align}
        \lm(Z_+) - \lm(Z_-) = \norm{\mu_1 - \mu_2}_{TV}.
    \end{align}
    Combining the above two, we get that 
    \begin{align}
        \lm(Z_+) = \frac{1}{2}\norm{\mu_1 - \mu_2}_{TV}.
    \end{align}
    Now,
    \begin{align*}
        \abs{\int_Z{v(z) \lm(dz)}} &= \abs{\int_{Z_+}{v(z) \lm(dz)} + \int_{Z_-}{v(z) \lm(dz)}}\\
        &\leq \abs{\lm(Z_+) \sup_{z \in Z}{v(z)} + \lm(Z_-) \inf_{z \in Z}{v(z)}} \\
        &= \big|\lm(Z_+) \sup_{z \in Z}{v(z)} - \lm(Z_+) \inf_{z \in Z}{v(z)} \\
        &\quad + \lm(Z_+) \inf_{z \in Z}{v(z)} + \lm(Z_-) \inf_{z \in Z}{v(z)}\big| \\
        &= \lm(Z_+) \br{\sup_{z \in Z}{v(z)} - \inf_{z \in Z}{v(z)}} \\
        &= \frac{1}{2}\norm{\mu_1 - \mu_2}_{TV} \spn{v}.
    \end{align*}
    Hence, we have proven the lemma.
\end{proof}

\begin{lem}\label{lem:diff_kern_comp}
    Let $\te_1$ and $\te_2$ be two transition probability kernels of two Markov chains with common state space $\cS$. Let $\max_{s \in \cS}{\norm{\te_1(s,\cdot) - \te_2(s,\cdot)}_{TV}} \leq c$. Then,
    \begin{align*}
        \norm{\te\uc{m}_1(s,\cdot) - \te\uc{m}_2(s,\cdot)}_{TV} \leq m\cdot c,~\forall m \in \bN.
    \end{align*}
    where $\te\uc{m}_i$ is the $m$-step transition kernel of the Markov chain with one-step transition kernel $\te_i$ for $i = 1, 2$.
\end{lem}
\begin{proof}
    We shall prove this using mathematical induction. The base case is given. Let us assume that,
    \begin{align*}
        \norm{\te\uc{i}_1(s,\cdot) - \te\uc{i}_2(s,\cdot)}_{TV} \leq i\cdot c,~\forall i = 1, 2, \ldots, m-1.
    \end{align*}
    See that
    \begin{align*}
        \norm{\te\uc{m}_1(s,\cdot) - \te\uc{m}_2(s,\cdot)}_{TV} &= \left\lVert \int_{\cS}{\te\uc{m-1}_1(s,s\up) \te_1(s\up,\cdot) ds\up} - \int_{\cS}{\te\uc{m-1}_2(s,s\up) \te_1(s\up,\cdot) ds\up}\right.\\
        &\qquad \left. + \int_{\cS}{\te\uc{m-1}_2(s,s\up) \te_1(s\up,\cdot) ds\up} - \int_{\cS}{\te\uc{m-1}_2(s,s\up) \te_2(s\up,\cdot) ds\up}\right\rVert_{TV} \\
        &\leq 2\sup_{A\in \cB_\cS}{\int_{\cS}{\br{\te\uc{m-1}_1(s,s\up) - \te\uc{m-1}_2(s,s\up)}\te_1(s\up,A) ds\up}} \\
        &\quad + 2\sup_{A\in \cB_\cS}{\int_{\cS}{\te\uc{m-1}_2(s,s\up) \br{\te_1(s\up,A) - \te_2(s\up,A)} ds\up}} \\
        &\leq \norm{\te\uc{m-1}_1(s,\cdot) - \te\uc{m-1}_2(s,\cdot)}_{TV} \sup_{A\in \cB_\cS}{\spn{\te_1(\cdot,A)}} \\
        &\quad + \int_{\cS}{\te\uc{m-1}_2(s,s\up) \norm{\te_1(s\up,\cdot) - \te_2(s\up,\cdot)}_{TV} ds\up} \\
        &\leq \norm{\te\uc{m-1}_1(s,\cdot) - \te\uc{m-1}_2(s,\cdot)}_{TV} + \max_{s\up \in \cS}{\norm{\te_1(s\up,\cdot) - \te_2(s\up,\cdot)}_{TV}},
    \end{align*}
    where the first inequality follows from triangle inequality and from the definition of total variation distance, the second inequality follows from Lemma~\ref{lem:bdd_dotdifLv} and by taking the supremum inside integration.~This concludes the proof of the lemma.
\end{proof}
\section{Additional Details of Simulation}

\begin{algorithm}[ht]
    \caption{Policy UCB with Uniform Discretization}
    \label{algo:pucb}
    \begin{algorithmic}
        \STATE {\bfseries Input} Horizon $T$, Constant $C_L$, confidence parameter $\delta$, discretization parameter $\eps$ and policy class $\Phi$
        \STATE {\bfseries Initialize} $h=0$, $k=0$, $\Phi^{act.} = \eps$-net of $\Phi$.
        \FOR{$t= 0$ to $T-1$}
            \IF{$h \geq H_k$}
                \STATE $k \leftarrow k+1$, $h \leftarrow 0$
                \STATE For every $\phi \in \Phi^{act.}$ compute $\text{Index}_t(\phi) = \frac{1}{N_t(\phi)}\sum_{i=0}^{t-1}{\ind{\phi_t = \phi} r(s_t,\phi_t(s_t))} + \diamf{t}{\phi}$.
                \STATE Choose $\phi_k \in \arg\max_{\phi \in \Phi^{act.}_t}{\text{Index}_t(\phi)}$.
                \STATE $H_k = 1 \vee N_t(\phi_k)$
            \ENDIF
            \STATE $h \leftarrow h+1$
            \STATE Play $a_t = \phi_k(s_t)$, observe $s_{t+1}$ and receive $r(s_t, a_t)$.
        \ENDFOR
	\end{algorithmic}
\end{algorithm}

\end{document}